%% file: gantk2.tex
\documentclass{article}

\usepackage[accepted]{icml2022}

\usepackage[T1]{fontenc}
\usepackage[utf8]{inputenc}
\usepackage[english]{babel}
\usepackage{amsmath}
\usepackage{amsthm}
\usepackage{amsfonts}
\usepackage{bm}
\usepackage{dsfont}
\usepackage{pifont}
\usepackage{etoolbox}
\usepackage{hyperref}
\usepackage{url}
\usepackage{graphicx}
\usepackage{xcolor}
\usepackage{mathtools}
\usepackage{amssymb}
\usepackage{commath}
\usepackage{dsfont}
\usepackage{mathrsfs}
\usepackage[short]{optidef}
\usepackage{stmaryrd}
\usepackage{nicefrac}
\usepackage[separate-uncertainty=true]{siunitx}
\usepackage{breqn}
\usepackage{microtype}
\usepackage{cases}
\usepackage{import}
\usepackage{booktabs}
\usepackage{transparent}
\usepackage{multirow}
\usepackage{multido}
\usepackage{makecell}
\usepackage{graphics}
\usepackage{natbib}
\usepackage{subfigure}
\usepackage{wrapfig}
\usepackage{comment}
\usepackage[capitalize, noabbrev]{cleveref}
\usepackage{datetime}
\usepackage[toc,page,header]{appendix}
\usepackage{minitoc}

\input{math_commands.tex}

\input{macros.tex}

\definecolor{mydarkblue}{rgb}{0,0.08,0.45}
\hypersetup{
    colorlinks=true,
    linkcolor=mydarkblue,
    citecolor=mydarkblue,
    filecolor=mydarkblue,
    urlcolor=mydarkblue,
    pdfview=FitH
}

\AtBeginDocument{}

\renewcommand{\partname}{Experiment}

\begin{document}

\doparttoc 
\faketableofcontents 

\twocolumn[
    \icmltitle{A Neural Tangent Kernel Perspective of GANs}

    \icmlsetsymbol{equal}{*}
    \begin{icmlauthorlist}
        \icmlauthor{Jean-Yves Franceschi}{equal,criteo,isir}
        \icmlauthor{Emmanuel de Bézenac}{equal,eth,isir}
        \icmlauthor{Ibrahim Ayed}{equal,isir,thales}
        \icmlauthor{Mickaël Chen}{valeo}
        \icmlauthor{Sylvain Lamprier}{isir}
        \icmlauthor{Patrick Gallinari}{isir,criteo}
    \end{icmlauthorlist}

    \icmlaffiliation{criteo}{Criteo AI Lab, Paris, France}
    \icmlaffiliation{isir}{Sorbonne Université, CNRS, ISIR, F-75005 Paris, France}
    \icmlaffiliation{eth}{Seminar for Applied Mathematics, D-MATH, ETH Zürich, Rämistrasse 101, Zürich-8092, Switzerland}
    \icmlaffiliation{thales}{ThereSIS Lab, Thales, Palaiseau, France}
    \icmlaffiliation{valeo}{Valeo.ai, Paris, France}

    \icmlcorrespondingauthor{Jean-Yves Franceschi}{jycja.franceschi@criteo.com}
    \icmlcorrespondingauthor{Emmanuel de Bézenac}{emmanuel.debezenac@sam.math.ethz.ch}

    \icmlkeywords{Machine Learning, Deep Learning, Generative Modeling, GANs, NTKs, Gradient Flows, ICML}

    \vskip 0.3in
]

\printAffiliationsAndNotice{\textsuperscript{*}Equal contribution, listed in a randomly chosen order.}

\begin{abstract}
    We propose a novel theoretical framework of analysis for Generative Adversarial Networks (GANs).
    We reveal a fundamental flaw of previous analyses which, by incorrectly modeling GANs' training scheme, are subject to ill-defined discriminator gradients.
    We overcome this issue which impedes a principled study of GAN training, solving it within our framework by taking into account the discriminator's architecture.
    To this end, we leverage the theory of infinite-width neural networks for the discriminator via its Neural Tangent Kernel.
    We characterize the trained discriminator for a wide range of losses and establish general differentiability properties of the network.
    From this, we derive new insights about the convergence of the generated distribution, advancing our understanding of GANs' training dynamics.
    We empirically corroborate these results via an analysis toolkit based on our framework, unveiling intuitions that are consistent with GAN practice.
\end{abstract}

\input{src/1_introduction.tex}

\input{src/2_related_work.tex}

\input{src/3_framework.tex}

\input{src/4_discriminator.tex}

\input{src/5_case_analysis.tex}

\input{src/6_experiments.tex}

\input{src/7_discussion.tex}

\bibliography{refs}
\bibliographystyle{icml2022}

\newpage
\appendix
\onecolumn

\renewcommand{\partname}{}
\renewcommand{\thepart}{}
\addcontentsline{toc}{section}{Supplementary Material} 
\part{Appendix} 
\parttoc 
\clearpage

\input{src/appendix.tex}

\end{document}

%% file: math_commands.tex
\def\eqref#1{equation~\ref{#1}}

\def\1{\bm{1}}

\DeclareMathAlphabet{\mathsfit}{\encodingdefault}{\sfdefault}{m}{sl}
\SetMathAlphabet{\mathsfit}{bold}{\encodingdefault}{\sfdefault}{bx}{n}

\def\gA{{\mathcal{A}}}

\def\gC{{\mathcal{C}}}

\def\gF{{\mathcal{F}}}
\def\gG{{\mathcal{G}}}
\def\gH{{\mathcal{H}}}

\def\gK{{\mathcal{K}}}
\def\gL{{\mathcal{L}}}

\def\gN{{\mathcal{N}}}

\def\gP{{\mathcal{P}}}

\def\gT{{\mathcal{T}}}

\def\gW{{\mathcal{W}}}

\def\sN{{\mathbb{N}}}

\newcommand{\E}{\mathbb{E}}

\newcommand{\R}{\mathbb{R}}

\DeclareMathOperator*{\argmax}{arg\,max}
\DeclareMathOperator*{\argmin}{arg\,min}

%% file: macros.tex
\newcommand{\bs}{\boldsymbol}

\newcommand{\T}{\gT}

\newcommand{\hmu}{\hat{\mu}}
\newcommand{\hgamma}{\hat{\gamma}}

\newcommand{\halpha}{\hat{\alpha}}
\newcommand{\hbeta}{\hat{\beta}}
\newcommand{\Loss}{\gL}
\newcommand{\Lo}{\Loss}
\newcommand{\Co}{\gC}
\newcommand{\erm}{\mathrm{e}}
\newcommand{\tg}{\ell}

\newcommand{\id}{\mathrm{id}}

\def\MMD{{\mathrm{MMD}}}
\def\grad{{\nabla}}
\newcommand{\mgrad}[1]{\grad^{#1}}
\newcommand{\Tk}[2]{\T_{#1, #2}}

\DeclarePairedDelimiter\parentheses{(}{)}
\DeclarePairedDelimiter\braces{\{}{\}}
\DeclarePairedDelimiter\brackets{[}{]}

\DeclarePairedDelimiter\lrbrackets{\llbracket}{\rrbracket}
\DeclarePairedDelimiter\euclideannorm{\|}{\|}
\DeclarePairedDelimiter\lrverts{\vert}{\vert}
\DeclarePairedDelimiter\rightvert{.}{\vert}

\DeclarePairedDelimiter\lrangle{\langle}{\rangle}
\DeclarePairedDelimiter\rightopeninterval{[}{)}
\DeclarePairedDelimiterX{\midx}[2]{(}{)}{#1\;\delimsize\vert\;#2}
\DeclarePairedDelimiterX{\midbracesx}[2]{\{}{\}}{#1\;\delimsize\vert\;#2}
\DeclarePairedDelimiterX{\parallelx}[2]{(}{)}{#1\;\delimsize\|\;#2}

\newcommand\GANTK{GAN(TK)\textsuperscript{2}}   

\newcommand{\sprod}[2]{\lrangle*{#1, #2}}

\newcommand{\app}[2]{#1 \parentheses*{#2}}

\DeclareMathOperator{\supp}{supp}

\DeclareMathOperator{\erf}{erf}

\DeclareMathOperator{\cl}{cl}

\newtheorem{assumption}{Assumption}
\newtheorem{prop}{Proposition}
\newtheorem{theorem}{Theorem}
\newtheorem{lemma}{Lemma}
\newtheorem{corollary}{Corollary}
\theoremstyle{definition}
\newtheorem{definition}{Definition}
\newtheorem{remark}{Remark}
\Crefname{assumption}{Assumption}{Assumptions}

\makeatletter
    \newtheorem*{rep@theorem}{\rep@title}
    \newcommand{\newreptheorem}[2]{%
    \newenvironment{rep#1}[1]{%
    \def\rep@title{#2 \ref{##1}}%
    \begin{rep@theorem}}%
    {\end{rep@theorem}}}
\makeatother

\newreptheorem{theorem}{Theorem}
\newreptheorem{prop}{Proposition}
\newreptheorem{lemma}{Lemma}
\newreptheorem{assumption}{Assumption}

\DeclareMathOperator{\spn}{Span}

%% file: src/1_introduction.tex
\section{Introduction}

Generative Adversarial Networks \citep[GANs;][]{Goodfellow2014} have become a canonical approach to generative modeling as they produce realistic samples for numerous data types, with a plethora of variants \citep{Wang2021}.
These models are notoriously difficult to train and require extensive hyperparameter tuning \citep{Brock2019, Karras2020, Liu2021}.
To alleviate these shortcomings, much effort has been put into better understanding their training process, resulting in a vast literature of theoretical analyses.
Many study the various GAN models, found to optimize different losses like the Jensen-Shannon (JS) divergence \citep{Goodfellow2014} and the earth mover’s distance $\gW_1$ \citep{Arjovsky2017b}, to conclude about their comparative advantages.
Yet, empirical evaluations \citep{Lucic2018, Kurach2019} showed that they can yield approximately the same performance.
This indicates that such theoretical works with an exclusive focus on the GAN formulation might not properly model practical settings.

Importantly, GANs are trained in practice with alternating gradient descent-ascent of the generator and discriminator, which the vast majority of analyses do not model.
Yet, this makes GAN training deviate from its formulation in prior works as a min-max problem: the networks are fixed w.r.t.\ to each other at each step in the former, while they depend on each other in the latter.
Therefore, ignoring this ubiquitous procedure prevents those works from adequately explaining GANs' empirical behavior, as it leads to two crucial problems.
Firstly, it alters the true implicitly optimized loss, which consequently differs from the widely adopted JS and $\gW_1$.
Secondly, it compels accurate frameworks to take into account the discriminator parameterization as a neural network with inductive biases influencing the generator's loss landscape, which most previous studies do not, or otherwise be subject to ill-defined discriminator gradients.

To solve these issues, we introduce the first framework of analysis for GANs modeling a wide range of discriminator architectures and GAN formulations, while encompassing alternating optimization.
To this end, we leverage advances in deep learning theory driven by Neural Tangent Kernels \citep[NTKs;][]{Jacot2018} to model discriminator training.
We develop theoretical results showing the relevance of our approach: we establish in our framework the differentiability of the discriminator, hence having well-defined gradients, by proving novel regularity results on its NTK.

This more accurate formalization enables us to derive new knowledge about the generator.
We formulate the dynamics of the generated distribution via the generator's NTK and link it to gradient flows on probability spaces, thereby helping us to discover its implicitly optimized loss.
We deduce in particular that, for GANs under the Integral Probability Metric (IPM), the generated distribution minimizes its Maximum Mean Discrepancy (MMD) given by the discriminator's NTK w.r.t.\ the target distribution.
Moreover, we release an analysis toolkit based on our framework, \GANTK, which we use to empirically validate our analysis and gather new empirical insights: for example, we study the singular performance of the ReLU activation in GAN architectures.

%% file: src/2_related_work.tex
\section{Related Work}
\label{sec:related_work}

We introduce a framework advancing GAN knowledge, supported by prior and novel contributions in the NTK theory.

\paragraph{Neural Tangent Kernels.}
NTKs were introduced by \citet{Jacot2018}, who showed that a trained neural network in the infinite-width regime equates to a kernel method, thereby making its training dynamics tractable and amenable to theoretical study.
This fundamental work has been followed by a thorough line of research generalizing and expanding its initial results \citep{Arora2019, Bietti2019, Lee2019, Liu2020, Sohl2020}, developing means of computing NTKs \citep{Novak2020, Yang2020a}, further analyzing these kernels \citep{Fan2020, Bietti2021, Chen2021}, studying and leveraging them in practice \citep{Zhou2019, Arora2020, Lee2020, Littwin2020b, Tancik2020}, and more broadly exploring infinite-width networks \citep{Littwin2020a, Yang2021, Alemohammad2021}.
These prior works validate that NTKs can encapsulate the characteristics of neural network architectures, providing a solid theoretical basis to understand the effect of architecture on learning problems.

\paragraph{GAN theory.}
A first line of research, started by \citet{Goodfellow2014} and pursued by many others \citep{Nowozin2016, Zhou2019, Sun2020}, studies the loss minimized by the generator.
Assuming that the discriminator is optimal and can take arbitrary values, different families of divergences can be recovered.
However, as noted by \citet{Arjovsky2017a}, these divergences should be ill-suited to GAN training, contrary to empirical evidence.
Our framework addresses this discrepancy, as it properly characterizes the generator's loss and gradient.

Another line of work analyzes the impact of the networks' architecture on the loss landscape of GANs.
Some works, on one hand, only study the solution of the usual min-max formulation of GANs, without considering their usual optimization via alternating gradient descent-ascent \citep{Liu2017, Bai2018, Sun2020, Biau2021, Sahiner2022}.
Not only are these results obtained under restrictive assumptions --~by focusing on a single GAN model like WGAN, or with discriminators and generators limited to shallow, linear or random features models~--, but overlooking alternating optimization hinders their ability to explain GANs' empirical behavior, as detailed in \cref{sec:framework}.

Some studies, on the other hand, deal with the dynamics and convergence of the generated distribution in this setting.
Nonetheless, as these dynamics are highly non-linear, this approach typically requires strong simplifying assumptions: \citet{Mescheder2017} assume the existence of Nash equilibria to the considered optimization problem; \citet{Mescheder2018} reduce the generated distribution to a single data\-point; \citet{DomingoEnrich2020} apply their zero-sum games analysis to mean-field mixtures of generators and discriminators; \citet{Balaji2021} restrict generators and discriminators to be linear or shallow networks; \citet{Yang2022} only work with random feature models as discriminators and a modified WGAN loss.
In contrast to these works, our framework provides a more comprehensive optimization and architecture modeling as we establish generally applicable results about the influence of the discriminator's architecture on the generator's dynamics.

\paragraph{GANs and NTKs.}
To the best of our knowledge, our contribution is the first to employ NTKs to comprehensively study GANs.
Only \citet{Jacot2020} and \citet{Chu2020} have already studied GANs in the light of NTKs, but their studies had restrictive assumptions and limited scope.
\citet{Jacot2020} explain, thanks to the generator's NTK, some GAN failure cases like generator collapse and identify normalization techniques to alleviate them, but without breaking down GANs' training dynamics.
\citet{Chu2020} frame the generator's training dynamics for both GANs and variational autoencoders \citep{Kingma2014, Rezende2014} as a Stein gradient flow under the generator's NTK like in our \cref{sec:dynamics_generated_distr}, but under a strong assumption on generator injectivity which we do not require.
Moreover, both works, focusing on the generator, fail to identify the consequences of the discriminator's parameterization on the generator's dynamics via alternating optimization which, encompassed in our framework, yields in \cref{sec:discriminator_theory,sec:case_analysis} novel results challenging standard GAN knowledge.

Besides the generator, we thoroughly investigate for the first time in the literature the discriminator and its effect on generator optimization via its NTK.
To this end, we derive novel results in NTK theory.
In particular, while other works studied the regularity of NTKs \citep{Bietti2019, Yang2019, Basri2019}, ours is, as far as we know, the first to state general differentiability results for NTKs and infinite-width networks.
Furthermore, we discover the link between IPM optimization and the NTK MMD, independently of and concurrently with \citet{Cheng2021}, although in a different context: they use the NTK MMD for two-sample statistical testing, whereas we find that IPM GANs actually optimize this metric, thereby explaining  the singular performance of NTKs within MMD gradient flows \citep{Arbel2019}.

%% file: src/3_framework.tex
\section{Limits of Previous Studies}
\label{sec:framework}

We present in this section the usual GAN formulation and illustrate the limitations of prior analyses.

First, let us introduce some notations.
Let $\Omega \subseteq \R^n$ be a closed convex set, $\app{\gP}{\Omega}$ the set of probability distributions over $\Omega$, and $\app{L^2}{\mu}$ the set of square-integrable functions from the support $\supp \mu$ of $\mu$ to $\R$ with respect to measure $\mu$, with scalar product $\sprod{\cdot}{\cdot}_{\app{L^2}{\mu}}$.
If $\Lambda \subseteq \Omega$, we write $\app{L^2}{\Lambda}$ for $\app{L^2}{\lambda}$, with $\lambda$ the Lebesgue measure on $\Lambda$.

\subsection{Generative Adversarial Networks}

GAN algorithms seek to produce samples from an unknown target distribution $\beta \in \app{\gP}{\Omega}$.
To this extent, a generator function $g \in \gG \colon \R^d \to \Omega$ parameterized by $\theta$ is learned to map a latent variable $z \sim p_z$ to the space of target samples such that the generated distribution $\alpha_g$ and $\beta$ are indistinguishable for a discriminator $f \in \gF$ parameterized by $\vartheta$.
The generator and the discriminator are trained in an adversarial manner as they are assigned conflicting objectives.

Many GAN models consist in solving the following optimization problem, with $a, b, c\colon \R \to \R$:
\begin{align}
    \label{eq:gan}
    \inf_{g \in \gG} \braces*{\app{\Co_{f^\star_{\alpha_g}}}{\alpha_g} \triangleq  \E_{x \sim \alpha_g} \brackets*{\app{c_{f^\star_{\alpha_g}}}{x}}},
\end{align}
where $c_f = c \circ f$, and $f^\star_{\alpha_g}$ is chosen to solve, or approximate, the following optimization problem:
\begin{equation}
    \label{eq:sup_discr}
    \sup_{f \in \gF} \braces*{\app{\Lo_{\alpha_g}}{f} \triangleq \E_{x \sim \alpha_g} \brackets*{\app{a_f}{x}} - \E_{y \sim \beta} \brackets*{\app{b_f}{y}}}.
\end{equation}
For instance, \citet{Goodfellow2014} originally used $\app{a}{x} = \app{\log}{1-\app{\sigma}{x}}$, $\app{b}{x} = \app{c}{x} = -\app{\log}{\app{\sigma}{x}}$, $\sigma$ being the sigmoid function; in LSGAN \citep{Mao2017}, $\app{a}{x} = -\parentheses*{x+1}^2$, $\app{b}{x} = \parentheses*{x-1}^2$, $\app{c}{x} = x^2$; and for Integral Probability Metrics \citep{Muller1997} used e.g.\ by \citet{Arjovsky2017b}, $a = b = c = \id$.
Many more fall under this formulation \citep{Nowozin2016, Lim2017}.

\cref{eq:gan} is then solved using gradient descent on the generator's parameters, with at each step $j \in \sN$:
\begin{equation}
    \label{eq:gen_descent}
    \theta_{j + 1} = \theta_j - \eta \E_{z \sim p_{z}} \brackets*{\nabla_\theta \app{g_{\theta_j}}{z}^{\top} \nabla_x \rightvert*{\app{c_{f^\star_{\alpha_{g_{\theta_j}}}}}{x}}_{x = \app{g_{\theta_j}}{z}}}.
\end{equation}
This is obtained via the chain rule from the generator's loss $\app{\Co_{f^\star_{\alpha_g}}}{\alpha_g}$ in \cref{eq:gan}.
However, we highlight that the gradient applied in \cref{eq:gen_descent} differs from $\app{\grad_{\theta} \Co_{f^\star_{\alpha_g}}}{\alpha_g}$: the terms taking into account the dependency of the optimal discriminator $f^\star_{\alpha_{g_{\theta}}}$ on the generator's parameters are discarded.
This is because the discriminator is, in practice, considered to be independent of the generator in the alternating optimization between the generator and the discriminator.

Since $\nabla_x \app{c_{f^\star_{\alpha}}}{x} = \app{\nabla_x f^\star_{\alpha}}{x} \cdot \app{c'}{\app{f^\star_{\alpha}}{x}}$, and as highlighted e.g.\ by \citet{Goodfellow2014} and \citet{Arjovsky2017a}, the gradient of the discriminator plays a crucial role in the convergence of GANs.
For example, if this vector field is null on the training data when $\alpha \neq \beta$, the generator's gradient is zero and convergence is impossible.
For this reason, this paper is devoted to developing a better understanding of this gradient field and its consequences on generator optimization when the discriminator is a neural network.
In order to characterize this gradient field, we must first study the discriminator itself.

\subsection{Alternating Optimization and the Necessity of Modeling the Discriminator Parameterization}
\label{sec:issues_large_discr_scpace}

For each GAN formulation, it is customary to elucidate the true generator loss $\app{\mathscr{C}}{\alpha_g, \beta}$ implemented by \cref{eq:sup_discr}, often assuming that $\gF = \app{L^2}{\Omega}$, i.e.\ the discriminator can take arbitrary values.
Under this assumption, $\mathscr{C}$ would have the form of a Jensen-Shannon divergence in the original GAN and of a Pearson $\chi^2$-divergence in LSGAN, for instance.

However, as pointed out by \citet{Arora2017}, the discriminator is trained in practice with a finite number of samples: both fake and target distributions are finite mixtures of Diracs, which we respectively denote as $\halpha_g$ and $\hbeta$.
Let $\hgamma_g = \frac{1}{2}\halpha_g + \frac{1}{2}\hbeta$ be the distribution of training samples.

\begin{assumption}[Finite training set]
    \label{hyp:finite_gamma}
    $\hgamma_g \in \app{\gP}{\Omega}$ is a finite mixture of Diracs.
\end{assumption}

In this setting, the Jensen-Shannon and $\chi^2$ divergences are constant since $\halpha_g$ and $\hbeta$ generally do not have the same support, which would imply that the generator could not be properly trained since it would receive null gradients.
This is the theoretical reason given by \citet{Arjovsky2017a} to introduce new losses and constraints for the discriminator such as in WGAN \citep{Arjovsky2017b}.
However, this is inconsistent with empirical results showing that GANs could already be trained adequately even without the latter losses and constraints \cite{Radford2016}.
This entails that widely accepted theoretical frameworks miss a central ingredient in their modeling of constrained-free GANs.
Uncovering the missing pieces and understanding how they affect training is one of the aims of the current work.

In fact, in the alternating optimization setting as in \cref{eq:gen_descent}, the constancy of $\Lo_{\halpha_g}$, or even of $\Co_{f^\star_{\alpha_g}}$, does not imply that $\nabla_x c_{f^\star_{\alpha_g}}$ in \cref{eq:gen_descent} is zero on these points.
This stems from the gradient of \cref{eq:gen_descent} ignoring the dependency of the optimal discriminator on the generator's parameters: while $\grad_{\theta} \app{\Co_{f^\star_{\alpha_g}}}{\alpha_g}$ might be null, the gradient of \cref{eq:gen_descent} differs and may not be zero, thereby changing the actual loss $\mathscr{C}$ optimized by the generator.
This fact is unaccounted for in many prior analyses, like the ones of \citet{Arjovsky2017b} and \citet{Arora2017}.
We refer to \cref{sec:optimality,app:alternating_optim} for further discussion.

Furthermore, in the previous theoretical frameworks where the discriminator can take arbitrary values, this gradient field is not even defined for any loss $\Lo_{\halpha_g}$.
Indeed, when the discriminator's loss $\app{\Lo_{\halpha_g}}{f}$ is only computed on the empirical distribution $\hgamma_g$ (as in most GAN formulations), the discriminator optimization problem of \cref{eq:sup_discr} never yields a unique optimal solution outside $\hgamma_g$.
This is illustrated by the following straightforward result.
\begin{prop}[Ill-Posed Problem in $\app{L^2}{\Omega}$]
    Suppose that $\gF = \app{L^2}{\Omega}$, $\supp \hgamma_g \subsetneq \Omega$.
    Then, for all $f, h \in \gF$ coinciding over $\supp \hgamma_g$, $\app{\Lo_{\halpha_g}}{f} = \app{\Lo_{\halpha_g}}{h}$ and \cref{eq:sup_discr} has either no or infinitely many optimal solutions in $\gF$, all coinciding over $\supp \hgamma_g$.
\end{prop}
In particular, the set of solutions, if non-empty, contains non-differentiable discriminators as well as discriminators with null or non-informative gradients.
This signifies that the loss alone does not impose any constraint on the values that $f_{\halpha_g}$ takes outside $\supp \hgamma_g$, and more particularly on its gradients.
Thus, this underspecification of the discriminator over $\Omega$ makes the gradient of the optimal discriminator in standard GAN analyses ill-defined.
Therefore, an analysis beyond the loss function is necessary to precisely determine the learning problem and true loss $\mathscr{C}$ of the generator implicitly defined by the discriminator under alternating optimization.

%% file: src/4_discriminator.tex
\section{NTK Analysis of GANs}
\label{sec:discriminator_theory}

To tackle the aforementioned issues, we notice that, in practice, the inner optimization problem of \cref{eq:sup_discr} is not solved exactly.
Instead, using alternating optimization, a proxy neural discriminator is trained using several steps of gradient ascent for each generator update \citep{Goodfellow2016}.
For a learning rate $\varepsilon$ and a fixed generator $g$, this results in the optimization procedure, from $i = 0$ to $N$:
\begin{align}
    \label{eq:discr_ascent}
    \vartheta_{i + 1}^g = \vartheta_i^g + \varepsilon \nabla_\vartheta \app{\Lo_{\halpha_g}}{f_{\vartheta_i^g}}, && f^\star_{\halpha_g} = f_{\vartheta_N^g}.
\end{align}
This training of the discriminator as a neural network solves the gradient indeterminacy of the previous section, but makes a theoretical analysis of its impact unattainable.
We propose to facilitate it thanks to the theory of NTKs.

We develop our framework modeling the discriminator using its NTK in \cref{sec:ntk_intro}.
We confirm in \cref{sec:existence_uniqueness_discriminator,sec:differentiability} that it is consistent by proving that the discriminator gradient is well-defined.
We then leverage this accurate framework to analyze the dynamics of the generated distribution under alternating optimization via the generator's NTK in \cref{sec:dynamics_generated_distr}.
We notably frame this dynamics as a gradient flow of the true generator loss $\mathscr{C}$, which we deduce to be non-increasing during training.

\subsection{Modeling Inductive Biases of the Discriminator in the Infinite-Width Limit}
\label{sec:ntk_intro}

We study the continuous-time version of \cref{eq:discr_ascent}:
\begin{equation}
    \label{eq:discr_param_descent}
    \partial_t \vartheta_{t}^g = \nabla_\vartheta \app{\Lo_{\halpha_g}}{f_{\vartheta_t^g}},
\end{equation}
which we consider in the infinite-width limit of the discriminator, making its analysis more tractable.

In the limit where the width of the hidden layers of $f_t \triangleq f_{\vartheta_t^g}$ tends to infinity, \citet{Jacot2018} showed that its so-called NTK $k_{\vartheta_t^g}$ remains constant during a gradient ascent such as \cref{eq:discr_param_descent}, i.e.\ there is a limiting kernel $k$ such that:
\begin{equation}
    \begin{multlined}
        \forall \tau \in \R_+, \:\: \forall x, y \in \R^n,  \:\:  \forall t \in \brackets*{0, \tau}, \\
        \app{k_{\vartheta_t^g}}{x, y}\: \triangleq \:\partial_{\vartheta} \app{f_t}{x}^{\top}\! \partial_{\vartheta} \app{f_t}{y} \:=\: \app{k}{x, y}.
    \end{multlined}
\end{equation}
In particular, $k$ only depends on the architecture of $f$ and the initialization distribution of its parameters.
The constancy of the NTK of $f_t$ during gradient descent holds for many standard architectures, typically without bottleneck and ending with a linear layer \citep{Liu2020}, which is the case of most standard discriminators in the setting of \cref{eq:sup_discr}.
We discuss the applicability of this approximation in \cref{app:finite_width}.
We more particularly highlight that, under the same conditions, the discriminator's NTK remains constant over the whole GAN optimization process of \cref{eq:gen_descent}, and not only under a fixed generator.

\begin{assumption}[Kernel]
    \label{hyp:kernel}
    $k\colon \Omega^2 \to \R$ is a symmetric positive semi-definite kernel with $k\in\app{L^2}{\Omega^2}$.
\end{assumption}

The constancy of the NTK simplifies the dynamics of training in the functional space.
In order to express these dynamics, we must first introduce some preliminary definitions.

\begin{definition}[Functional gradient]
    Whenever a functional $\gL \colon \app{L^2}{\mu} \to \R$ has sufficient regularity, its gradient w.r.t.\ $\mu$ evaluated at $f \in \app{L^2}{\mu}$ is defined in the usual way as the element $\mgrad{\mu} \app{\gL}{f} \in \app{L^2}{\mu}$ such that for all $\psi \in \app{L^2}{\mu}$:
    \begin{equation}
        \lim_{\varepsilon \to 0} \frac{1}{\varepsilon} \parentheses*{\app{\gL}{f + \varepsilon \psi} - \app{\gL}{f}} = \sprod{\mgrad{\mu} \app{\gL}{f}}{\psi}_{\app{L^2}{\mu}}.
    \end{equation}
\end{definition}

\begin{definition}[RKHS w.r.t.\ $\mu$ and kernel integral operator \citep{Sriperumbudur2010}]
    \label{def:rkhs}
    If $k$ follows \cref{hyp:kernel} and $\mu \in \app{\gP}{\Omega}$ is a finite mixture of Diracs, we define the Reproducing Kernel Hilbert Space (RKHS) $\gH_k^{\mu}$ of $k$ with respect to $\mu$ given by the Moore–Aronszajn theorem as the linear span of functions $\app{k}{x, \cdot}$ for $x \in \supp \mu$.
    Its kernel integral operator from Mercer's theorem is defined as:
    \begin{align}
        \Tk{k}{\mu}\colon \app{L^2}{\mu} \to \gH_k^{\mu}, && h \mapsto \int_x \app{k}{\cdot, x} \app{h}{x}\dif\app{\mu}{x}.
    \end{align}
    Note that $\Tk{k}{\mu}$ generates $\gH_k^{\mu}$, and elements of $\gH_k^{\mu}$ are functions defined over all $\Omega$ as $\gH_k^{\mu} \subseteq \app{L^2}{\Omega}$.
\end{definition}

The results of \citet{Jacot2018} imply that the infinite-width discriminator $f_t$ trained by \cref{eq:discr_param_descent} obeys the following differential equation in-between generator updates:
\begin{equation}
    \label{eq:discr_ode}
    \partial_t f_t = \app{\Tk{k}{\hgamma_g}}{\mgrad{\hgamma_g} \app{\Lo_{\halpha}}{f_t}}.
\end{equation}
Within the alternating optimization of GANs at generator step $j$, $f_0$ would correspond to the previous discriminator step $f^\star_{\alpha_{g_{\theta_j}}} \triangleq f^j$, and $f^{j+1} = f_\tau$, with $\tau$ being the training time of the discriminator in-between generator updates.

In the following \cref{sec:existence_uniqueness_discriminator,sec:differentiability}, we rely on this differential equation to assess under mild assumptions that the proposed framework is sound w.r.t.\ the aforementioned gradient indeterminacy issues.
We first prove that \cref{eq:discr_ode} uniquely defines the discriminator for any initial condition.
We then conclude by proving the differentiability of the resulting trained network.
These results are not GAN-specific but generalize to networks trained under empirical losses like \cref{eq:sup_discr}, e.g.\ for classification and regression.

\subsection{Existence, Uniqueness and Characterization of the Discriminator}
\label{sec:existence_uniqueness_discriminator}

The following is a positive result on the existence and uniqueness of the discriminator that also characterizes its general form, amenable to theoretical analysis.
Presented in the context of a discrete distribution $\hgamma_g$ but generalizable to broader distributions, this result is proved in \cref{app:proof_ode_solution}.

\begin{assumption}[Loss regularity]
    \label{hyp:differentiable_a_b}
    $a$ and $b$ from \cref{eq:sup_discr} are differentiable with Lipschitz derivatives over $\R$.
\end{assumption}

\begin{theorem}[Solution of gradient descent]
    \label{thm:discr_characterization}
    Under \cref{hyp:finite_gamma,hyp:differentiable_a_b,hyp:kernel}, \cref{eq:discr_ode} with initial value $f_0 \in \app{L^2}{\Omega}$ admits a unique solution $f_\cdot \colon \R_+ \to \app{L^2}{\Omega}$.
    Moreover, the following holds for all $t \in \R_+$:
    \begin{equation}
        \label{eq:discr_dynamics}
        \begin{aligned}
            \forall t \in \R_+, f_t & = f_0 + \int_0^t \app{\Tk{k}{\hgamma_g}}{\mgrad{\hgamma_g} \app{\Lo_{\halpha_g}}{f_{s}}} \dif s \\
            & = f_0 + \app{\Tk{k}{\hgamma_g}}{\int_0^t \mgrad{\hgamma_g} \app{\Lo_{\halpha_g}}{f_{s}} \dif s}.
        \end{aligned}
    \end{equation}
\end{theorem}

As for any given training time $t$, there exists a unique $f_t \in \app{L^2}{\Omega}$, defined over all of $\Omega$ and not only the training set, the aforementioned issue in \cref{sec:issues_large_discr_scpace} of determining the discriminator associated to $\hgamma_g$ is now resolved.
It is now possible to study the discriminator in its general form thanks to \cref{eq:discr_dynamics}.
It involves two terms: the previous discriminator state $f_0 = f^j$, as well as the kernel operator of an integral.
This integral is a function that is undefined outside $\supp \hgamma_g$, as by definition $\mgrad{\hgamma_g} \app{\Lo_{\halpha_g}}{f_{s}} \in \app{L^2}{\hgamma_g}$.
Fortunately, the kernel operator behaves like a smoothing operator, as it not only defines the function on all of $\Omega$ but embeds it in a highly structured space.

\begin{corollary}[Training and RKHS]
    \label{cor:discr_rkhs}
    Under \cref{hyp:finite_gamma,hyp:differentiable_a_b,hyp:kernel}, $f_t - f_0$ belongs to the RKHS $\gH_{k}^{\hgamma_g}$ for all $t \in \R_+$.
\end{corollary}

In our setting, this space is generated from the NTK $k$, which only depends on the discriminator architecture, and not on the loss function.
This highlights the crucial role of the discriminator's implicit biases, and enables us to characterize its regularity for a given architecture.

\subsection{Differentiability of the Discriminator and its NTK}
\label{sec:differentiability}

We study in this section the smoothness, i.e.\ infinite differentiability, of the discriminator, which we demonstrate in \cref{app:proof_differentiability}.
It mostly relies on the differentiability of the kernel $k$, by \cref{eq:discr_dynamics}, which is obtained by characterizing the regularity of the corresponding conjugate kernel \citep{Lee2018b}.
Therefore, we prove the differentiability of the NTKs of standard architectures, and then conclude about the differentiability of $f_t$.

\begin{assumption}[Discriminator architecture]
    \label{hyp:discr_archi}
    The discriminator is a standard architecture (fully connected, convolutional or residual).
    The activation can be any standard function: $\tanh$, softplus, ReLU-like, sigmoid, Gaussian, etc.
\end{assumption}

\begin{assumption}[Discriminator regularity]
    \label{hyp:discr_archi_smooth_act}
    The activation function is smooth.
\end{assumption}

\begin{assumption}[Discriminator bias]
    \label{hyp:discr_archi_bias}
    Linear layers have non-null bias terms.
\end{assumption}

We first prove the differentiability of the NTK.

\begin{prop}[Differentiability of $k$]
    \label{prop:ntk_differentiability}
    Let $k$ be the NTK of an infinite-width network from \cref{hyp:discr_archi}.
    For any $y \in \Omega$, $\app{k}{\cdot,y}$ is smooth everywhere over $\Omega$ under \cref{hyp:discr_archi_smooth_act}, or almost everywhere if \cref{hyp:discr_archi_bias} holds instead.
\end{prop}

From \cref{prop:ntk_differentiability}, NTKs satisfy \cref{hyp:kernel}.
Using \cref{cor:discr_rkhs}, we thus conclude on the differentiability of $f_t$.
\begin{theorem}[Differentiability of $f_t$]
    \label{thm:diff_ntk_informal}
    Suppose that $k$ is the NTK of an infinite-width network following \cref{hyp:discr_archi}.
    Then $f_t$ is smooth everywhere over $\Omega$ under \cref{hyp:discr_archi_smooth_act}, or almost everywhere when \cref{hyp:discr_archi_bias} holds instead.
\end{theorem}

\begin{remark}[Bias-free ReLU networks]
    \label{rk:bias_free_relu}
    ReLU networks with hidden layers and no bias are not differentiable at $0$.
    However, by introducing non-zero bias, this non-differentiability at $0$ disappears in the NTK and the infinite-width discriminator.
    This observation explains some experimental results in \cref{sec:experiments}.
    Note that \citet{Bietti2019} state that the bias-free ReLU kernel is not Lipschitz even outside $0$.
    However, we find this result to be incorrect.
    We further discuss this matter in \cref{app:diff_relu_kernel}.
\end{remark}

This result demonstrates that, for a wide range of GANs, e.g.\ vanilla GAN and LSGAN, the optimized discriminator indeed admits gradients, making the gradient flow given to the generator well-defined in our framework.
This supports our motivation to bring the theory closer to the empirical evidence that many GAN models do work in practice while their theoretical interpretation until now has been stating the opposite \citep{Arjovsky2017a}.

\subsection{Dynamics of the Generated Distribution}
\label{sec:dynamics_generated_distr}

By ensuring the existence of $\grad f_{\halpha_g}^\star$, the previous results allow us to study \cref{eq:gen_descent}.
We consider it in continuous-time like \cref{eq:discr_param_descent}, with training time $\tg$ as well as $g_{\tg} \triangleq g_{\theta_{\tg}}$ and $\alpha_{\tg} \triangleq \alpha_{g_{\tg}}$.
NTKs enable us to describe the generated distribution's dynamics and uncover the true generated loss $\mathscr{C}$ in the following manner, as shown in \cref{app:dynamics_generated_proof}.

\begin{prop}[Dynamics of $\alpha_{\tg}$]
    \label{prop:gen_dynamics}
    Under \cref{hyp:discr_archi,hyp:discr_archi_smooth_act}, \cref{eq:gen_descent} is well-posed and yields in continuous-time, with $k_{g_{\tg}}$ the NTK of the generator $g_{\tg}$:
    \begin{equation}
        \partial_{\tg} g_{\tg} = - \app{\Tk{k_{g_{\tg}}}{p_{z}}}{z \mapsto \grad_{x} \rightvert*{\app{c_{f^{\star}_{\halpha_{g_{\tg}}}}}{x}}_{x = \app{g_{\tg}}{z}}}.
    \end{equation}
    Equivalently, the following continuity equation holds for the joint distribution $\alpha^{z}_{\tg}$ of $\parentheses*{z, \app{g_{\tg}}{z}}$ under $z \sim p_{z}$:
    \begin{equation}
        \label{eq:gen_continuity_equation}
        \partial_{\tg} \alpha^{z}_{\tg} = \grad_{x} \cdot \parentheses*{\alpha^{z}_{\tg} \app{\Tk{k_{g_{\tg}}}{p_{z}}}{z \mapsto \grad_{x} \rightvert*{\app{c_{f^{\star}_{\halpha_{g_{\tg}}}}}{x}}_{x = \app{g_{\tg}}{z}}}},
    \end{equation}
    where $\alpha_{\tg}$ is the marginalization of $\alpha^{z}_{\tg}$ over $z \sim p_{z}$.
\end{prop}

In its infinite-width limit, the generator's NTK is also constant: $k_{g_{\tg}} = k_g$; let us study the latter proposition under this assumption.
Suppose that there exists a functional $\mathscr{C}$ over $\app{L^2}{\Omega}$ such that $c_{f^{\star}_{\halpha}} = \partial_{\halpha} \app{\mathscr{C}}{\halpha}$.
Standard results in gradient flows theory --~see \citet[Chapter~10]{Ambrosio2008} for a detailed exposition or \citet[Appendix~A.3]{Arbel2019} for a summary~-- state that $\grad c_{f^{\star}_{\halpha}}$ is then the strong subdifferential of $\app{\mathscr{C}}{\halpha}$ for the Wasserstein geometry.

When $\app{k_g}{z, z'} = \delta_{z - z'} I_n$ with $\delta$ a Dirac centered at $0$, we have $\Tk{k_{g}}{p_{z}} = \id$.
Then, from \cref{eq:gen_continuity_equation}, $\alpha^z_{\tg}$ follows the Wasserstein gradient flow with $\mathscr{C}$ as potential.
This implies that $\app{\mathscr{C}}{\halpha_{\tg}}$ is decreasing w.r.t.\ the generator's training time $\tg$.
In other words, the generator $g$ is trained to minimize $\app{\mathscr{C}}{\halpha_g}$.
Hence, this result characterizes the implicit objective of the generator as $\mathscr{C}$ satisfying $c_{f^{\star}_{\halpha}} = \partial_{\halpha} \app{\mathscr{C}}{\halpha}$.

In the general case, $\Tk{k_{g}}{p_{z}}$ introduces interactions between generated particles as a consequence of the neural parameterization of the generator.
Then, \cref{eq:gen_continuity_equation} amounts to following the same gradient flow as before, but in a Stein geometry \citep{Duncan2019} --~instead of a Wasserstein geometry~-- determined by the generator's integral operator, directly implying that in this case $\app{\mathscr{C}}{\halpha_{\tg}}$ also decreases during training.
This geometrical understanding opens interesting perspectives for theoretical analysis, e.g.\ we see that GAN training in this regime generalizes Stein variational gradient descent \citep{Liu2016}, with the Kullback-Leibler minimization objective between generated and target distributions being replaced with $\app{\mathscr{C}}{\halpha}$.

Improving our understanding of \cref{eq:gen_continuity_equation} is fundamental in order to elucidate the open problem of the neural generator's convergence.
Our study enables us to shed light on these dynamics and highlights the necessity of pursuing the study of GANs via NTKs to obtain a more comprehensive understanding of them, which is the purpose of the rest of this paper.
In particular, the non-interacting case where $\Tk{k_{g}}{p_{z}} = \id$ already yields particularly useful insights that we explore in \cref{sec:experiments}.
Moreover, we discuss in the following section standard GAN losses and determine the minimized functional $\mathscr{C}$ in these cases.

%% file: src/5_case_analysis.tex
\section{Study of Specific Losses}
\label{sec:case_analysis}

Armed with the previous framework, we derive in this section more fine-grained results about the optimized loss $\mathscr{C}$ for standard GAN models.
Proofs are detailed in \cref{app:ode_solve}.

\subsection{The IPM as an NTK MMD Minimizer}
\label{sec:ipm_loss}

We study the case of the IPM loss, with the following remarkable discriminator expression, from which we deduce the objective minimized by the generator.

\begin{prop}[IPM discriminator]
    \label{prop:ipm_mmd}
    Under \cref{hyp:finite_gamma,hyp:kernel}, the solutions of \cref{eq:discr_ode} for $a = b = \id$ are $f_t = f_0 + t f^\ast_{\halpha_g}$, where $f^\ast_{\halpha_g}$ is the unnormalized MMD witness function \citep{Gretton2012} with kernel $k$, yielding:
    \begin{equation}
        \begin{aligned}
            f^\ast_{\halpha_g} & = \E_{x \sim \halpha_g} \brackets*{\app{k}{x, \cdot}} - \E_{y \sim \hbeta} \brackets*{\app{k}{y, \cdot}}, \\ \app{\Lo_{\halpha_g}}{f_t} & = \app{\Lo_{\halpha_g}}{f_0} + t \cdot \app{\MMD^2_{k}}{\halpha_g, \hbeta}.
        \end{aligned}
    \end{equation}
\end{prop}

The latter result signifies that the direction of the gradient given to the discriminator at each of its optimization step is optimal within the RKHS of its NTK, stemming from the linearity of the IPM loss.
The connection with MMD is especially interesting as it has been thoroughly studied in the literature \citep{Muandet2017}.
If $k$ is characteristic, as discussed in \cref{app:characteristic_ntk}, then it defines a distance between distributions.
Moreover, the statistical properties of the loss induced by the discriminator directly follow from those of the MMD: it is an unbiased estimator with a squared sample complexity that is independent of the dimension of the samples \citep{Gretton2007}.

Suppose that the discriminator is reinitialized at every step of the generator, with $f_0 = 0$ in \cref{eq:discr_ode}; this is possible with the initialization scheme of \citet{Zhang2020}.
Then, as $c = \id$ and from \cref{prop:ipm_mmd}, $\grad c_{f_{\halpha}} = \tau \grad f^\ast_{\halpha_g}$, where $\tau$ is the training time of the discriminator.
The latter gradient constitutes the gradient flow of the squared MMD, as shown by \citet{Arbel2019} with convergence guarantees and discretization properties in the absence of generator.
This signifies that $\app{\mathscr{C}}{\halpha} = \tau \app{\MMD^2_{k}}{\halpha_g, \hbeta}$ (see \cref{sec:dynamics_generated_distr}).

Therefore, in the IPM case, we discover via \cref{prop:ipm_mmd} that the generator is actually trained to minimize the MMD between the empirical generated and target distributions, w.r.t.\ the NTK of the discriminator.
This novel connection implies that prior MMD GAN convergence results, like the ones of \citet{Mroueh2021} about the generator trained in such conditions, even though they were established without considering the discriminator's NTK, remarkably transfer to the general unconstrained IPM case.

We further discuss our IPM results in the following remarks.

\begin{remark}[IPM and WGAN]
    Along with a constraint on the set of functions, the IPM is involved in the earth mover’s distance $\gW_1$ \citep{Villani2009} --~used in WGAN and StyleGAN \citep{Karras2019}, close to the hinge loss of BigGAN \citep{Brock2019}~--, the MMD --~used in MMD GAN \citep{Li2017}~--, the total variation, etc.
    In \cref{prop:ipm_mmd}, we study the IPM with the sole constraint of having a neural discriminator.
    Our analysis implies that this suffices to ensure relevant gradients, given the aforementioned convergence results.
    This contradicts the recurring assertion that the Lipschitz constraint of WGAN \citep{Arjovsky2017b} is necessary to solve the gradient issues of prior approaches.
    Indeed, these issues originate from the analyses inadequacy, as shown in this work.
    Hence, while WGAN tackles them by changing the loss and adding a constraint, we fundamentally address them with a refined framework.
    A WGAN analysis, left for future work, would require combining the neural discriminator and Lipschitz constraints.
\end{remark}

\begin{remark}[Instance smoothing]
    We show for IPMs that modeling the discriminator's architecture amounts to smoothing out the input distribution using the kernel integral operator $\Tk{k}{\hgamma_g}$ and can thus be seen as a generalization of the regularization technique for GANs called instance noise \citep{Sonderby2017}.
    This is discussed in \cref{app:instance_noise}.
\end{remark}

\begin{remark}[Regularization by training time]
    \cref{prop:ipm_mmd} highlights the importance of discriminator training time, which needs to be controlled to regularize its gradient magnitude.
    This corresponds to customary practices where the discriminator is trained for a small number of steps to avoid divergence issues, like in DCGAN \citep{Radford2016}.
    In the IPM case, we have, with $\euclideannorm*{\cdot}_{\gH_k^\hgamma}$ as the RKHS semi-norm:
    \begin{equation}
        \euclideannorm*{f_t}_{\gH_k^\hgamma}^2 \leq \euclideannorm*{f_0}_{\gH_k^\hgamma}^2 + t^2 \euclideannorm*{f^\ast_{\halpha_g}}_{\gH_k^\hgamma}^2,
    \end{equation}
    with equality when $f_0=0$.
    This provides a simple criterion to control the discriminator norm by its training time.
    For example, assuming $f_0=0$, setting $t=\euclideannorm*{f^\ast_{\halpha_g}}_{\gH_k^\hgamma}^{-1}$ recovers the MMD dual constraint of a unit-norm discriminator, i.e. that $\euclideannorm*{f_t}_{\gH_k^\hgamma} = 1$, yielding $\app{\Lo_{\halpha_g}}{f_t} = \app{\MMD_{k}}{\halpha_g, \hbeta}$.
\end{remark}

\subsection{LSGAN and New Divergences}
\label{sec:optimality}

Optimality of the discriminator can be proved when assuming that its loss function is well-behaved.
Let us consider the case of LSGAN, for which \cref{eq:discr_ode} can be solved by adapting the results from \citet{Jacot2018} for regression.

\begin{prop}[LSGAN discr.]
    \label{prop:lsgan}
    Under \cref{hyp:finite_gamma,hyp:kernel}, the solutions of \cref{eq:discr_ode} for $a = -\parentheses*{\id + 1}^2$ and $b = -\parentheses*{\id - 1}^2$ are defined for all $t \in \R_+$ as:
    \begin{align}
        f_t = \app{\app{\exp}{-4 t \Tk{k}{\hgamma_g}}}{f_0 - \rho} + \rho, && \textstyle \rho = \ensuremath{\mathinner{\frac{\dif{\parentheses*{\hbeta - \halpha_g}}}{\dif{\parentheses*{\hbeta + \halpha_g}}}}}.
    \end{align}
\end{prop}

In the previous result, $\rho$ is the optimum of $\Lo_{\halpha_g}$ over $\app{L^2}{\hgamma_g}$.
When $k$ is positive definite over $\hgamma_g$ (see \cref{app:characteristic_ntk}), $f_t$ tends to the optimum for $\Lo_{\halpha_g}$ as its limit is $\rho$ over $\supp \hgamma_g$.
Nonetheless, unlike the discriminator with arbitrary values of \cref{sec:issues_large_discr_scpace}, $f_\infty$ is defined over all $\Omega$ thanks to the integral operator $\Tk{k}{\hgamma_g}$.
It is also the solution to the minimum norm interpolant problem in the RKHS \citep{Jacot2018}, therefore explaining why the discriminator does not overfit in scarce data regimes (see \cref{sec:experiments}), and consequently has bounded gradients despite large training times.
We also prove a generalization of this optimality conclusion for concave bounded losses in \cref{app:optimality_concave}.

Following the discussion initiated in \cref{sec:issues_large_discr_scpace} and applying it to LSGAN using \cref{prop:lsgan}, similarly to the Jensen-Shannon, the resulting generator loss on discrete training data is constant when the discriminator is optimal.
However, the gradients received by the generator are not necessarily null, e.g.\ in the empirical analysis of \cref{sec:experiments}.
This is because the learning problem of the generator induced by the discriminator makes the generator minimize another loss $\mathscr{C}$, as explained in \cref{sec:dynamics_generated_distr}.
This raises the question of determining $\mathscr{C}$ for LSGAN and other standard losses.
Furthermore, the same problem arises in the case of incompletely trained discriminators $f_t$.
Unlike the IPM case for which the results of \citet{Arbel2019} who leveraged the theory of \citet{Ambrosio2008} led to a remarkable solution, this connection remains to be established for other adversarial losses.
We leave this as future work.

%% file: src/6_experiments.tex
\begin{figure*}
    \begin{minipage}[c]{0.48\textwidth}
        \centering
        \includegraphics[width=\textwidth]{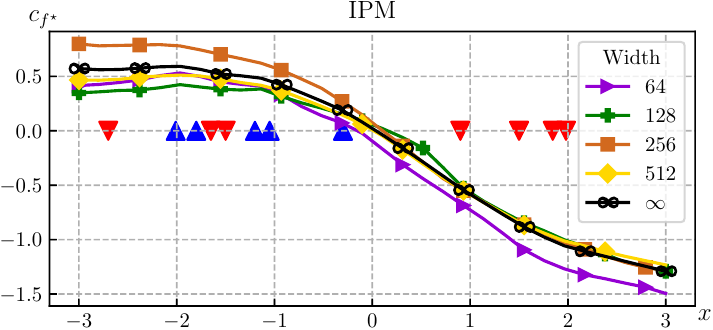}
        \par
        \hfill\includegraphics[width=0.975\textwidth]{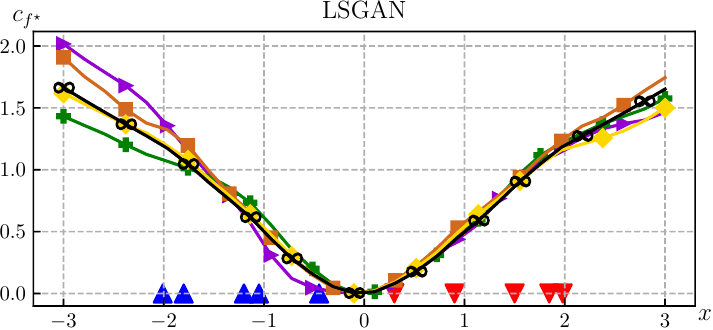}
    \end{minipage}%
    \hfill
    \begin{minipage}[c]{.48\textwidth}
        \centering
        \includegraphics[width=\textwidth]{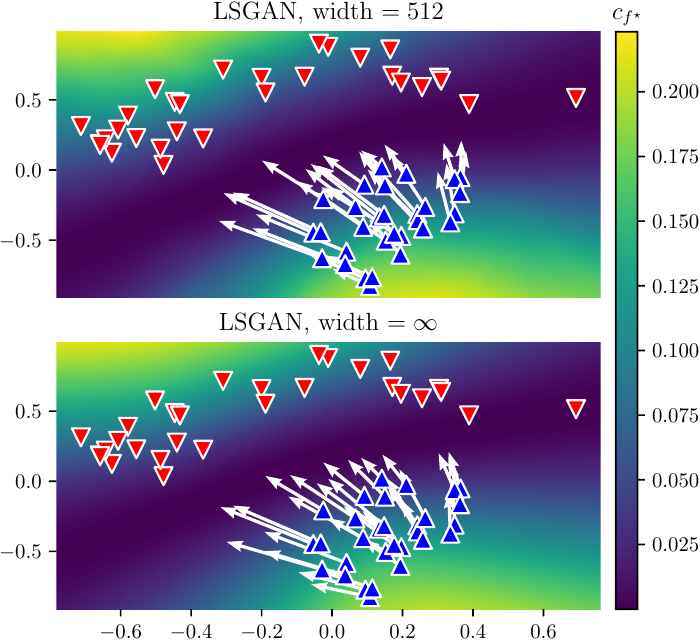}
    \end{minipage}%
    \caption{
        Values of $c_{f^\star}$ for LSGAN and IPM, where $f^\star$ is a $3$-layer ReLU MLP with bias and varying width trained on the dataset represented by \textcolor{red}{\ding{116}} (real) and \textcolor{blue}{\ding{115}} (fake) markers, initialized at $f_0 = 0$.
        The infinite-width network is trained for a time $\tau = 1$ and the finite-width networks using $10$ gradient descent steps with learning rate $\varepsilon = 0.1$, to make training times correspond.
        The gradients $\grad_x c_{f^\star}$ are shown with white arrows on the two-dimensional plots for the fake distribution.
    }
    \label{fig:convergence_1d}
\end{figure*}

\section{Empirical Study}
\label{sec:experiments}

We present a selection of empirical results for different losses and architectures to show the relevance of our framework, with more insights in \cref{app:more_experiments}, by evaluating its adequacy and practical implications on GAN convergence.
All experiments are performed with the proposed Generative Adversarial Neural Tangent Kernel ToolKit \GANTK\, that we release at \url{https://github.com/emited/gantk2} in the hope that the community leverages and expands it for principled GAN analyses.
It is based on the JAX Neural Tangents library \citep{Novak2020}, and is convenient to evaluate architectures and losses based on different visualizations and analyses.

For the sake of efficiency and for these experiments only, we choose $f_0 = 0$ using the antisymmetrical initialization \citep{Zhang2020}.
Indeed, in the analytical computations of the infinite-width regime, taking into account all previous discriminator states for each generator step is computationally infeasible.
This choice also allows us to ignore residual gradients from the initialization, which introduce noise in the optimization process.

\paragraph{Adequacy for fixed distributions.}
We first study the case where generated and target distributions are fixed.
In this setting, we qualitatively study the similarity between the finite- and infinite-width regimes of the discriminator.
\cref{fig:convergence_1d} shows $c_{f^\star}$ and its gradients on one- and two-dimensional data for LSGAN and IPM losses with a ReLU MLP with $3$ hidden layers of varying widths.
We find the behavior of finite-width discriminators to be close to their infinite-width counterpart for standard widths, and converges rapidly to the given limit as the width becomes larger.

In the rest of this section, we focus on the study of convergence of the generated distribution.

\begin{figure*}
    \centering
    \includegraphics[width=\textwidth]{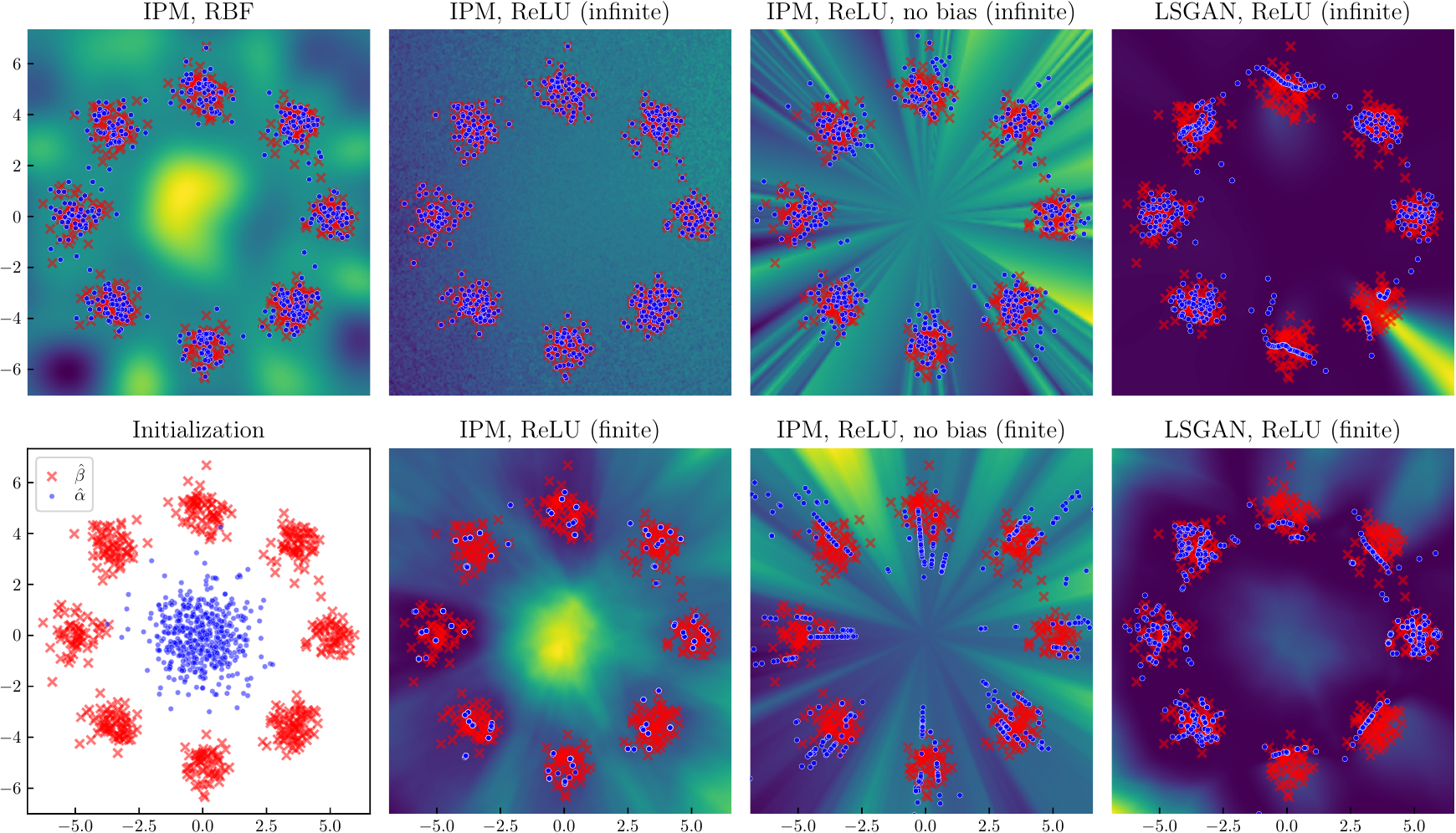}
    \caption{
        Generator (\textcolor{blue}{\ding{108}}) and target (\textcolor{red}{$\times$}) samples for different methods.
        In the background, $c_{f^\star}$.
    }
    \label{fig:8_gaussians}
\end{figure*}

\paragraph{Experimental setting.}
We consider a target distribution sampled from $8$ Gaussians evenly distributed on a centered sphere (cf.\ \cref{fig:8_gaussians}), in a setup similar to that of \citet{Metz2017}, \citet{Srivastava2017} and \citet{Arjovsky2017b}.
We alleviate the complexity of the analysis by following \cref{eq:gen_continuity_equation} with $\Tk{k_{g_{\tg}}}{p_{z}} = \id$, similarly to \citet{Mroueh2019} and \citet{Arbel2019}, thereby modeling the generator's evolution by considering a finite number of samples, initially Gaussian.
For IPM and LSGAN losses, we evaluate the convergence of the generated distributions for a discriminator with ReLU activations in the finite- and infinite-width regime, either with or without bias.
We also comparatively evaluate the advantages of this architecture by considering the case where the infinite-width loss is not given by an NTK, but by the popular Radial Basis Function (RBF) kernel, which is characteristic and presents attractive properties \citep{Muandet2017}.
We refer to \cref{fig:8_gaussians} for qualitative results and \cref{tab:res_8_gaussians_sigmoid} in \cref{app:more_experiments} for a numerical evaluation.
Note that similar results for more datasets, including MNIST and CelebA, and architectures are available in \cref{app:more_experiments}.

\paragraph{Adequacy.}
We observe that correlated performances between the finite- and infinite-width regimes, ReLU networks being considerably better in the latter.
Remarkably, for the infinite-width IPM, generated and target distributions perfectly match.
This can be explained by the high capacity of infinite-width networks; it has already been shown that NTKs benefit from low-data regimes \citep{Arora2020}.

\paragraph{Impact of bias.}
The bias-free discriminator performs worse than with bias, for both regimes and both losses.
This is in line with findings of e.g.\ \citet{Basri2019}, and can be explained in our theoretical framework by comparing their NTKs.
Indeed, the NTK of a bias-free ReLU network is not characteristic, whereas its bias counterpart was proven to present powerful approximation properties \citep{Ji2020}.
Furthermore, results of \cref{sec:differentiability} state that the ReLU NTK with bias is differentiable at $0$, whereas its bias-free version is not, which can disrupt optimization based on its gradients: note in \cref{fig:8_gaussians} the abrupt streaks of the discriminator directed towards $0$ and their consequences on convergence.

\paragraph{NTK vs. RBF.}
We observe the superiority of NTKs over the RBF kernel.
This highlights that the gradients of a ReLU network with bias are particularly well adapted to GANs.
Visualizations of these gradients in the infinite-width limit are available in \cref{sec:gradient_field} and further corroborate these findings.
More generally, we believe that the NTK of ReLU networks could be of particular interest for kernel methods requiring the computation of a spatial gradient, like Stein variational gradient descent \citep{Liu2016}.

%% file: src/7_discussion.tex
\section{Conclusion}
\label{sec:discussion}

Leveraging the theory of infinite-width neural networks, we propose a framework of analysis for GANs explicitly modeling a large variety of discriminator architectures under the alternating optimization setting.
We show that the proposed framework more accurately models GAN training compared to prior approaches by deriving properties of the trained discriminator.
We demonstrate the analysis opportunities of the proposed modeling by studying the generated distribution that we find to follow a gradient flow on probability spaces minimizing some functional that we characterize.
We further study the latter for specific GAN losses and architectures, both theoretically and empirically, notably using our public GAN analysis toolkit.
We believe that this work will serve as a basis for more elaborate analyses, thus leading to more principled, better GAN models.

\section*{Acknowledgements}

We would like to thank all members of the MLIA team from the ISIR laboratory of Sorbonne Université for helpful discussions and comments.

We acknowledge financial support from the DEEPNUM ANR project (ANR-21-CE23-0017-02), the ETH Foundations of Data Science, and the European Union’s Horizon 2020 research and innovation programme under grant agreement 825619 (AI4EU).
This work was granted access to the HPC resources of IDRIS under allocations 2020-AD011011360 and 2021-AD011011360R1 made by GENCI (Grand Equipement National de Calcul Intensif).
Patrick Gallinari is additionally funded by the 2019 ANR AI Chairs program via the DL4CLIM project.

%% file: src/appendix.tex
In the course of this appendix, we drop the subscript $g$ for $\hgamma_g$, $\halpha_g$ and other notations when the dependency on a fixed generator $g$ is clear and indicated in the main paper, for the sake of clarity.

\section{Proofs of Theoretical Results and Additional Results}
\label{app:proofs}

We prove in this section all theoretical results mentioned in \cref{sec:discriminator_theory,sec:case_analysis}.
\cref{app:proof_ode_solution} is devoted to the proof of \cref{thm:discr_characterization}, \cref{app:proof_differentiability} focuses on proving the differentiability results skimmed in \cref{sec:differentiability}, \cref{app:dynamics_generated_proof} contains the demonstration of \cref{prop:gen_dynamics}, and \cref{app:optimality_concave,app:ode_solve} develop the results presented in \cref{sec:case_analysis}.

We will need in the course of these proofs the following standard definition.
For any measurable function $T$ and measure $\mu$, $T_\sharp\mu$ denotes the push-forward measure which is defined as $\app{T_\sharp \mu}{B} = \app{\mu}{\app{T^{-1}}{B}}$, for any measurable set $B$.

\subsection{Recall of Assumptions in the Paper}
\label{sec:assumptions}

\begin{repassumption}{hyp:finite_gamma}[Finite training set]
    $\hgamma \in \app{\gP}{\Omega}$ is a finite mixture of Diracs.
\end{repassumption}

\begin{repassumption}{hyp:kernel}[Kernel]
    $k\colon \Omega^2 \to \R$ is a symmetric positive semi-definite kernel with $k\in\app{L^2}{\Omega^2}$.
\end{repassumption}

\begin{repassumption}{hyp:differentiable_a_b}[Loss regularity]
    $a$ and $b$ from \cref{eq:sup_discr} are differentiable with Lipschitz derivatives over $\R$.
\end{repassumption}

\begin{repassumption}{hyp:discr_archi}[Discriminator architecture]
    The discriminator is a standard architecture (fully connected, convolutional or residual).
    Any activation $\phi$ in the network satisfies the following properties:
    \begin{itemize}
        \item $\phi$ is smooth everywhere except on a finite set $D$;
        \item for all $j\in\sN$, there exist scalars $\lambda_1^{\parentheses*{j}}$ and $\lambda_2^{\parentheses*{j}}$ such that:
        \begin{equation}
            \label{eq:phi_bound}
            \forall x\in\R\setminus D, \lrverts*{\app{\phi^{\parentheses*{j}}}{x}} \leq \lambda_1^{\parentheses*{j}}\lrverts*{x}+\lambda_2^{\parentheses*{j}},
        \end{equation}
        where $\phi^{\parentheses*{j}}$ is the $j$-th derivative of $\phi$.
    \end{itemize}
\end{repassumption}

\begin{repassumption}{hyp:discr_archi_smooth_act}[Discriminator regularity]
    $D = \emptyset$, i.e.\ $\phi$ is smooth.
\end{repassumption}

\begin{repassumption}{hyp:discr_archi_bias}[Discriminator bias]
    Linear layers have non-null bias terms.
    Moreover, for all $x, y \in \R$ such that $x \neq y$, the following holds:
    \begin{equation}
        \label{eq:hyp_relu}
        \E_{\varepsilon \sim \app{\gN}{0, 1}} \app{\phi}{x \varepsilon}^2 \neq \E_{\varepsilon \sim \app{\gN}{0, 1}} \app{\phi}{y \varepsilon}^2.
    \end{equation}
\end{repassumption}

\begin{remark}[Typical activations]
    \cref{hyp:discr_archi,hyp:discr_archi_smooth_act,hyp:discr_archi_bias} cover multiple standard activation functions, including $\tanh$, softplus, ReLU, leaky ReLU and sigmoid.
\end{remark}

\subsection{On the Solutions of \texorpdfstring{\cref{eq:discr_ode}}{Equation (\ref{eq:discr_ode})}}
\label{app:proof_ode_solution}

The methods used in this section are adaptations to our setting of standard methods of proof.
In particular, they can be easily adapted to slightly different contexts, the main ingredient being the structure of the kernel integral operator.
Moreover, it is also worth noting that, although we relied on \cref{hyp:finite_gamma} for $\hgamma$, the results are essentially unchanged if we take a compactly supported measure $\gamma$ instead.

We decompose the proof into several intermediate results.
\cref{thm:existence_uniqueness,prop:integral_inversion}, stated and demonstrated in this section, correspond when combined to \cref{thm:discr_characterization}.

Let us first prove the following two intermediate lemmas.

\begin{lemma}
    \label{lem_completude}
    Let $\delta T>0$ and $\mathcal{F}_{\delta T} = \app{\mathcal{C}}{\brackets*{0,\delta T},\app{B_{\app{L^2}{\hgamma}}}{f_0,1}}$ endowed with the norm:
    \begin{equation}
    \forall u\in\mathcal{F}_{\delta T},\ \euclideannorm*{u} = \sup_{t\in\brackets*{0,\delta T}} \euclideannorm*{u_t}_{\app{L^2}{\hgamma}}.
    \end{equation}
    Then $\mathcal{F}_{\delta T}$ is complete.
\end{lemma}
\begin{proof}
    Let $\parentheses*{u^n}_n$ be a Cauchy sequence in $\mathcal{F}_{\delta T}$. For a fixed $t\in\brackets*{0,\delta T}$:
    \begin{equation}
        \forall n,m, \euclideannorm*{u^n_t-u^m_t}_{\app{L^2}{\hgamma}} \leq \euclideannorm*{u^n-u^m},
    \end{equation}
    which shows that $\parentheses*{u^n_t}_n$ is a Cauchy sequence in ${\app{L^2}{\hgamma}}$.
    ${\app{L^2}{\hgamma}}$ being complete, $\parentheses*{u^n_t}_n$ converges to a $u^\infty_t\in{\app{L^2}{\hgamma}}$.
    Moreover, for $\varepsilon>0$, because $\parentheses*{u^n}$ is Cauchy, we can choose $N$ such that:
    \begin{equation}
        \forall n,m\geq N, \euclideannorm*{u^n-u^m} \leq \varepsilon.
    \end{equation}
    We thus have that:
    \begin{equation}
        \forall t, \forall n,m\geq N, \euclideannorm*{u_t^n-u_t^m}_{\app{L^2}{\hgamma}}\leq \varepsilon.
    \end{equation}
    Then, by taking $m$ to $\infty$, by continuity of the $\app{L^2}{\hgamma}$ norm:
    \begin{equation}
        \forall t, \forall n\geq N, \euclideannorm*{u_t^n-u_t^\infty}_{\app{L^2}{\hgamma}}\leq \varepsilon,
    \end{equation}
    which means that:
    \begin{equation}
        \forall n\geq N, \euclideannorm*{u^n-u^\infty}\leq \varepsilon.
    \end{equation}
    so that $\parentheses*{u^n}_n$ tends to $u^\infty$.

    Moreover, as:
    \begin{equation}
        \forall n, \euclideannorm*{u^n_t}_{\app{L^2}{\hgamma}}\leq 1,
    \end{equation}
    we have that $\euclideannorm*{u^\infty_t}_{\app{L^2}{\hgamma}}\leq 1$.

    Finally, let us consider $s,t\in\brackets*{0,\delta T}$. We have that:
    \begin{equation}
        \forall n, \euclideannorm*{u^\infty_t-u^\infty_s}_{\app{L^2}{\hgamma}}\leq \euclideannorm*{u^\infty_t-u^n_t}_{\app{L^2}{\hgamma}} + \euclideannorm*{u^n_t-u^n_s}_{\app{L^2}{\hgamma}} + \euclideannorm*{u^\infty_s-u^n_s}_{\app{L^2}{\hgamma}}.
    \end{equation}
    The first and the third terms can then be taken as small as needed by definition of $u^\infty$ by taking $n$ high enough, while the second can be made to tend to $0$ as $t$ tends to $s$ by continuity of $u^n$.
    This proves the continuity of $u^\infty$ and shows that $u^\infty\in\mathcal{F}_{\delta T}$.
\end{proof}

\begin{lemma}
    \label{lem_ineq}
    For any $F\in \app{L^2}{\hgamma}$, we have that $F\in \app{L^2}{\halpha}$ and $F\in \app{L^2}{\hbeta}$ with:
    \begin{equation}
        \euclideannorm*{F}_{\app{L^2}{\halpha}} \leq \sqrt{2}\euclideannorm*{F}_{L^2(\hgamma)}\text{ and }\euclideannorm*{F}_{\app{L^2}{\hbeta}}\leq \sqrt{2}\euclideannorm*{F}_{L^2(\hgamma)}.
    \end{equation}
\end{lemma}
\begin{proof}
    For any $F\in \app{L^2}{\hgamma}$, we have that
    \begin{equation}
    \euclideannorm*{F}^2_{\app{L^2}{\hgamma}} = \dfrac{1}{2}\euclideannorm*{F}^2_{\app{L^2}{\halpha}} + \dfrac{1}{2}\euclideannorm*{F}^2_{\app{L^2}{\hbeta}},
    \end{equation}
    so that $F\in \app{L^2}{\halpha}$ and $F\in \app{L^2}{\hbeta}$ with:
    \begin{align}
        \euclideannorm*{F}^2_{\app{L^2}{\halpha}} & = 2\euclideannorm*{F}^2_{\app{L^2}{\hgamma}} - \euclideannorm*{F}^2_{\app{L^2}{\hbeta}} \leq 2\euclideannorm*{F}^2_{L^2(\hgamma)}, && \euclideannorm*{F}^2_{\app{L^2}{\hbeta}} & = 2\euclideannorm*{F}_{\app{L^2}{\hgamma}} - \euclideannorm*{F}_{\app{L^2}{\halpha}} \leq 2\euclideannorm*{F}^2_{L^2(\hgamma)},
    \end{align}
    which allows us to conclude.
\end{proof}

From this, we can prove the existence and uniqueness of the initial value problem from \cref{eq:discr_ode}.
\begin{theorem}[Existence and Uniqueness]
    \label{thm:existence_uniqueness}
    Under \cref{hyp:finite_gamma,hyp:differentiable_a_b,hyp:kernel}, \cref{eq:discr_ode} with initial value $f_0$ admits a unique solution $f_\cdot:\R_+\rightarrow\app{L^2}{\Omega}$.
\end{theorem}
\begin{proof}
    ~
    \paragraph{A few inequalities.}
    We start this proof by proving a few inequalities.

    Let $f,g\in\app{L^2}{\hgamma}$. We have, by the Cauchy-Schwarz inequality, for all $z\in\Omega$:
    \begin{equation}
        \lrverts*{\app{\parentheses*{\app{\Tk{k}{\hgamma}}{\mgrad{\hgamma} \app{\Lo_{\halpha}}{f}} - \app{\Tk{k}{\hgamma}}{\mgrad{\hgamma} \app{\Lo_{\halpha}}{g}}}}{z}} \leq \euclideannorm*{\app{k}{z,\cdot}}_{\app{L^2}{\hgamma}}\euclideannorm*{\mgrad{\hgamma}\app{\Lo_{\halpha}}{f} - \mgrad{\hgamma} \app{\Lo_{\halpha}}{g}}_{\app{L^2}{\hgamma}}.
    \end{equation}
    Moreover, by definition:
    \begin{equation}
        \sprod{\mgrad{\hgamma}\app{\Lo_{\halpha}}{f} - \mgrad{\hgamma} \app{\Lo_{\halpha}}{g}}{h}_{\app{L^2}{\hgamma}} = \int \parentheses*{a'_f-a'_g}h\dif\halpha - \int \parentheses*{b'_f-b'_g}h\dif\hbeta,
    \end{equation}
    so that:
    \begin{equation}
        \euclideannorm*{\mgrad{\hgamma}\app{\Lo_{\halpha}}{f} - \mgrad{\hgamma} \app{\Lo_{\halpha}}{g}}^2_{\app{L^2}{\hgamma}} \leq \euclideannorm*{\mgrad{\hgamma}\app{\Lo_{\halpha}}{f} - \mgrad{\hgamma} \app{\Lo_{\halpha}}{g}}_{\app{L^2}{\hgamma}}\parentheses*{\euclideannorm*{a'_f-a'_g}_{\app{L^2}{\halpha}} + \euclideannorm*{b'_f-b'_g}_{\app{L^2}{\hbeta}}},
    \end{equation}
    and then, along with \cref{lem_ineq}:
    \begin{equation}
        \euclideannorm*{\mgrad{\hgamma}\app{\Lo_{\halpha}}{f} - \mgrad{\hgamma} \app{\Lo_{\halpha}}{g}}_{\app{L^2}{\hgamma}} \leq \euclideannorm*{a'_f-a'_g}_{\app{L^2}{\halpha}} + \euclideannorm*{b'_f-b'_g}_{\app{L^2}{\hbeta}} \leq \sqrt{2} \parentheses*{\euclideannorm*{a'_f-a'_g}_{\app{L^2}{\hgamma}} + \euclideannorm*{b'_f-b'_g}_{\app{L^2}{\hgamma}}}.
    \end{equation}
    By \cref{hyp:differentiable_a_b}, we know that $a'$ and $b'$ are Lipschitz with constants that we denote $K_1$ and $K_2$.
    We can then write for all $x$:
    \begin{align}
        \lrverts*{\app{a'}{\app{f}{x}}-\app{a'}{\app{g}{x}}} \leq K_1\lrverts*{\app{f}{x}-\app{g}{x}}, && \lrverts*{\app{b'}{\app{f}{x}}-\app{b'}{\app{g}{x}}} \leq K_2\lrverts*{\app{f}{x}-\app{g}{x}},
    \end{align}
    so that:
    \begin{align}
        \euclideannorm*{a'_f-a'_g}_{\app{L^2}{\hgamma}} \leq K_1\euclideannorm*{f-g}_{\app{L^2}{\hgamma}}, && \euclideannorm*{b'_f-b'_g}_{\app{L^2}{\hgamma}} \leq K_2\euclideannorm*{f-g}_{\app{L^2}{\hgamma}}.
    \end{align}
    Finally, we can now write, for all $z\in\Omega$:
    \begin{equation}
        \label{eq:ineq_omega}
        \lrverts*{\app{\parentheses*{\app{\Tk{k}{\hgamma}}{\mgrad{\hgamma} \app{\Lo_{\halpha}}{f}} - \app{\Tk{k}{\hgamma}}{\mgrad{\hgamma} \app{\Lo_{\halpha}}{g}}}}{z}}
        \leq \sqrt{2}\parentheses*{K_1+K_2}\euclideannorm*{f-g}_{\app{L^2}{\hgamma}}\euclideannorm*{\app{k}{z,\cdot}}_{\app{L^2}{\hgamma}},
        \tag{A}
    \end{equation}
    and then:
    \begin{equation}
        \label{eq:A_lip}
        \euclideannorm*{\app{\Tk{k}{\hgamma}}{\mgrad{\hgamma} \app{\Lo_{\halpha}}{f}} - \app{\Tk{k}{\hgamma}}{\mgrad{\hgamma} \app{\Lo_{\halpha}}{g}}}_{\app{L^2}{\hgamma}} \leq K\euclideannorm*{f-g}_{\app{L^2}{\hgamma}},
        \tag{B}
    \end{equation}
    where $K=\sqrt{2}\parentheses*{K_1+K_2}\sqrt{\int\euclideannorm*{\app{k}{z,\cdot}}^2_{\app{L^2}{\hgamma}}\dif\app{\hgamma}{z}}$ is finite as a finite sum of finite terms from \cref{hyp:finite_gamma,hyp:kernel}.
    In particular, putting $g=0$ and using the triangular inequality also gives us:
    \begin{equation}
        \label{eq:A_bound}
        \euclideannorm*{\app{\Tk{k}{\hgamma}}{\mgrad{\hgamma} \app{\Lo_{\halpha}}{f}}}_{\app{L^2}{\hgamma}} \leq K\euclideannorm*{f}_{\app{L^2}{\hgamma}} + M,
        \tag{B'}
    \end{equation}
    where $M=\euclideannorm*{\app{\Tk{k}{\hgamma}}{\mgrad{\hgamma} \app{\Lo_{\halpha}}{0}}}_{\app{L^2}{\hgamma}}$.

    \paragraph{Existence and uniqueness in $\app{L^2}{\hgamma}$.}

    We now adapt the standard fixed point proof to prove existence and uniqueness of a solution to the studied equation in $\app{L^2}{\hgamma}$.

    We consider the family of spaces $\mathcal{F}_{\delta T} = \app{\mathcal{C}}{\brackets*{0,\delta T},\app{B_{\app{L^2}{\hgamma}}}{f_0,1}}$. $\mathcal{F}_{\delta T}$ is defined, for $\delta T>0$, as the space of continuous functions from $\brackets*{0,\delta T}$ to the closed ball of radius $1$ centered around $f_0$ in $\app{L^2}{\hgamma}$ which we endow with the norm:
    \begin{equation}
        \forall u\in\mathcal{F}_{\delta T},\ \euclideannorm*{u} = \sup_{t\in\brackets*{0,\delta T}} \euclideannorm*{u_t}_{\app{L^2}{\hgamma}}.
    \end{equation}

    We now define the application $\Phi$ where $\app{\Phi}{u}$ is defined as, for any $u\in\mathcal{F}_{\delta T}$:
    \begin{equation}
        \Phi(u)_t = f_0 + \int_0^t \app{\Tk{k}{\hgamma}}{\mgrad{\hgamma} \app{\Lo_{\halpha}}{u_s}} \dif s.
    \end{equation}
    We have, using \cref{eq:A_bound}:
    \begin{equation}
        \euclideannorm*{\app{\Phi}{u}_t-f_0}_{\app{L^2}{\hgamma}} \leq \int_0^t \parentheses*{K\euclideannorm*{u_s}_{\app{L^2}{\hgamma}}+M} \dif s \leq \parentheses*{K+M}\delta T.
    \end{equation}
    Thus, taking $\delta T = \parentheses*{2\parentheses*{K+M}}^{-1}$ makes $\Phi$ an application from $\mathcal{F}_{\delta T}$ into itself.
    Moreover, we have:
    \begin{equation}
        \forall u,v\in\mathcal{F}_{\delta T}, \euclideannorm*{\app{\Phi}{u}-\app{\Phi}{v}} \leq \frac{1}{2}\euclideannorm*{u-v},
    \end{equation}
    which means that $\Phi$ is a contraction of $\mathcal{F}_{\delta T}$.
    \cref{lem_completude} and the Banach-Picard theorem then tell us that $\Phi$ has a unique fixed point in $\mathcal{F}_{\delta T}$.
    It is then obvious that such a fixed point is a solution of \cref{eq:discr_ode} over $\brackets*{0,\delta T}$.

    Let us now consider the maximal $T>0$ such that a solution $f_t$ of \cref{eq:discr_ode} is defined over $\rightopeninterval*{0,T}$.
    We have, using \cref{eq:A_bound}:
    \begin{equation}
        \forall t\in\rightopeninterval*{0,T}, \euclideannorm*{f_t}_{\app{L^2}{\hgamma}} \leq \euclideannorm*{f_0}_{\app{L^2}{\hgamma}} + \int_0^t \parentheses*{\euclideannorm*{f_s}_{\app{L^2}{\hgamma}}+M}\dif s,
    \end{equation}
    which, using Grönwall's lemma, gives:
    \begin{equation}
        \label{eq:gronwall_lem}
        \forall t\in\rightopeninterval*{0,T}, \euclideannorm*{f_t}_{\app{L^2}{\hgamma}} \leq \euclideannorm*{f_0}_{\app{L^2}{\hgamma}}\erm^{KT} + \frac{M}{K}\parentheses*{\erm^{KT}-1}.
    \end{equation}
    Define $g^n = f_{T-\frac{1}{n}}$.
    We have, again using \cref{eq:A_bound}:
    \begin{equation}
        \forall m\geq n, \euclideannorm*{g^n-g^m}_{\app{L^2}{\hgamma}} \leq \int_{T-\frac{1}{n}}^{T-\frac{1}{m}}(K\euclideannorm*{f_s}+M)\dif s \leq \parentheses*{\frac{1}{n} - \frac{1}{m}}\parentheses*{\euclideannorm*{f_0}_{\app{L^2}{\hgamma}}\erm^{KT} + \frac{M}{K}\parentheses*{\erm^{KT}-1}},
    \end{equation}
    which shows that $\parentheses*{g^n}_n$ is a Cauchy sequence.
    $\app{L^2}{\hgamma}$ being complete, we can thus consider its limit $g^\infty$.
    Clearly, $f_t$ tends to $g^\infty$ in $\app{L^2}{\hgamma}$.
    By considering the initial value problem associated with \cref{eq:discr_ode} starting from $g^\infty$, we can thus extend the solution $f_t$ to $\rightopeninterval*{0,T+\delta T}$, thus contradicting the maximality of $T$, which proves that the solution can be extended to $\R_+$.

    \paragraph{Existence and uniqueness in $\app{L^2}{\Omega}$.}

    We now conclude the proof by extending the previous solution to $\app{L^2}{\Omega}$.
    We keep the same notations as above and, in particular, $f$ is the unique solution of \cref{eq:discr_ode} with initial value $f_0$.

    Let us define $\Tilde{f}$ as:
    \begin{equation}
        \forall t,\forall x, \app{\Tilde{f}_t}{x} = \app{f_0}{x} + \int_0^t \app{\app{\Tk{k}{\hgamma}}{\mgrad{\hgamma} \app{\Lo_{\halpha}}{f_s}}}{x} \dif s,
    \end{equation}
    where the r.h.s.\ only depends on $f$ and is thus well-defined.
    By remarking that $\Tilde{f}$ is equal to $f$ on $\supp \hgamma$ and that, for every $s$,
    \begin{equation}
        \app{\Tk{k}{\hgamma}}{\mgrad{\hgamma} \app{\Lo_{\halpha}}{\Tilde{f}_s}} = \app{\Tk{k}{\hgamma}}{\mgrad{\hgamma} \app{\Lo_{\halpha}}{\rightvert*{\Tilde{f}_s}_{\supp \hgamma}}} = \app{\Tk{k}{\hgamma}}{\mgrad{\hgamma} \app{\Lo_{\halpha}}{f_s}},
    \end{equation}
    we see that $\Tilde{f}$ is solution to \cref{eq:discr_ode}.
    Moreover, from \cref{hyp:kernel}, we know that, for any $z\in\Omega$, $\int \app{k}{z,x}^2\dif\app{\Omega}{x}$ is finite and, from \cref{hyp:finite_gamma}, that $\euclideannorm*{\app{k}{z,\cdot}}^2_{\app{L^2}{\hgamma}}$ is a finite sum of terms $\app{k}{z,x_i}^2$ which shows that $\int \euclideannorm*{\app{k}{z,\cdot}}^2_{\app{L^2}{\hgamma}}\dif\app{\Omega}{z}$ is finite, again from \cref{hyp:kernel}.
    We can then say that $\Tilde{f}_s\in\app{L^2}{\Omega}$ for any $s$ by using the above with \cref{eq:ineq_omega} taken for $g=0$.

    Finally, suppose $h$ is a solution to \cref{eq:discr_ode} with initial value $f_0$.
    We know that $\rightvert*{h}_{\supp \hgamma}$ coincides with $f$ and thus with $\rightvert*{\Tilde{f}}_{\supp \hgamma}$ in $\app{L^2}{\hgamma}$ as we already proved uniqueness in the latter space.
    Thus, we have that $\euclideannorm*{\rightvert*{h_s}_{\supp \hgamma}-\rightvert*{\Tilde{f}_s}_{\supp \hgamma}}_{\app{L^2}{\hgamma}} = 0$ for any $s$.
    Now, we have:
    \begin{equation}
        \begin{multlined}
            \forall z \in\Omega, \forall s, \lrverts*{\app{\parentheses*{\app{\Tk{k}{\hgamma}}{\mgrad{\hgamma} \app{\Lo_{\halpha}}{h_s}} - \app{\Tk{k}{\hgamma}}{\mgrad{\hgamma} \app{\Lo_{\halpha}}{\Tilde{f}_s}}}}{z}} \\
            = \lrverts*{\app{\parentheses*{\app{\Tk{k}{\hgamma}}{\mgrad{\hgamma} \app{\Lo_{\halpha}}{\rightvert*{h_s}_{\supp \hgamma}}} - \app{\Tk{k}{\hgamma}}{\mgrad{\hgamma} \app{\Lo_{\halpha}}{\rightvert*{\Tilde{f}_s}_{\supp \hgamma}}}}}{z}} \leq 0,
        \end{multlined}
    \end{equation}
    by \cref{eq:ineq_omega}.
    This shows that $\app{\partial_t}{\Tilde{f}-h} = 0$ and, given that $h_0=\Tilde{f}_0=f_0$, we have $h=\Tilde{f}$ which concludes the proof.

\end{proof}

There only remains to prove for \cref{thm:discr_characterization} the inversion between the integral over time and the integral operator.
We first prove an intermediate lemma and then conclude with the proof of the inversion.

\begin{lemma}
    \label{lem:finiteness}
    Under \cref{hyp:finite_gamma,hyp:differentiable_a_b,hyp:kernel}, $\int_0^T \parentheses*{\euclideannorm*{a'}_{\app{L^2}{\parentheses*{f_s}_\sharp\halpha}} + \euclideannorm*{b'}_{\app{L^2}{\parentheses*{f_s}_\sharp\hbeta}}} \dif s$ is finite for any $T>0$.
\end{lemma}
\begin{proof}
    Let $T>0$.
    We have, by \cref{hyp:differentiable_a_b} and the triangular inequality:
    \begin{equation}
        \forall x, \lrverts*{\app{a'}{\app{f}{x}}}\leq K_1\lrverts*{\app{f}{x}}+M_1,
    \end{equation}
    where $M_1 = \lrverts*{\app{a'}{0}}$.
    We can then write, using \cref{lem_ineq} and the inequality from \cref{eq:gronwall_lem}:
    \begin{equation}
        \forall s\leq T, \euclideannorm*{a'}_{\app{L^2}{\parentheses*{f_s}_\sharp\halpha}} \leq K_1\sqrt{2}\euclideannorm*{f_s}_{\app{L^2}{\hgamma}} + M_1 \leq K_1\sqrt{2}\parentheses*{\euclideannorm*{f_0}_{\app{L^2}{\hgamma}}\erm^{KT} + \frac{M}{K}\parentheses*{\erm^{KT}-1}} + M_1,
    \end{equation}
    the latter being constant in $s$ and thus integrable on $\brackets*{0,T}$.
    We can then bound $\euclideannorm*{b'}_{\app{L^2}{\parentheses*{f_s}_\sharp\hbeta}}$ similarly, which concludes the proof.
\end{proof}

\begin{prop}[Integral inversion]
    \label{prop:integral_inversion}
    Under \cref{hyp:finite_gamma,hyp:differentiable_a_b,hyp:kernel}, the following integral inversion holds:
    \begin{equation}
        f_t = f_0 + \int_0^t \app{\Tk{k_f}{\hgamma}}{\mgrad{\hgamma} \app{\Lo_{\halpha, \hbeta}}{f_{s}}} \dif s = f_0 + \app{\Tk{k_f}{\hgamma}}{\int_0^t \mgrad{\hgamma} \app{\Lo_{\halpha, \hbeta}}{f_{s}} \dif s}.
    \end{equation}
\end{prop}
\begin{proof}
    By definition, a straightforward computation gives, for any function $h\in\app{L^2}{\hgamma}$:
    \begin{equation}
        \sprod{\mgrad{\hgamma} \app{\Lo_{\halpha}}{f}}{h}_{\app{L^2}{\hgamma}} = \dif \app{\Lo_{\halpha}}{f}\brackets*{h} = \int a'_{f}h\dif\halpha  - \int b'_{f}h \dif\hbeta.
    \end{equation}
    We can then write:
    \begin{equation}
        \euclideannorm*{\mgrad{\hgamma} \app{\Lo_{\halpha}}{f_{t}}}_{\app{L^2}{\hgamma}}^2 = \sprod{\mgrad{\hgamma} \app{\Lo_{\halpha}}{f_{t}}}{\mgrad{\hgamma} \app{\Lo_{\halpha}}{f_{t}}}_{\app{L^2}{\hgamma}} = \int a'_{f_t}\mgrad{\hgamma} \app{\Lo_{\halpha}}{f_{t}}\dif\halpha  - \int b'_{f_t}\mgrad{\hgamma} \app{\Lo_{\halpha}}{f_{t}} \dif\hbeta,
    \end{equation}
    so that, with the Cauchy-Schwarz inequality and \cref{lem_ineq}:
    \begin{equation}
       \begin{aligned}
            \euclideannorm*{\mgrad{\hgamma} \app{\Lo_{\halpha}}{f_{t}}}_{\app{L^2}{\hgamma}}^2
            & \leq \int \lrverts*{a'_{f_t}} \lrverts*{\mgrad{\hgamma} \app{\Lo_{\halpha}}{f_{t}}}\dif\halpha  + \int \lrverts*{b'_{f_t}} \lrverts*{\mgrad{\hgamma} \app{\Lo_{\halpha}}{f_{t}}} \dif\hbeta \\
            & \leq \euclideannorm*{a'_{f_t}}_{\app{L^2}{\halpha}}\euclideannorm*{\mgrad{\hgamma} \app{\Lo_{\halpha}}{f_{t}}}_{\app{L^2}{\halpha}} + \euclideannorm*{b'_{f_t}}_{\app{L^2}{\hbeta}}\euclideannorm*{\mgrad{\hgamma} \app{\Lo_{\halpha}}{f_{t}}}_{\app{L^2}{\hbeta}} \\
            & \leq \sqrt{2}\euclideannorm*{\mgrad{\hgamma} \app{\Lo_{\halpha}}{f_{t}}}_{L^2(\hgamma)}\brackets*{\euclideannorm*{a'_{f_t}}_{\app{L^2}{\halpha}} + \euclideannorm*{b'_{f_t}}_{\app{L^2}{\hbeta}}},
        \end{aligned}
    \end{equation}
    which then gives us:
    \begin{equation}\label{eq:lemma}
        \euclideannorm*{\mgrad{\hgamma} \app{\Lo_{\halpha}}{f_{t}}}_{\app{L^2}{\hgamma}} \leq \sqrt{2}\brackets*{\euclideannorm*{a'}_{\app{L^2}{\parentheses*{f_t}_\sharp\halpha}} + \euclideannorm*{b'}_{\app{L^2}{\parentheses*{f_t}_\sharp\hbeta}}}.
    \end{equation}

    By the Cauchy-Schwarz inequality and \cref{eq:lemma}, we then have for all $z$:
    \begin{equation}
        \begin{multlined}
            \int_0^t \int_x \lrverts*{\app{k}{z, x} \app{\mgrad{\hgamma} \app{\Lo_{\halpha}}{f_{s}}}{x}}\dif\app{\hgamma}{x}\dif s \leq \int_0^t \euclideannorm*{\app{k}{z, \cdot}}_{\app{L^2}{\hgamma}}\euclideannorm*{\mgrad{\hgamma} \app{\Lo_{\halpha}}{f_{s}}}_{\app{L^2}{\hgamma}}\dif s \\
            \leq \sqrt{2}\euclideannorm*{\app{k}{z, \cdot}}_{\app{L^2}{\hgamma}}\int_0^t\brackets*{\euclideannorm*{a'}_{\app{L^2}{\parentheses*{f_s}_\sharp\halpha}} + \euclideannorm*{b'}_{\app{L^2}{\parentheses*{f_s}_\sharp\hbeta}}}\dif s.
        \end{multlined}
    \end{equation}
    The latter being finite by \cref{lem:finiteness}, we can now use Fubini's theorem to conclude that:
    \begin{equation}
        \begin{aligned}
            \int_0^t \app{\Tk{k_f}{\hgamma}}{\mgrad{\hgamma} \app{\Lo_{\halpha}}{f_{s}}} \dif s & = \int_0^t \int_x\app{k}{\cdot, x} \app{\mgrad{\hgamma} \app{\Lo_{\halpha}}{f_{s}}}{x}\dif\app{\hgamma}{x}\dif s \\
            & = \int_x \app{k}{\cdot, x} \brackets*{\int_0^t \app{\mgrad{\hgamma} \app{\Lo_{\halpha}}{f_{s}}}{x}\dif s}\dif\app{\hgamma}{x}\\
            & = \app{\Tk{k_f}{\hgamma}}{\int_0^t \app{\mgrad{\hgamma} \app{\Lo_{\halpha}}{f_{s}}}{x}\dif s}.
        \end{aligned}
    \end{equation}
\end{proof}

\subsection{Differentiability of Infinite-Width Networks and their NTKs}
\label{app:proof_differentiability}

Given \cref{thm:discr_characterization}, establishing the desired differentiability of $f_t$ can be done by separately proving similar results on both $f_t - f_0$ and $f_0$.

In both cases, this involves the differentiability of the following activation kernel $\app{\gK_{\phi}}{A}$ given another differentiable kernel $A$:
\begin{equation}
    \label{eq:act_kernel}
    \app{\gK_{\phi}}{A}\colon x, y \mapsto \mathbb{E}_{f\sim \app{\mathcal{GP}}{0,A}}\brackets*{\app{\phi}{\app{f}{x}}\app{\phi}{\app{f}{y}}},
\end{equation}
where $\app{\mathcal{GP}}{0,A}$ is a univariate centered Gaussian Process (GP) with covariance function $A$.
Indeed, the kernel-transforming operator $\gK_{\phi}$ is central in the recursive computation of the neural network conjugate kernel sss which determines the NTK (involved in $f_t - f_0 \in \gH_{k}^{\hgamma_g}$) as well as the behavior of the network at initialization (which follows a GP with the conjugate kernel as covariance).

Hence, our proof of \cref{thm:diff_ntk_informal} relies on the preservation of kernel smoothness through $\gK_{\phi}$, proved in \cref{app:proof_differentiability_act_kernel}, which ensures the smoothness of the conjugate kernel, the NTK and, in turn, of $f_t$ as addressed in \cref{app:proof_differentiability_networks} which concludes the overall proof.

Before developing these two main steps, we first need to state the following lemma showing the regularity of samples of a GP from the regularity of the corresponding kernel.

\begin{lemma}[GP regularity]
    \label{lem:gp_as_diff}
    Let $A\colon\R^n\times\R^n\rightarrow\R$ be a symmetric kernel.
    Let $V$ an open set such that $A$ is $\mathcal{C}^\infty$ on $V\times V$.
    Then the GP induced by the kernel $A$ has a.s.\ $\mathcal{C}^\infty$ sample paths on $V$.
\end{lemma}
\begin{proof}
    Because $A$ is $\mathcal{C}^\infty$ on $V\times V$, we know, from Theorem 2.2.2 of \citet{Adler1981} for example, that the corresponding GP $f$ is mean-square smooth on $V$.
    If we take $\alpha$ a $k$-th order multi-index, we also know, again from \citet{Adler1981}, that $\partial^\alpha f$ is also a GP with covariance kernel $\partial^\alpha A$.
    As $A$ is $\mathcal{C}^\infty$, $\partial^\alpha A$ then is differentiable and $\partial^\alpha f$ has partial derivatives which are mean-square continuous.
    Then, by the Corollary 5.3.12 of \citet{Scheuerer2009}, we can say that $\partial^\alpha f$ has continuous sample paths a.s.\ which means that $f\in \app{\mathcal{C}^k}{V}$.
    This proves the lemma.
\end{proof}

\subsubsection{\texorpdfstring{$\gK_{\phi}$}{K_phi} Preserves Kernel Differentiability}
\label{app:proof_differentiability_act_kernel}

Given the definition of $\app{\gK_{\phi}}{A}$ in \cref{eq:act_kernel}, we choose to prove its differentiability via the dominated convergence theorem and Leibniz integral rule.
This requires to derive separate proofs depending on whether $\phi$ is smooth everywhere or almost everywhere.

The former case allows us to apply strong GP regularity results leading to $\gK_{\phi}$ preserving kernel smoothness without additional hypothesis in \cref{lem:rec_diff_smooth}.
The latter case requires a careful decomposition of the expectation of \cref{eq:act_kernel} via two-dimensional Gaussian sampling to circumvent the non-differentiability points of $\phi$, yielding additional constraints on kernels $A$ for $\gK_{\phi}$ to preserve their smoothness in \cref{lem:rec_diff_non_smooth}; these constraints are typically verified in the case of neural networks with bias (cf.\ \cref{app:proof_differentiability_networks}).

In any case, we emphasize that these differentiability constraints may not be tight and are only sufficient conditions ensuring the smoothness of $\app{\gK_{\phi}}{A}$.

\begin{lemma}[$\gK_{\phi}$ with smooth $\phi$]
    \label{lem:rec_diff_smooth}
    Let $A\colon\R^n\times\R^n\rightarrow\R$ be a symmetric positive semi-definite kernel and $\phi\colon\R\rightarrow\R$.
    We suppose that $\phi$ is an activation function following \cref{hyp:discr_archi,hyp:discr_archi_smooth_act}; in particular, $\phi$ is smooth.

    Let $y \in \R^n$ and $U$ be an open subset of $\R^n$ such that $x \mapsto \app{A}{x, x}$ and $x \mapsto \app{A}{x, y}$ are infinitely differentiable over $U$.
    Then, $x \mapsto \app{\app{\gK_{\phi}}{A}}{x, x}$ and $x \mapsto \app{\app{\gK_{\phi}}{A}}{x, y}$ are infinitely differentiable over $U$ as well.
\end{lemma}
\begin{proof}
    In order to prove the smoothness results over the open set $U$, it suffices to prove them on any open bounded subset of $U$.
    Let then $V \subseteq U$ be an open bounded set.
    Without loss of generality, we can assume that its closure $\cl V$ is also included in $U$.

    We define $B_1$ and $B_2$ from \cref{eq:act_kernel} as follows, for all $x \in V$:
    \begin{align}
        \label{eq:k_phi_app}
        \app{B_1}{x} \triangleq \app{\app{\gK_{\phi}}{A}}{x, y} = \mathbb{E}_{f\sim \app{\mathcal{GP}}{0,A}}\brackets*{\app{\phi}{\app{f}{x}}\app{\phi}{\app{f}{y}}}, && \app{B_2}{x} \triangleq \app{\app{\gK_{\phi}}{A}}{x, x} = \mathbb{E}_{f\sim \app{\mathcal{GP}}{0,A}}\brackets*{\app{\phi}{\app{f}{x}}^2}.
    \end{align}
    In the previous expressions, \cref{lem:gp_as_diff} tells us that we can take $f$ to be $\mathcal{C}^\infty$ over $\cl V$ with probability one.
    Hence, $B_1$ and $B_2$ are expectations of smooth functions over $V$.
    We seek to apply the dominated convergence theorem to prove that $B_1$ and $B_2$ are, in turn, smooth over $V$.
    To this end, we prove in the following the integrability of the derivatives of their integrands.

    Let $\alpha=\parentheses*{\alpha_1,\ldots,\alpha_n}\in\mathbb{N}^n$.
    Using the usual notations for multi-indexed partial derivatives, via a multivariate Faà di Bruno formula \citep{Leipnik2007}, we can write the derivatives $\partial^{\alpha} \parentheses*{\psi\circ f}$ at $x\in V$ for $\psi \in \braces*{\phi, \phi^2}$ as a weighted sum of terms of the form:
    \begin{equation}
        \app{\psi^{\parentheses*{j}}}{\app{f}{x}}\app{g_1}{x}\cdots \app{g_N}{x},
    \end{equation}
    where the $g_i$s are partial derivatives of $f$ at $x$.
    As $A$ is $\mathcal{C}^\infty$ over $V$, each of the $g_i$s is thus a GP with a $\mathcal{C}^\infty$ covariance function by \cref{lem:gp_as_diff}.
    We can also write for all $x\in V$:
    \begin{equation}
        \begin{aligned}
            \lrverts*{\app{\psi^{\parentheses*{j}}}{\app{f}{x}}\app{g_1}{x}\cdots \app{g_N}{x}} & \leq \sup_{z\in \cl V} \lrverts*{\app{\psi^{\parentheses*{j}}}{\app{f}{z}}\app{g_1}{z}\cdots \app{g_N}{z}} \\
            & \leq \sup_{z_0\in \cl V} \lrverts*{\app{\psi^{\parentheses*{j}}}{\app{f}{z_0}}} \sup_{z_1\in \cl V}\lrverts*{\app{g_1}{z_1}} \cdots \sup_{z_N\in \cl V} \lrverts*{\app{g_N}{z_N}}.
        \end{aligned}
    \end{equation}

    For each $i$, because the covariance function of $g_i$ is smooth over the compact set $\cl V$, its variance admits a maximum in $\cl V$ and we take $\sigma^2_i$ the double of its value.
    We then know from \citet{Adler1990}, that there is an $M_i$ such that:
    \begin{equation}
        \label{eq:integrability_1}
        \forall m \in \sN, \mathbb{E}_{f\sim \app{\mathcal{GP}}{0,A}}\brackets*{\sup_{z_i\in \cl V} \lrverts*{\app{g_i}{z_i}}^m} \leq M^m_i \mathbb{E} \lrverts*{Y_i}^m,
    \end{equation}
    where $Y_i$ is a Gaussian distribution which variance is $\sigma_i^2$, the right-hand side thus being finite.

    We also have, by \cref{hyp:discr_archi} from \cref{sec:assumptions}, that:
    \begin{equation}
        \label{eq:integrability_2}
        \sup_{z\in \cl V} \lrverts*{\app{\phi^{\parentheses*{j}}}{\app{f}{z}}}^2 \leq \sup_{z\in \cl V} \parentheses*{\lambda_1^{\parentheses*{j}}\lrverts*{\app{f}{z}}+\lambda_2^{\parentheses*{j}}}^2,
    \end{equation}
    which is shown to be integrable over $f$ by the same arguments as for the $g_i$s.
    Moreover, the Faà di Bruno formula decomposes $\psi^{\parentheses*{j}}$ when $\psi = \phi^2$ as a weighted sum of terms of the form $\phi^{\parentheses*{l}}\phi^{\parentheses*{l'}}$ with $l, l' \in \sN$.
    Therefore, thanks to similar arguments, for any $\psi\in\braces*{\phi, \phi^2}$:
    \begin{equation}
        \mathbb{E}_{f\sim \app{\mathcal{GP}}{0,A}} \brackets*{\sup_{z\in \cl V} \lrverts*{\app{\psi^{\parentheses*{j}}}{\app{f}{z}}}^2} < \infty.
    \end{equation}

    Now, by using the Cauchy-Schwarz inequality, we have that:
    \begin{equation}
        \begin{multlined}
            \mathbb{E}\brackets*{\sup_{z_0\in \cl V} \lrverts*{\app{\psi^{\parentheses*{j}}}{\app{f}{z_0}}} \sup_{z_1\in \cl V}\lrverts*{\app{g_1}{z_1}} \cdots \sup_{z_N\in \cl V} \lrverts*{\app{g_N}{z_N}}} \\
            \leq \sqrt{\mathbb{E}\brackets*{\sup_{z_0\in \cl V} \lrverts*{\app{\psi^{\parentheses*{j}}}{\app{f}{z_0}}}^2}}\sqrt{\mathbb{E}\brackets*{\sup_{z_1\in \cl V}\lrverts*{\app{g_1}{z_1}}^2 \cdots \sup_{z_N\in \cl V} \lrverts*{\app{g_N}{z_N}}^2}}.
        \end{multlined}
    \end{equation}
    By iterated applications of the Cauchy-Schwarz inequality and using the previous arguments, we can then show that:
    \begin{equation}
        \sup_{z_0\in \cl V} \lrverts*{\app{\psi^{\parentheses*{j}}}{\app{f}{z_0}}} \sup_{z_1\in \cl V}\lrverts*{\app{g_1}{z_1}} \cdots \sup_{z_N\in \cl V} \lrverts*{\app{g_N}{z_N}}
    \end{equation}
    is integrable over $f$.
    Additionally, note that by the same arguments for the case of $\psi = \phi$, a multiplication by $\app{\phi}{\app{f}{y}}$ preserves this integrability.

    We can then write for all $x \in V$, by a standard corollary of the dominated convergence theorem:
    \begin{align}
        \partial^{\alpha} \app{B_1}{x} = \mathbb{E}_{f\sim \app{\mathcal{GP}}{0,A}}\brackets*{\partial^{\alpha}\rightvert*{\parentheses*{\phi\circ f}}_{x}\app{\phi}{\app{f}{y}}}, && \partial^{\alpha} \app{B_2}{x} = \mathbb{E}_{f\sim \app{\mathcal{GP}}{0,A}}\brackets*{\partial^{\alpha}\rightvert*{\parentheses*{\phi^2\circ f}}_{x}},
    \end{align}
    which shows that $B_1$ and $B_2$ are $\mathcal{C}^\infty$ over $V$.
    This in turn means that $B_1$ and $B_2$ are $\mathcal{C}^\infty$ over $U$.
\end{proof}

\begin{lemma}[$\gK_{\phi}$ with piecewise smooth $\phi$]
    \label{lem:rec_diff_non_smooth}
    Let $A\colon\R^n\times\R^n\rightarrow\R$ be a symmetric positive semi-definite kernel and $\phi\colon\R\rightarrow\R$.
    We suppose that $\phi$ is an activation function following \cref{hyp:discr_archi,hyp:discr_archi_bias} (cf.\ \cref{sec:assumptions}).
    Let us define the matrix $\Sigma_A^{x,y}$ as:
    \begin{equation}
        \Sigma_A^{x,y} \triangleq \begin{pmatrix} \app{A}{x,x} & \app{A}{x,y} \\ \app{A}{x,y} & \app{A}{y,y} \end{pmatrix}.
    \end{equation}

    Let $y \in \R^n$ and $U$ be an open subset of $\R^n$ such that $x \mapsto \app{A}{x, x}$ and $x \mapsto \app{A}{x, y}$ are infinitely differentiable over $U$.
    Then, $x \mapsto \app{\app{\gK_{\phi}}{A}}{x, x}$ and $x \mapsto \app{\app{\gK_{\phi}}{A}}{x, y}$ are infinitely differentiable over $U' \triangleq \midbracesx*{x \in U}{\Sigma_A^{x,y} \text{ is invertible}}$.
\end{lemma}
\begin{proof}
    Since $\det\Sigma_A^{x,y}$ is smooth over $U$ and $U' = \midbracesx*{x \in U}{\det\Sigma_A^{x,y} > 0}$, $U'$ is an open subset of $U$.
    Hence, similarly to the proof of \cref{lem:rec_diff_smooth}, it suffices to prove the smoothness of $B_1$ and $B_2$ defined in \cref{eq:k_phi_app} on any open bounded subset of $U'$.
    Let then $V \subseteq \R^n$ be an open bounded set such that $\cl V \subseteq U'$.
    Note that $\det\Sigma_A^{x,y} > 0$ implies that $\app{A}{x, x} > 0$ and $\app{A}{y, y} > 0$.

    We will conduct in the following the proof that $B_1$ is smooth over $V$.
    Like in the proof of \cref{lem:rec_diff_smooth}, the smoothness of $B_2$ follows the same reasoning with little adaptation; in particular, it relies on the fact that $\app{A}{x, x} > 0$ for all $x \in U'$, making its square root smooth over $\cl V$.

    Since the dominated convergence theorem cannot be directly applied from \cref{eq:k_phi_app} because of $\phi$'s potential non-differentiability points $D$, let us decompose its expression for all $x \in U'$:
    \begin{align}
        \app{B_1}{x} & = \mathbb{E}_{f\sim \app{\mathcal{GP}}{0,A}}\brackets*{\app{\phi}{\app{f}{x}}\app{\phi}{\app{f}{y}}} = \mathbb{E}_{\parentheses*{z, z'} \sim \app{\gN}{\parentheses*{0, 0}, \Sigma_A^{x,y}}} \brackets*{\app{\phi}{z} \app{\phi}{z'}} \\
        & = \mathbb{E}_{z' \sim \app{\gN}{0, \app{A}{y, y}}} \brackets*{\app{\phi}{z'} \mathbb{E}_{z \sim \app{\gN}{\frac{\app{A}{x,y}}{\app{A}{y,y}}z', \app{A}{x, x} - \frac{\app{A}{x, y}^2}{\app{A}{y, y}}}} \brackets*{\app{\phi}{z}}} \\
        & = \mathbb{E}_{z' \sim \app{\gN}{0, \app{A}{y, y}}} \brackets*{\app{\phi}{z'} \app{h}{z', x}},
    \end{align}
    where $h$ is defined as:
    \begin{align}
        \label{eq:k_phi_before_separation}
        \app{h}{z', x} \triangleq \int_{-\infty}^{+\infty} \app{\phi}{z} \cdot \frac{1}{\app{\sigma}{x}\sqrt{2\pi}} \erm^{-\frac{1}{2}\parentheses*{\frac{z - \app{\mu}{x}z'}{\app{\sigma}{x}}}^2} \dif z, && \app{\mu}{x} = \frac{\app{A}{x,y}}{\app{A}{y,y}}, && \app{\sigma}{x} = \sqrt{\frac{\det\Sigma_A^{x,y}}{\app{A}{y,y}}}.
    \end{align}
    Now, if $D = \braces*{c_1, \ldots, c_L}$ with $L \in \sN$ and $c_1 < \cdots < c_L$, the $c_l$s constitute the non-differentiability points of $\phi$; we can then decompose the integration of $\phi$ in \cref{eq:k_phi_before_separation} as a sum of $L+1$ integrals with differentiable integrands, using $c_0 = -\infty$ and $c_{L+1} = +\infty$:
    \begin{equation}
        \app{h}{\varepsilon, x} = \frac{1}{\sqrt{2\pi}} \sum_{l = 0}^L \int_{c_l}^{c_{l+1}} \frac{\app{\phi}{z}}{\app{\sigma}{x}} \erm^{-\frac{1}{2}\parentheses*{\frac{z - \app{\mu}{x}z'}{\app{\sigma}{x}}}^2} \dif z.
    \end{equation}
    Therefore, it remains to show the smoothness of all applications $B_{1, l}$ for $l \in \lrbrackets*{0, L}$ defined as:
    \begin{equation}
        \label{eq:k_phi_after_separation}
        \app{B_{1, j}}{x} = \mathbb{E}_{z' \sim \app{\gN}{0, \app{A}{y, y}}} \brackets*{\int_{c_l}^{c_{l+1}} \frac{\app{\phi}{z'} \app{\phi}{z}}{\app{\sigma}{x} \sqrt{2\pi}} \erm^{-\frac{1}{2}\parentheses*{\frac{z - \app{\mu}{x}z'}{\app{\sigma}{x}}}^2} \dif z}.
    \end{equation}

    The rest of this proof unfolds similarly to the one of \cref{lem:rec_diff_smooth}.
    Indeed, the integrand of \cref{eq:k_phi_after_separation} is smooth over $\cl V$.
    There remains to show that all derivatives of this integrand are dominated by an integrable function of $z$ and $z'$.
    Consider the following integrand:
    \begin{equation}
        \app{\iota}{z, z', x} = \frac{\app{\phi}{z'} \app{\phi}{z}}{\app{\sigma}{x} \sqrt{2\pi}} \erm^{-\frac{1}{2}\parentheses*{\frac{z - \app{\mu}{x}z'}{\app{\sigma}{x}}}^2}.
    \end{equation}
    By applying the multivariate Faà di Bruno formula and noticing that $\sigma$ and $\mu$ are smooth over the closed set $\cl V$, we know that the derivatives of $\app{\iota}{z, z', x}$ with respect to $x$ for any derivation order are weighted sums of terms of the form:
    \begin{equation}
        z^k {z'}^{k'} \app{\kappa}{x} \frac{\app{\phi}{z'} \app{\phi}{z}}{\app{\sigma}{x} \sqrt{2\pi}} \erm^{-\frac{1}{2}\parentheses*{\frac{z - \app{\mu}{x}z'}{\app{\sigma}{x}}}^2},
    \end{equation}
    where $\kappa$ is a smooth function over $\cl V$ and $k, k' \in \sN$.
    Moreover, because $\sigma$, $\mu$ and $\kappa$ are smooth over the closed set $\cl V$ with positive values for $\sigma$, there are constants $a_1$, $a_2$ and $a_3$ such that:
    \begin{equation}
        \lrverts*{z^k {z'}^{k'} \app{\kappa}{x} \frac{\app{\phi}{z'} \app{\phi}{z}}{\app{\sigma}{x} \sqrt{2\pi}} \erm^{-\frac{1}{2}\parentheses*{\frac{z - \app{\mu}{x}z'}{\app{\sigma}{x}}}^2}} \leq \lrverts*{z^k {z'}^{k'} \app{\phi}{z'} \app{\phi}{z}} a_3 \erm^{-\frac{1}{2}\parentheses*{\frac{z - a_1 z'}{a_2}}^2},
    \end{equation}
    which is integrable over $z$ via \cref{hyp:discr_archi,eq:phi_bound}.
    Finally, let us notice that for some constants $b_1$, $b_2$ and $b_3$:
    \begin{equation}
        \int_{c_l}^{c_{l+1}} \lrverts*{z^k {z'}^{k'} \app{\phi}{z'} \app{\phi}{z}} a_3 \erm^{-\frac{1}{2}\parentheses*{\frac{z - a_1 z'}{a_2}}^2} \leq b_1 \E_{z \sim \app{\gN}{b_2 z', b_3}} \lrverts*{z^k {z'}^{k'} \app{\phi}{z'} \app{\phi}{z}},
    \end{equation}
    which is also integrable with respect to $\E_{z' \sim \app{\gN}{0, \app{A}{y, y}}}$ by similar arguments (see also the integrability of \cref{,eq:integrability_1} in \cref{lem:rec_diff_smooth}).
    This concludes the proof of integrability required to apply the dominated convergence theorem, allowing us to conclude about the smoothness of all $B_{1, j}$ and, in turn, of $B_1$ over $U'$.
\end{proof}

\begin{remark}[Relaxed condition for smoothness]
    \label{rk:relax_invertibility}
    The invertibility condition of \cref{lem:rec_diff_non_smooth} is actually stronger than needed: it suffices to assume that the rank of $\Sigma_A^{x,y}$ remains constant in a neighborhood of $x$.
\end{remark}

\subsubsection{Differentiability of Conjugate Kernels, NTKs and Discriminators}
\label{app:proof_differentiability_networks}

From the previous lemmas, we can then prove the results of \cref{sec:differentiability}.
We start by demonstrating the smoothness of the conjugate kernel for dense networks, and conclude in consequence about the smoothness of the NTK and trained network.

\begin{lemma}[Differentiability of the conjugate kernel]
    \label{lem:conjugate_diff}
    Let $k_{\mathrm{c}}$ be the conjugate kernel \citep{Lee2018b} of an infinite-width dense non-residual architecture such as in \cref{hyp:discr_archi}.
    For any $y \in \R^n$, the following holds for $A \in \braces*{k_{\mathrm{c}}, \app{\gK_{\phi'}}{k_{\mathrm{c}}}}$:
    \begin{itemize}
        \item if \cref{hyp:discr_archi_smooth_act} holds, then $x \mapsto \app{A}{x,x}$ and $x \mapsto \app{A}{x,y}$ are smooth everywhere over $\R^n$;
        \item if \cref{hyp:discr_archi_bias} holds, then $x \mapsto \app{A}{x,x}$ and $x \mapsto \app{A}{x,y}$ are smooth over an open set whose complement has null Lebesgue measure.
    \end{itemize}
\end{lemma}
\begin{proof}
    We define the following kernel:
    \begin{equation}
        \app{C_L^\phi}{x,y} = \mathbb{E}_{f\sim \app{\gG\gP}{0,C_{L-1}^\phi}}\brackets*{\app{\phi}{\app{f}{x}}\app{\phi}{\app{f}{y}}} + \beta^2 = \app{\gK_{\phi}}{C_{L-1}^\phi} + \beta^2,
    \end{equation}
    with:
    \begin{equation}
        \app{C_0^\phi}{x,y} = \frac{1}{n} x^\top y + \beta^2.
    \end{equation}
    We have that $k_{\mathrm{c}} = C_L^\phi$, with $L$ being the number of hidden layers in the network.
    Therefore, \cref{lem:rec_diff_smooth} ensures the smoothness result under \cref{hyp:discr_archi_smooth_act}.

    Let us now consider \cref{hyp:discr_archi_bias} (cf.\ the detailed assumption in \cref{sec:assumptions}); in particular, $\beta > 0$.
    We prove by induction over $L$ in the following that:
    \begin{itemize}
        \item $B_1\colon x \mapsto \app{C_L^\phi}{x, y}$ is smooth over $U = \midbracesx*{x\in\R^n}{\euclideannorm*{x} \neq \euclideannorm*{y}}$;
        \item $B_2\colon x \mapsto \app{C_L^\phi}{x, x}$ is smooth;
        \item for all $x, x' \in \R^n$ with $\euclideannorm*{x} \neq \euclideannorm*{x'}$, $\app{B_2}{x} \neq \app{B_2}{x'}$.
    \end{itemize}
    The result is immediate for $L = 0$.
    We now suppose that it holds for some $L\in\sN$ and prove that it also holds for $L+1$ hidden layers.
    Let us express $B_2$:
    \begin{equation}
        \app{B_2}{x} = \mathbb{E}_{\varepsilon\sim \app{\gN}{0, 1}}\brackets*{\app{\phi}{\varepsilon\sqrt{\app{C_L^\phi}{x,x} + \beta^2}}^2}.
    \end{equation}
    Using \cref{lem:rec_diff_non_smooth,rk:relax_invertibility}, the fact that $\beta > 0$ and the induction hypothesis ensures that $B_2$ is smooth.
    Moreover, \cref{hyp:discr_archi_bias}, in particular \cref{eq:hyp_relu}, allows us to assert that $\euclideannorm*{x} \neq \euclideannorm*{x'}$ implies $\app{B_2}{x} \neq \app{B_2}{x'}$.

    Finally, in order to apply \cref{lem:rec_diff_non_smooth} to prove the smoothness of $B_1$ over $U$, there remains to show that the following matrix is invertible:
    \begin{equation}
        \Sigma_\beta^{x,y} \triangleq \begin{pmatrix} \app{C_L^\phi}{x,x} + \beta^2 & \app{C_L^\phi}{x,y} + \beta^2 \\ \app{C_L^\phi}{x,y} + \beta^2 & \app{C_L^\phi}{y,y} + \beta^2 \end{pmatrix}.
    \end{equation}
    Let us compute its determinant:
    \begin{equation}
        \begin{aligned}
            \det \Sigma_\beta^{x,y} & = \parentheses*{\app{C_L^\phi}{x,x} + \beta^2}\parentheses*{\app{C_L^\phi}{y,y} + \beta^2} - \parentheses*{\app{C_L^\phi}{x,y} + \beta^2}^2 \\
            & = \det \Sigma_0^{x,y} + \beta^2 \parentheses*{\app{C_L^\phi}{x,x} + \app{C_L^\phi}{y,y} - 2\app{C_L^\phi}{x,y}}.
        \end{aligned}
    \end{equation}
    $C_L^\phi$ is a symmetric positive semi-definite kernel, thus:
    \begin{equation}
        \det \Sigma_\beta^{x,y} - \det \Sigma_0^{x,y} = \beta^2 \cdot \begin{pmatrix} 1 & -1 \end{pmatrix} \Sigma_0^{x,y} \begin{pmatrix} 1 \\ -1 \end{pmatrix} \geq 0.
    \end{equation}
    Hence, if $\det \Sigma_0^{x,y} > 0$, then $\det \Sigma_\beta^{x,y} > 0$.
    Besides, if $\det \Sigma_0^{x,y} = 0$, then:
    \begin{equation}
        \det \Sigma_\beta^{x,y} = \beta^2 \parentheses*{\sqrt{\app{B_2}{x}} - \sqrt{\app{B_2}{y}}}^2 > 0,
    \end{equation}
    for all $x \in U$.
    This proves that $B_1$ is indeed smooth over $U$, and concludes the induction.

    Note that $U$ is indeed an open set whose complement in $\R^n$ has null Lebesgue measure.
    Overall, the result is thus proved for $A = k_{\mathrm{c}}$; a similar reasoning using the previous induction result also transfers the result to $A = \app{\gK_{\phi'}}{k_{\mathrm{c}}}$.
\end{proof}

\begin{repprop}{prop:ntk_differentiability}[Differentiability of $k$]
    Let $k$ be the NTK of an infinite-width architecture following \cref{hyp:discr_archi}.
    For any $y \in \R^n$:
    \begin{itemize}
        \item if \cref{hyp:discr_archi_smooth_act} holds, then $\app{k}{\cdot,y}$ is smooth everywhere over $\R^n$;
        \item if \cref{hyp:discr_archi_bias} holds, then $\app{k}{\cdot,y}$ is smooth almost everywhere over $\R^n$, in particular over an open set whose complement has null Lebesgue measure.
    \end{itemize}
\end{repprop}
\begin{proof}
    According to the definitions of \citet{Jacot2018}, \citet{Arora2019} and \citet{Huang2020}, the smoothness of the kernel is guaranteed whenever the conjugate kernel $k_{\mathrm{c}}$ and its transform $\app{\gK_{\phi'}}{k_{\mathrm{c}}}$ are smooth; the result of \cref{lem:conjugate_diff} then applies.
    In the case of residual networks, there is a slight adaptation of the formula which does not change its regularity.
    Regarding convolutional networks, their conjugate kernels and NTKs involve finite combinations of such dense conjugate kernels and NTKs, thereby preserving their smoothness almost everywhere.
\end{proof}

\begin{reptheorem}{thm:diff_ntk_informal}[Differentiability of $f_t$]
    Let $f_t$ be a solution to \cref{eq:discr_ode} under \cref{hyp:finite_gamma,hyp:differentiable_a_b} by \cref{thm:discr_characterization}, with $k$ the NTK of an infinite-width neural network and $f_0$ an initialization of the latter.

    Then, under \cref{hyp:discr_archi,hyp:discr_archi_smooth_act}, $f_t$ is smooth everywhere.
    Under \cref{hyp:discr_archi,hyp:discr_archi_bias}, $f_t$ is smooth almost everywhere, in particular over an open set whose complement has null Lebesgue measure.
\end{reptheorem}
\begin{proof}
    From \cref{thm:discr_characterization}, we have:
    \begin{equation}
        f_t-f_0 = \app{\Tk{k}{\hgamma}}{\int_0^t\mgrad{\hgamma} \app{\Lo_{\halpha}}{f_s}\dif s}.
    \end{equation}
    We observe that $\app{\Tk{k}{\hgamma}}{h}$ has, for any $h \in \app{L^2}{\hgamma}$, a regularity which only depends on the regularity of $\app{k}{\cdot,y}$ for $y\in\supp{\hgamma}$.
    Indeed, if $\app{k}{\cdot,y}$ is smooth in a certain neighborhood $V$ for every such $y$, we can bound $\partial^{\alpha} \app{k}{\cdot,y}$ over $V$ for every $y$ and any multi-index $\alpha$ and then use dominated convergence to prove that $\app{\app{\Tk{k}{\hgamma}}{h}}{\cdot}$ is smooth over $V$.
    Therefore, the regularity of $\app{k}{\cdot, y}$ transfers to $f_t - f_0$.
    Given \cref{prop:ntk_differentiability}, there remains to prove the same result for $f_0$.

    The theorem then follows from the fact that $f_0$ has the same regularity as its conjugate kernel $k_{\mathrm{c}}$ thanks to \cref{lem:gp_as_diff} because $f_0$ is a sample from the GP with kernel $k_{\mathrm{c}}$.
    \cref{lem:conjugate_diff} shows the smoothness almost everywhere over an open set of applications $x \mapsto \app{k_{\mathrm{c}}}{x, y}$; to apply \cref{lem:gp_as_diff} and concludes this proof, this result must be generalized to prove the smoothness of $k_{\mathrm{c}}$ with respect to both its inputs.
    This can be done by generalizing the proofs of \cref{lem:rec_diff_smooth,lem:rec_diff_non_smooth} to show the smoothness of kernels with respect to both $x$ and $y$, with the same arguments than for $x$ alone.
\end{proof}

\begin{remark}
    In the previous theorem, $f_0$ is considered to be the initialization of the network.
    However, we highlight that, without loss of generality, this theorem encompasses the change of training distribution $\hgamma$ during GAN training.
    Indeed, as explained in \cref{sec:ntk_intro}, $f_0$ after $j$ steps of generator training can actually be decomposed as, for some $h_k \in \app{L^2}{\hgamma_k}$, $k \in \lrbrackets*{1, j}$:
    \begin{equation}
        f_0 = f^0 + \sum_{k=1}^j \app{\Tk{k}{\hgamma_k}}{h_k},
    \end{equation}
    by taking into account the updates of the discriminators over the whole GAN optimization process.
    The proof of \cref{thm:diff_ntk_informal} can then be applied similarly in this case by showing the differentiability of $f_0 - f^0$ on the one hand and of $f^0$, being the initialization of the discriminator at the very beginning of GAN training, on the other hand.
\end{remark}

\subsection{Dynamics of the Generated Distribution}
\label{app:dynamics_generated_proof}

We derive in this proposition the differential equation governing the dynamics of the generated distribution.

\begin{repprop}{prop:gen_dynamics}[Dynamics of $\alpha_{\tg}$]
    Under \cref{hyp:discr_archi,hyp:discr_archi_smooth_act}, \cref{eq:gen_descent} is well-posed.
    Let us consider its continuous-time version with discriminators trained on discrete distributions as described above:
    \begin{equation}
        \label{eq:gen_descent_continuous}
        \partial_{\tg} \theta_{\tg} = - \E_{z \sim p_{z}} \brackets*{\grad_\theta \app{g_{\tg}}{z}^{\top} \grad_{x} \rightvert*{\app{c_{f^\star_{\halpha_{g_{\tg}}}}}{x}}_{x = \app{g_{\tg}}{z}}}.
    \end{equation}
    This yields, with $k_{g_{\tg}}$ the NTK of the generator $g_{\tg}$:
    \begin{equation}
        \partial_{\tg} g_{\tg} = - \app{\Tk{k_{g_{\tg}}}{p_{z}}}{z \mapsto \grad_{x} \rightvert*{\app{c_{f^{\star}_{\halpha_{g_{\tg}}}}}{x}}_{x = \app{g_{\tg}}{z}}}.
    \end{equation}
    Equivalently, the following continuity equation holds for the joint distribution $\alpha^{z}_{\tg} \triangleq \parentheses*{\id, g_{\ell}}_\sharp p_{z}$:
    \begin{equation}
        \partial_{\tg} \alpha^{z}_{\tg} = - \grad_{x} \cdot \parentheses*{\alpha^{z}_{\tg} \app{\Tk{k_{g_{\tg}}}{p_{z}}}{z \mapsto \grad_{x} \rightvert*{\app{c_{f^{\star}_{\halpha_{g_{\tg}}}}}{x}}_{x = \app{g_{\tg}}{z}}}}.
    \end{equation}
\end{repprop}
\begin{proof}
    \cref{hyp:discr_archi,hyp:discr_archi_smooth_act} ensure, via \cref{prop:ntk_differentiability,thm:diff_ntk_informal} that the trained discriminator is differentiable everywhere at all times, whatever the state of the generator.
    Therefore, \cref{eq:gen_descent} is well-posed.

    By following \citet[Equation~(5)]{Mroueh2019}'s reasoning on a similar equation, \cref{eq:gen_descent_continuous} yields the following generator dynamics for all inputs $z \in \R^d$:
    \begin{equation}
        \partial_{\tg} \app{g_{\tg}}{z} = - \E_{z' \sim p_{z}} \brackets*{\grad_{\theta_{\tg}} \app{g_{\tg}}{z}^\top \grad_{\theta_{\tg}} \app{g_{\tg}}{z'} \grad_{x} \rightvert*{\app{c_{f^{\star}_{\halpha_{g_{\tg}}}}}{x}}_{x = \app{g_{\tg}}{z'}}}.
    \end{equation}
    We recognize the NTK $k_{g_{\tg}}$ of the generator as:
    \begin{equation}
        \app{k_{g_{\tg}}}{z, z'} \triangleq \grad_{\theta_{\tg}} \app{g_{\tg}}{z}^\top \grad_{\theta_{\tg}} \app{g_{\tg}}{z'}.
    \end{equation}
    From this, we obtain the dynamics of the generator:
    \begin{equation}
        \partial_{\tg} g_{\tg} = - \app{\Tk{k_{g_{\tg}}}{p_{z}}}{z \mapsto \grad_{x} \rightvert*{\app{c_{f^{\star}_{\halpha_{g_{\tg}}}}}{x}}_{x = \app{g_{\tg}}{z}}}.
    \end{equation}

    In other words, the transported particles $\parentheses*{z, \app{g_{\tg}}{z}}$ have trajectories $X_{\tg}$ which are solutions of the Ordinary Differential Equation (ODE):
    \begin{equation}
        \dod{X_{\tg}}{\tg} = \parentheses*{0, \app{v_{\tg}}{X_{\tg}}},
    \end{equation}
    where:
    \begin{equation}
        v_{\tg}=-\app{\Tk{k_{g_{\tg}}}{p_{z}}}{z \mapsto \grad_{x} \rightvert*{\app{c_{f^{\star}_{\halpha_{g_{\tg}}}}}{x}}_{x = \app{g_{\tg}}{z}}}.
    \end{equation}
    Then, because $\alpha^{z}_{\tg} \triangleq \parentheses*{\id, g_{\ell}}_\sharp p_{z}$ is the induced transported density, following \citet{Ambrosio2014}, whenever the ODE above is well-defined and has unique solutions~(which is necessarily the case for any trained $g$), $\alpha^{z}_{\tg}$ verifies the continuity equation with the velocity field $v_{\tg}$:
    \begin{equation}
        \begin{aligned}
            \partial_{\tg} \alpha^{z}_{\tg} & = - \grad_{z, x} \cdot \parentheses*{\alpha^{z}_{\tg} \parentheses*{0, -\app{\Tk{k_{g_{\tg}}}{p_{z}}}{z \mapsto \grad_{x} \rightvert*{\app{c_{f^{\star}_{\halpha_{g_{\tg}}}}}{x}}_{x = \app{g_{\tg}}{z}}}}} \\
            & = \grad_{x} \cdot \parentheses*{\alpha^{z}_{\tg} \app{\Tk{k_{g_{\tg}}}{p_{z}}}{z \mapsto \grad_{x} \rightvert*{\app{c_{f^{\star}_{\halpha_{g_{\tg}}}}}{x}}_{x = \app{g_{\tg}}{z}}}}.
        \end{aligned}
    \end{equation}
    This yields the desired result.
\end{proof}

\subsection{Optimality in Concave Setting}
\label{app:optimality_concave}

We derive an optimality result for concave bounded loss functions of the discriminator and positive definite kernels.

\subsubsection{Assumptions}

We first assume that the NTK is positive definite over the training dataset.
\begin{assumption}[Positive definite kernel]
    \label{hyp:positive_definite_ntk}
    $k$ is positive definite over $\hgamma$.
\end{assumption}
This positive definiteness property equates for finite datasets to the invertibility of the mapping
\begin{equation}
    \label{eq:inv_integral_operator}
    \begin{aligned}
        \rightvert*{\Tk{k}{\hgamma}}_{\supp \hgamma}\colon \app{L^2}{\hgamma} & \to \app{L^2}{\hgamma} \\
        h & \mapsto \rightvert*{\app{\Tk{k}{\hgamma}}{h}}_{\supp \hgamma}
    \end{aligned}
    ,
\end{equation}
that can be seen as a multiplication by the invertible Gram matrix of $k$ over $\hgamma$.
We further discuss this hypothesis in \cref{app:characteristic_ntk}.

We also assume the following properties on the discriminator loss function.
\begin{assumption}[Concave loss]
    \label{hyp:concave_bounded_loss}
    $\Lo_{\halpha}$ is concave and bounded from above, and its supremum is reached on a unique point $y^\star$ in $\app{L^2}{\hgamma}$.
\end{assumption}
Moreover, we need for the sake of the proof a uniform continuity assumption on the solution to \cref{eq:discr_ode}.
\begin{assumption}[Solution continuity]
    \label{hyp:uniformly_continuous_solution}
    $t \mapsto \rightvert*{f_t}_{\supp \hgamma}$ is uniformly continuous over $\R_+$.
\end{assumption}
Note that these assumptions are verified in the case of LSGAN, which is the typical application of the optimality results that we prove in the following.

\subsubsection{Optimality Result}

\begin{prop}[Asymptotic optimality]
    \label{prop:optimality_concave}
    Under \cref{hyp:finite_gamma,hyp:kernel,hyp:differentiable_a_b,hyp:concave_bounded_loss,hyp:positive_definite_ntk,hyp:uniformly_continuous_solution}, $f_t$ converges pointwise when $t \to \infty$, and:
    \begin{align}
        \app{\Lo_{\halpha}}{f_t} \xrightarrow[t \to \infty]{} \app{\Lo_{\halpha}}{y^\star}, && f_\infty = f_0 + \app{\Tk{k}{\hgamma}}{\app{\rightvert*{\Tk{k}{\hgamma}}_{\supp \hgamma}^{-1}}{y^\star - \rightvert*{f_0}_{\supp \hgamma}}}, && \rightvert*{f_\infty}_{\supp \hgamma} = y^\star,
    \end{align}
    where we recall that:
    \begin{equation}
        y^\star = \argmax_{y \in \app{L^2}{\hgamma}} \app{\Lo_{\halpha}}{y}.
    \end{equation}
\end{prop}

This result ensures that, for concave losses such as LSGAN, the optimum for $\Lo_{\halpha}$ in $\app{L^2}{\Omega}$ is reached for infinite training times by neural network training in the infinite-width regime when the NTK of the discriminator is positive definite.
However, this also provides the expression of the optimal network outside $\supp \hgamma$ thanks to the smoothing of $\hgamma$.

In order to prove this proposition, we need the following intermediate results: the first one about the functional gradient of $\Lo_\halpha$ on the solution $f_t$; the second one about a direct application of positive definite kernels showing that one can retrieve $f \in \gH_k^\hgamma$ over all $\Omega$ from its restriction to $\supp \hgamma$.

\begin{lemma}
    \label{lem:zero_gradient}
    Under \cref{hyp:finite_gamma,hyp:kernel,hyp:differentiable_a_b,hyp:concave_bounded_loss,hyp:positive_definite_ntk,hyp:uniformly_continuous_solution}, $\mgrad{\hgamma} \app{\Lo_\halpha}{f_t} \to 0$ when $t \to \infty$.
    Since $\supp \hgamma$ is finite, this limit can be interpreted pointwise.
\end{lemma}
\begin{proof}
    \cref{hyp:finite_gamma,hyp:kernel,hyp:differentiable_a_b} ensure the existence and uniqueness of $f_t$, by \cref{thm:discr_characterization}.

    $t \mapsto \hat{f}_t \triangleq \rightvert*{f_t}_{\supp \hgamma}$ and $\Lo_\halpha$ being differentiable, $t \mapsto \app{\Lo_\halpha}{f_t}$ is differentiable, and:
    \begin{equation}
        \partial_t \app{\Lo_\halpha}{f_t} = \sprod{\mgrad{\hgamma} \app{\Lo_\halpha}{f_t}}{ \partial_t \hat{f}_t}_{\app{L^2}{\hgamma}} = \sprod{\mgrad{\hgamma} \app{\Lo_\halpha}{f_t}}{\app{\Tk{k}{\hgamma}}{\mgrad{\hgamma} \app{\Lo_{\halpha}}{f_t}}}_{\app{L^2}{\hgamma}},
    \end{equation}
    using \cref{eq:discr_ode}.
    This equates to:
    \begin{equation}
        \label{eq:loss_evolution}
        \partial_t \app{\Lo_\halpha}{f_t} = \euclideannorm*{\app{\Tk{k}{\hgamma}}{\mgrad{\hgamma} \app{\Lo_\halpha}{f_t}}}_{\gH_k^\hgamma}^2 \geq 0,
    \end{equation}
    where $\euclideannorm*{\cdot}_{\gH_k^\hgamma}$ is the semi-norm associated to the RKHS $\gH_k^\hgamma$.
    Note that this semi-norm is dependent on the restriction of its input to $\supp \hgamma$ only.
    Therefore, $t \mapsto \app{\Lo_\halpha}{f_t}$ is increasing.
    Since $\Lo_\halpha$ is bounded from above, $t \mapsto \app{\Lo_\halpha}{f_t}$ admits a limit when $t \to \infty$.

    We now aim at proving from the latter fact that $\partial_t \app{\Lo_\halpha}{f_t} \to 0$ when $t \to \infty$.
    We notice that $\euclideannorm*{\cdot}_{\gH_k^\hgamma}^2$ is uniformly continuous over $\app{L^2}{\hgamma}$ since $\supp \hgamma$ is finite, $\mgrad{\hgamma} \Lo_\halpha$ is uniformly continuous over $\app{L^2}{\hgamma}$ since $a'$ and $b'$ are Lipschitz-continuous, $\rightvert*{\Tk{k}{\hgamma}}_{\supp \hgamma}$ is uniformly continuous as it amounts to a finite matrix multiplication, and \cref{hyp:uniformly_continuous_solution} gives that $t \mapsto \rightvert*{f_t}_{\supp \hgamma}$ is uniformly continuous over $\R_+$.
    Therefore, their composition $t \mapsto \partial_t \app{\Lo_\halpha}{f_t}$ (from \cref{eq:loss_evolution}) is uniformly continuous over $\R_+$.
    Using Barbălat’s Lemma \citep{Farkas2016}, we conclude that $\partial_t \app{\Lo_\halpha}{f_t} \to 0$ when $t \to \infty$.

    Furthermore, $k$ is positive definite over $\hgamma$ by \cref{hyp:positive_definite_ntk}, so $\euclideannorm*{\cdot}_{\gH_k^\hgamma}$ is actually a norm.
    Therefore, since $\supp \hgamma$ is finite, the following pointwise convergence holds:
    \begin{equation}
        \mgrad{\hgamma} \app{\Lo_\halpha}{f_t} \xrightarrow[t \to \infty]{} 0.
    \end{equation}
\end{proof}

\begin{lemma}[$\gH_k^\hgamma$ determined by $\supp \hgamma$]
    \label{lem:inv_tk}
    Under \cref{hyp:finite_gamma,hyp:kernel,hyp:positive_definite_ntk}, for all $f \in \gH_k^\hgamma$, the following holds:
    \begin{equation}
        f = \app{\Tk{k}{\hgamma}}{\app{\rightvert*{\Tk{k}{\hgamma}}_{\supp \hgamma}^{-1}}{\rightvert*{f}_{\supp \hgamma}}}.
    \end{equation}
\end{lemma}
\begin{proof}
    Since $k$ is positive definite by \cref{hyp:positive_definite_ntk}, then $\rightvert*{\Tk{k}{\hgamma}}_{\supp \hgamma}$ from \cref{eq:inv_integral_operator} is invertible.
    Let $f \in \gH_k^\hgamma$.
    Then, by definition of the RKHS in \cref{def:rkhs}, there exists $h \in \app{L^2}{\hgamma}$ such that $f = \app{\Tk{k}{\hgamma}}{h}$.
    In particular, $\rightvert*{f}_{\supp \hgamma} = \app{\rightvert*{\Tk{k}{\hgamma}}_{\supp \hgamma}}{h}$, hence $h = \app{\rightvert*{\Tk{k}{\hgamma}}_{\supp \hgamma}^{-1}}{\rightvert*{f}_{\supp \hgamma}}$.
\end{proof}

We can now prove the desired proposition.

\begin{proof}[Proof of \cref{prop:optimality_concave}]
    Let us first show that $f_t$ converges to the optimum $y^\star$ in $\app{L^2}{\hgamma}$.
    By applying \cref{lem:zero_gradient}, we know that $\mgrad{\hgamma} \app{\Lo_\halpha}{f_t} \to 0$ when $t \to \infty$.
    Given that the supremum of the differentiable concave function $\Lo_\halpha\colon \app{L^2}{\hgamma} \to \R$ is achieved at a unique point $y^\star \in \app{L^2}{\hgamma}$ with finite $\supp \hgamma$, then the latter convergence result implies that $\hat{f}_t \triangleq \rightvert*{f_t}_{\supp \hgamma}$ converges pointwise to $y^\star$ when $t \to \infty$.

    Given this convergence in $\app{L^2}{\hgamma}$, we can deduce convergence on the whole domain $\Omega$ by noticing that $f_t - f_0 \in \gH_k^\hgamma$, from \cref{cor:discr_rkhs}.
    Thus, using \cref{lem:inv_tk}:
    \begin{equation}
        f_t - f_0 = \app{\Tk{k}{\hgamma}}{\app{\rightvert*{\Tk{k}{\hgamma}}_{\supp \hgamma}^{-1}}{\rightvert*{\parentheses*{f_t - f_0}}_{\supp \hgamma}}}.
    \end{equation}
    Again, since $\supp \hgamma$ is finite, and $\rightvert*{\Tk{k}{\hgamma}}_{\supp \hgamma}^{-1}$ can be expressed as a matrix multiplication, the fact that $f_t$ converges to $y^\star$ over $\supp \hgamma$ implies that:
    \begin{equation}
        \app{\rightvert*{\Tk{k}{\hgamma}}_{\supp \hgamma}^{-1}}{\rightvert*{\parentheses*{f_t - f_0}}_{\supp \hgamma}} \xrightarrow[t \to \infty]{} \app{\rightvert*{\Tk{k}{\hgamma}}_{\supp \hgamma}^{-1}}{y^\star - \rightvert*{f_0}_{\supp \hgamma}}.
    \end{equation}
    Finally, using the definition of the integral operator in \cref{def:rkhs}, the latter convergence implies the following desired pointwise convergence:
    \begin{align}
        f_t \xrightarrow[t \to \infty]{} f_0 + \app{\Tk{k}{\hgamma}}{\app{\rightvert*{\Tk{k}{\hgamma}}_{\supp \hgamma}^{-1}}{y^\star - \rightvert*{f_0}_{\supp \hgamma}}}.
    \end{align}

    We showed at the beginning of this proof that $f_t$ converges to the optimum $y^\star$ in $\app{L^2}{\hgamma}$, so $\app{\Lo_{\halpha}}{f_t} \to \app{\Lo_{\halpha}}{y^\star}$ by continuity of $\Lo_{\halpha}$ as claimed in the proposition.
\end{proof}

\subsection{Case Studies of Discriminator Dynamics}
\label{app:ode_solve}

We study in the rest of this section the expression of the discriminators in the case of the IPM loss and LSGAN, as described in \cref{sec:case_analysis}, and of the original GAN formulation.

\subsubsection{Preliminaries}

We first need to introduce some definitions.

The presented solutions to \cref{eq:discr_ode} leverage a notion of functions of linear operators, similarly to functions of matrices \citep{Higham2008}.
We define such functions in the simplified case of non-negative symmetric compact operators with a finite number of eigenvalues, such as $\Tk{k}{\hgamma}$.

\begin{definition}[Linear operator]
    \label{def:operator_function}
    Let $\gA\colon \app{L^2}{\hgamma} \to \app{L^2}{\Omega}$ be a non-negative symmetric compact linear operator with a finite number of eigenvalues, for which the spectral theorem guarantees the existence of a countable orthonormal basis of eigenfunctions with non-negative eigenvalues.
    If $\varphi\colon \R_+ \to \R$, we define $\app{\varphi}{\gA}$ as the linear operator with the same eigenspaces as $\gA$, with their respective eigenvalues mapped by $\varphi$; in other words, if $\lambda$ is an eigenvalue of $\gA$, then $\app{\varphi}{\gA}$ admits the eigenvalue $\app{\varphi}{\lambda}$ with the same eigenspace.
\end{definition}

In the case where $\gA$ is a matrix, this amounts to diagonalizing $\gA$ and transforming its diagonalization elementwise using $\varphi$.
Note that $\Tk{k}{\hgamma}$ has a finite number of eigenvalues since it is generated by a finite linear combination of linear operators (see \cref{def:rkhs}).

We also need to define the following Radon–Nikodym derivatives with inputs in $\supp \hgamma$:
\begin{align}
    \label{eq:radon_nikodym}
    \rho = \dod{\parentheses*{\hbeta - \halpha}}{\parentheses*{\hbeta + \halpha}}, && \rho_1 = \dod{\halpha}{\hgamma}, && \rho_2 = \dod{\hbeta}{\hgamma},
\end{align}
knowing that
\begin{align}
    \rho = \frac{1}{2} \parentheses*{\rho_2 - \rho_1}, && \rho_1 + \rho_2 = 2.
\end{align}
These functions help us to compute the functional gradient of $\Lo_{\halpha}$, as follows.
\begin{lemma}[Loss derivative]
    \label{lem:gradient_loss}
    Under \cref{hyp:differentiable_a_b}:
    \begin{equation}
        \mgrad{\hgamma} \app{\Lo_{\halpha}}{f} = \rho_1 a'_f - \rho_2 b'_f = \rho_1 \cdot \parentheses*{a' \circ f} - \rho_2 \cdot \parentheses*{b' \circ f}.
    \end{equation}
\end{lemma}
\begin{proof}
    We have from \cref{eq:sup_discr}:
    \begin{equation}
        \label{eq:scalar_product_loss}
        \app{\Lo_{\halpha}}{f} = \E_{x \sim \halpha} \brackets*{\app{a_f}{x}} - \E_{y \sim \hbeta} \brackets*{\app{b_f}{y}} = \sprod{\rho_1}{a_f}_{\app{L^2}{\hgamma}} - \sprod{\rho_2}{b_f}_{\app{L^2}{\hgamma}},
    \end{equation}
    hence by composition:
    \begin{equation}
        \mgrad{\hgamma} \app{\Lo_{\halpha}}{f} = \rho_1 \cdot \parentheses*{a' \circ f} - \rho_2 \cdot \parentheses*{b' \circ f} = \rho_1 a'_f - \rho_2 b'_f.
    \end{equation}
\end{proof}

\subsubsection{LSGAN}

\begin{repprop}{prop:lsgan}[LSGAN discriminator]
    Under \cref{hyp:finite_gamma,hyp:kernel}, the solutions of \cref{eq:discr_ode} for $a = -\parentheses*{\id + 1}^2$ and $b = -\parentheses*{\id - 1}^2$ are  the functions defined for all $t \in \R_+$ as:
    \begin{equation}
        f_t = \app{\app{\exp}{-4 t \Tk{k}{\hgamma}}}{f_0 - \rho} + \rho = f_0 + \app{\app{\varphi_t}{\Tk{k}{\hgamma}}}{f_0 - \rho},
    \end{equation}
    where:
    \begin{equation}
        \varphi_t\colon x \mapsto \erm^{-4 t x} - 1.
    \end{equation}
\end{repprop}
\begin{proof}
    \cref{hyp:finite_gamma,hyp:kernel} are already assumed and \cref{hyp:differentiable_a_b} holds for the given $a$ and $b$ in LSGAN.
    Thus, \cref{thm:discr_characterization} applies, and there exists a unique solution $t \mapsto f_t$ to \cref{eq:discr_ode} over $\R_+$ in $\app{L^2}{\Omega}$ for a given initial condition $f_0$.
    Therefore, there remains to prove that, for a given initial condition $f_0$,
    \begin{equation}
        \label{eq:lsgan_sol_candidate}
        g\colon t \mapsto g_t = f_0 + \app{\app{\varphi_t}{\Tk{k}{\hgamma}}}{f_0 - \rho}
    \end{equation}
    is a solution to \cref{eq:discr_ode} with $g_0 = f_0$ and $g_t \in \app{L^2}{\Omega}$ for all $t \in \R_+$.

    Let us first express the gradient of $\Lo_{\halpha}$.
    We have from \cref{lem:gradient_loss}, with $a_f = -\parentheses*{f + 1}^2$ and $b_f = -\parentheses*{f - 1}^2$:
    \begin{equation}
        \mgrad{\hgamma} \app{\Lo_{\halpha}}{f} = \rho_1 a'_f - \rho_2 b'_f = -2 \rho_1 \parentheses*{f + 1} - 2 \rho_2 \parentheses*{f - 1} = 4 \rho - 4 f.
    \end{equation}
    So \cref{eq:discr_ode} equates to:
    \begin{equation}
        \label{eq:discr_ode_lsgan}
        \partial_t f_t = 4 \app{\Tk{k}{\hgamma}}{\rho - f_t}.
    \end{equation}

    Now let us prove that $g_t$ is a solution to \cref{eq:discr_ode_lsgan}.
    We have:
    \begin{equation}
        \label{eq:lsgan_sol_candidate_derivative}
        \partial_t g_t = -4 \app{\parentheses*{\Tk{k}{\hgamma} \circ \app{\exp}{- 4 t \Tk{k}{\hgamma}}}}{f_0 - \rho} = -4 \app{\parentheses*{\Tk{k}{\hgamma} \circ \app{\exp}{- 4 t \Tk{k}{\hgamma}}}}{f_0 - \rho}.
    \end{equation}
    Restricted to $\supp \hgamma$, we can write from \cref{eq:lsgan_sol_candidate}:
    \begin{equation}
        g_t = f_0 + \app{\parentheses*{\app{\exp}{- 4 t \rightvert*{\Tk{k}{\hgamma}}_{\supp \hgamma}} - \id_{\app{L^2}{\hgamma}}}}{f_0 - \rho},
    \end{equation}
    and plugging this in \cref{eq:lsgan_sol_candidate_derivative}:
    \begin{equation}
        \partial_t g_t = -4 \app{\Tk{k}{\hgamma}}{g_t - \rho},
    \end{equation}
    where we retrieve the differential equation of \cref{eq:discr_ode_lsgan}.
    Therefore, $g_t$ is a solution to \cref{eq:discr_ode_lsgan}.

    It is clear that $g_0 = f_0$.
    Moreover, $\Tk{k}{\hgamma}$ being decomposable in a finite orthonormal basis of elements of operators over $\app{L^2}{\Omega}$, its exponential has values in $\app{L^2}{\Omega}$ as well, making $g_t$ belong to $\app{L^2}{\Omega}$ for all $t$.
    With this, the proof is complete.
\end{proof}

\subsubsection{IPMs}

\begin{repprop}{prop:ipm_mmd}[IPM discriminator]
    Under \cref{hyp:finite_gamma,hyp:kernel}, the solutions of \cref{eq:discr_ode} for $a = b = \id$ are the functions of the form $f_t = f_0 + t f^\ast_{\halpha}$, where $f^\ast_{\halpha}$ is the unnormalized MMD witness function, yielding:
    \begin{align}
        f^\ast_{\halpha} = \E_{x \sim \halpha} \brackets*{\app{k}{x, \cdot}} - \E_{y \sim \hbeta} \brackets*{\app{k}{y, \cdot}}, && \app{\Lo_{\halpha}}{f_t} = \app{\Lo_{\halpha}}{f_0} + t \cdot \app{\MMD^2_{k}}{\halpha, \hbeta}.
    \end{align}
\end{repprop}
\begin{proof}
    \cref{hyp:finite_gamma,hyp:kernel} are already assumed and \cref{hyp:differentiable_a_b} holds for the given $a$ and $b$ of the IPM loss.
    Thus, \cref{thm:discr_characterization} applies, and there exists a unique solution $t \mapsto f_t$ to \cref{eq:discr_ode} over $\R_+$ in $\app{L^2}{\Omega}$ for a given initial condition $f_0$.
    Therefore, in order to find the solution of \cref{eq:discr_ode}, there remains to prove that, for a given initial condition $f_0$,
    \begin{equation}
        g\colon t \mapsto g_t = f_0 + t f^\ast_{\halpha}
    \end{equation}
    is a solution to \cref{eq:discr_ode} with $g_0 = f_0$ and $g_t \in \app{L^2}{\Omega}$ for all $t \in \R_+$.

    Let us first express the gradient of $\Lo_{\halpha}$.
    We have from \cref{lem:gradient_loss}, with $a_f = b_f = f$:
    \begin{equation}
        \mgrad{\hgamma} \app{\Lo_{\halpha}}{f} = \rho_1 a'_f - \rho_2 b'_f = - 2 \rho.
    \end{equation}
    So \cref{eq:discr_ode} equates to:
    \begin{equation}
        \partial_t f_t = - 2 \app{\Tk{k}{\hgamma}}{\rho} = 2 \int_x \app{k}{\cdot, x} \app{\rho}{x} \dif \app{\hgamma}{x} = \int_x \app{k}{\cdot, x} \dif \app{\halpha}{x} - \int_y \app{k}{\cdot, y} \dif \app{\hbeta}{y},
    \end{equation}
    by definition of $\rho$ (see \cref{eq:radon_nikodym}), yielding:
    \begin{equation}
        \label{eq:discr_ode_ipm}
        \partial_t f_t = f_\halpha^\ast.
    \end{equation}
    Clearly, $t \mapsto g_t = f_0 + t f_\halpha^\ast$ is a solution of the latter equation, $g_0 = f_0$ and $g_t \in \app{L^2}{\Omega}$ given that $\supp \hgamma$ is finite and $k \in \app{L^2}{\Omega^2}$ by assumption.
    The set of solutions for the IPM loss is thus characterized.

    Finally, let us compute $\app{\Lo_{\halpha}}{f_t}$.
    By linearity of $\Lo_{\halpha}$ for $a = b = \id$:
    \begin{equation}
        \app{\Lo_{\halpha}}{f_t} = \app{\Lo_{\halpha}}{f_0} + t \cdot \app{\Lo_{\halpha}}{f_\halpha^\ast} = \app{\Lo_{\halpha}}{f_0} + t \cdot \app{\Lo_{\halpha}}{\app{\Tk{k}{\hgamma}}{- 2 \rho}}.
    \end{equation}
    But, from \cref{eq:scalar_product_loss}, $\app{\Lo_{\halpha}}{f} = \sprod{- 2 \rho}{f}_{\app{L^2}{\hgamma}}$, hence:
    \begin{equation}
        \app{\Lo_{\halpha}}{f_t} = \app{\Lo_{\halpha}}{f_0} + t \cdot \sprod{- 2 \rho}{\app{\Tk{k}{\hgamma}}{- 2 \rho}}_{\app{L^2}{\hgamma}} = \app{\Lo_{\halpha}}{f_0} + t \cdot \euclideannorm*{\app{\Tk{k}{\hgamma}}{- 2 \rho}}_{\gH^\hgamma_k}^2.
    \end{equation}
    By noticing that $\app{\Tk{k}{\hgamma}}{- 2 \rho} = f_\halpha^\ast$ and that $\euclideannorm*{f_\halpha^\ast}_{\gH^\hgamma_k} = \app{\MMD_{k}}{\halpha, \hbeta}$ since $f^\ast_{\halpha}$ is the unnormalized MMD witness function, the expression of $\app{\Lo_{\halpha}}{f_t}$ in the proposition is obtained.
\end{proof}

\subsubsection{Vanilla GAN}

Unfortunately, finding the solutions to \cref{eq:discr_ode} in the case of the original GAN formulation, i.e.\ $a = \app{\log}{1 - \sigma}$ and $b = -\log \sigma$,
remains to the extent of our knowledge an open problem.
We provide in the rest of this section some leads that might prove useful for more advanced analyses.

Let us first determine the expression of \cref{eq:discr_ode} for vanilla GAN.
\begin{lemma}
    For $a = \app{\log}{1 - \sigma}$ and $b = -\log \sigma$, \cref{eq:discr_ode} equates to:
    \begin{equation}
        \label{eq:discr_ode_vanilla}
        \partial_t f_t = \app{\Tk{k}{\hgamma}}{\rho_2 - 2 \app{\sigma}{f}}.
    \end{equation}
\end{lemma}
\begin{proof}
    We have from \cref{lem:gradient_loss}, with $a_f = b_f = f$:
    \begin{equation}
        \mgrad{\hgamma} \app{\Lo_{\halpha}}{f} = \rho_1 a'_f - \rho_2 b'_f = - \rho_1 \frac{\app{\sigma'}{f}}{1 - \app{\sigma}{f}} + \rho_2 \frac{\app{\sigma'}{f}}{\app{\sigma}{f}}.
    \end{equation}
    By noticing that $\app{\sigma'}{f} = \app{\sigma}{f} \parentheses*{1 - \app{\sigma}{f}}$, we obtain:
    \begin{equation}
        \mgrad{\hgamma} \app{\Lo_{\halpha}}{f} = \rho_1 a'_f - \rho_2 b'_f = - \rho_1 \app{\sigma}{f} + \rho_2 \parentheses*{1 - \app{\sigma}{f}} = \rho_2 - 2 \app{\sigma}{f}.
    \end{equation}
    By plugging the latter expression in \cref{eq:discr_ode}, the desired result is achieved.
\end{proof}

Note that \cref{hyp:differentiable_a_b} holds for these choices of $a$ and $b$.
Therefore, under \cref{hyp:finite_gamma,hyp:kernel}, there exists a unique solution to \cref{eq:discr_ode_vanilla} in $\R_+ \to \app{L^2}{\Omega}$ with a given initialization $f_0$.

Let us first study \cref{eq:discr_ode_vanilla} in the simplified case of a one-dimensional ordinary differential equation.
\begin{prop}
    \label{prop:vanilla_1d}
    Let $r \in \braces*{0, 2}$ and $\lambda \in \R$.
    The set of differentiable solutions over $\R$ to this ordinary differential equation:
    \begin{equation}
        \label{eq:discr_ode_vanilla_1d}
        \partial_t y_t = \lambda \parentheses*{r - 2 \app{\sigma}{y_t}}
    \end{equation}
    is the following set:
    \begin{equation}
        \label{eq:solutions_vanilla_1d}
        S = \midbracesx*{y\colon t \mapsto  \parentheses*{1 - r} \parentheses*{\app{W}{\erm^{2 \lambda t + C}} - 2 \lambda t - C}}{C \in \R},
    \end{equation}
    where $W$ the is principal branch of the Lambert $W$ function \citep{Corless1996}.
\end{prop}
\begin{proof}
    The theorem of Cauchy-Lipschitz ensures that there exists a unique global solution to \cref{eq:discr_ode_vanilla_1d} for a given initial condition $y_0 \in \R$.
    Therefore, we only need to show that all elements of $S$ are solutions of \cref{eq:discr_ode_vanilla_1d} and that they can cover any initial condition.

    Let us first prove that $y\colon t \mapsto  \parentheses*{1 - r} \parentheses*{\app{W}{\erm^{2 \lambda t + C}} - 2 \lambda t - C}$ is a solution of \cref{eq:discr_ode_vanilla_1d}.
    Let us express the derivative of $y$:
    \begin{equation}
        \frac{1}{1 - r} \partial_t y_t = 2 \lambda \parentheses*{\erm^{2 \lambda t + C} \app{W'}{\erm^{2 \lambda t + C}} - 1}.
    \end{equation}
    $\app{W'}{z} = \frac{\app{W}{z}}{z \parentheses*{1 + \app{W}{z}}}$, so:
    \begin{equation}
        \frac{1}{1 - r} \partial_t y_t  = 2 \lambda \parentheses*{\frac{\app{W}{\erm^{2 \lambda t + C}}}{1 + \app{W}{\erm^{2 \lambda t + C}}} - 1} = - \frac{2 \lambda}{1 + \app{W}{\erm^{2 \lambda t + C}}}.
    \end{equation}
    Moreover, $\app{W}{z} = z \erm^{- \app{W}{z}}$, and with $r - 1 \in \braces*{1, -1}$:
    \begin{equation}
        \frac{1}{1 - r} \partial_t y_t  = - \frac{2 \lambda}{1 + \erm^{2 \lambda t + C}\erm^{-\app{W}{\erm^{2 \lambda t + C}}}} = - \frac{2 \lambda}{1 + \erm^{\parentheses*{r - 1} y_t}}.
    \end{equation}
    Finally, we notice that, since $r \in \braces*{0, 2}$:
    \begin{equation}
        \lambda \parentheses*{r - 2 \app{\sigma}{y_t}} = - \frac{2 \lambda \parentheses*{1 - r}}{1 + \erm^{\parentheses*{r - 1} y_t}}.
    \end{equation}
    Therefore:
    \begin{equation}
        \partial_t y_t  = \lambda \parentheses*{r - 2 \app{\sigma}{y_t}}
    \end{equation}
    and $y_t$ is a solution to \cref{eq:discr_ode_vanilla_1d}.

    Since $y_0 = \parentheses*{1 - r} \parentheses*{\app{W}{\erm^C} - C}$ and $z \mapsto \app{W}{\erm^z} - z$ can be proven to be bijective over $\R$, the elements of $S$ can cover any initial condition.
    With this, the result is proved.
\end{proof}

Suppose that $f_0 = 0$ in \cref{eq:discr_ode_vanilla} and that $\rho_2$ has values in $\braces*{0, 2}$ --~i.e.\ $\halpha$ and $\hbeta$ have disjoint supports (which is the typical case for distributions with finite support).
From \cref{prop:vanilla_1d}, a candidate solution would be:
\begin{equation}
    \label{eq:candidate_vanilla}
    f_t = \app{\app{\varphi_t}{x}}{\rho_2 - 1} = -\app{\app{\varphi_t}{x}}{\rho},
\end{equation}
where:
\begin{equation}
    \varphi_t\colon x \mapsto \app{W}{\erm^{2 t x + 1}} - 2 t x - 1,
\end{equation}
since the initial condition $y_0 = 0$ gives the constant value $C = 1$ in \cref{eq:solutions_vanilla_1d}.
Note that the Lambert $W$ function of a symmetric linear operator is well-defined, all the more so as we choose the principal branch of the Lambert function in our case; see the work of \citet{Corless2017} for more details.
Note also that the estimation of $\app{W}{\erm^z}$ is actually numerically stable using approximations from \citet{Iacono2017}.

However, \cref{eq:candidate_vanilla} cannot be a solution of \cref{eq:discr_ode_vanilla}.
Indeed, one can prove by following essentially the same reasoning as the proof of \cref{prop:vanilla_1d} that:
\begin{equation}
    \label{eq:derivative_candidate}
    \partial_t f_t = 2 \app{\parentheses*{\Tk{k}{\hgamma} \circ \parentheses*{\app{\psi_t}{\Tk{k}{\hgamma}}}^{-1}}}{\rho_2 - 1},
\end{equation}
with:
\begin{equation}
    \psi_t\colon x \mapsto 1 + \app{W}{\erm^{2 t x + 1}} > 0.
\end{equation}
However, this does not allow us to obtain \cref{eq:discr_ode_vanilla} since in the latter the sigmoid is taken coordinate-wise, where the exponential in \cref{eq:derivative_candidate} acts on matrices.

Nonetheless, for $t$ small enough, $f_t$ as defined in \cref{eq:derivative_candidate} should approximate the solution of \cref{eq:discr_ode_vanilla}, since sigmoid is approximately linear around $0$ and $f_t \approx 0$ when $t$ is small enough.
We find in practice that for reasonable values of $t$, e.g.\ $t \leq 5$, the approximate solution of \cref{eq:derivative_candidate} is actually close to the numerical solution of \cref{eq:discr_ode_vanilla} obtained using an ODE solver.
Thus, we provide here a candidate approximate expression for the discriminator in the setting of the original GAN formulation --~i.e., for binary classifiers.
We leave for future work a more in-depth study of this case.

\section{Discussions and Remarks}

We develop in this section some remarks and explanations on the topics that are broached in the main paper.

\subsection{From Finite to Infinite-Width Networks}
\label{app:finite_width}

The constancy of the neural tangent kernel during training when the width of the network becomes increasingly large is broadly applicable.
As summarized by \citet{Liu2020}, typical neural networks with the building blocks of multilayer perceptrons and convolutional neural networks comply with this property, as long as they end with a linear layer and they do not have any bottleneck --~indeed, this constancy needs the minimum internal width to grow unbounded \citep{Arora2019}.
This includes, for example, residual convolutional neural networks \citep{He2016}.
The requirement of a final linear activation can be circumvented by transferring this activation into the loss function, as we did for the original GAN formulation in \cref{sec:framework}.
This makes our framework encompass a wide range of discriminator architectures.

Indeed, many building blocks of state-of-the-art discriminators can be studied in this infinite-width regime with a constant NTK, as highlighted by the exhaustiveness of the Neural Tangents library \citep{Novak2020}.
Assumptions about the used activation functions are mild and include many standard activations such as ReLU, sigmoid and $\tanh$.
Beyond fully connected linear layers and convolutions, NTK constancy also affect typical operations such as self-attention \citep{Hron2020}, layer normalization and batch normalization \citep{Yang2020a}.
This variety of networks affected by the constancy of the NTK supports the generality of our approach, as it includes powerful discriminator architectures such as BigGAN \citep{Brock2019}.

We highlight that the NTK of the discriminator remains constant throughout the whole GAN optimization process, and not only under a fixed generator.
Indeed, if it remains constant in-between generator updates, then it also remains constant when the generator changes.
This is because, for a finite training time, the constancy of the NTK solely depends on the network architecture and initialization, regardless of the training loss which may change in the course of training without affecting the NTK.

There are nevertheless some limits to the NTK approximation, as we are not aware of works studying the application of the infinite-width regime to some operations such as spectral normalization, and networks in the regime of a constant NTK cannot perform feature learning as they are equivalent to kernel methods \citep{Geiger2020, Yang2021}.
However, this framework remains general and constitutes the most advanced attempt at theoretically modeling the discriminator's architecture in GANs.

\subsection{Loss of the Generator and its Gradient}
\label{app:alternating_optim}

We highlight in this section the importance of taking into account alternating optimization and discriminator gradients in the optimization of the generator.
Let us focus on an example similar to the one of \citet[Example~1]{Arjovsky2017b} and choose as $\beta$ a single Dirac centered at $0$ and as $\alpha_g = \alpha_\theta$ single Dirac centered at $x_\theta = \theta$ (the generator parameters being the coordinates of the generated point).
Let us study for the sake of simplicity the case of LSGAN since it is a recurring example in this work, but a similar reasoning can be done for other GAN instances.

In the theoretical min-max formulation of GANs considered by \citet{Arjovsky2017b}, the generator is trained to minimize the following quantity:
\begin{equation}
    \app{\Co_{f^\star_{\alpha_\theta}}}{\alpha_\theta} \triangleq  \E_{x \sim \alpha_\theta} \brackets*{\app{c_{f^\star_{\alpha_\theta}}}{x}} = \app{f^\star_{\alpha_\theta}}{x_\theta}^2,
\end{equation}
where:
\begin{equation}
    \begin{aligned}
        f^\star_{\alpha_\theta} & = \argmax_{f \in \app{L^2}{\frac{1}{2} \alpha_\theta + \frac{1}{2} \beta}} \braces*{\app{\Lo_{\alpha_\theta}}{f} \triangleq \E_{x \sim \alpha_\theta} \brackets*{\app{a_f}{x}} - \E_{y \sim \beta} \brackets*{\app{b_f}{y}}} \\
        & = \argmin_{f \in \app{L^2}{\frac{1}{2} \alpha_\theta + \frac{1}{2} \beta}} \braces*{\parentheses*{\app{f^\star_{\alpha_\theta}}{x_\theta} + 1}^2 + \parentheses*{\app{f^\star_{\alpha_\theta}}{0} - 1}^2}.
    \end{aligned}
\end{equation}
Consequently, $\app{f^\star_{\alpha_\theta}}{0} = 1$ and $\app{f^\star_{\alpha_\theta}}{x_\theta} = -1$ when $x_\theta \neq 0$, thus in this case:
\begin{equation}
    \app{\Co_{f^\star_{\alpha_\theta}}}{\alpha_\theta} = 1.
\end{equation}
This constancy of the generator loss would make it impossible to be learned by gradient descent, as pointed out by \citet{Arjovsky2017b}.

However, the setting does not correspond to the actual optimization process used in practice and represented by \cref{eq:gen_descent}.
We do have $\grad_\theta \app{\Co_{f^\star_{\alpha_\theta}}}{\alpha_\theta} = 0$ when $x_\theta \neq 0$, but the generator never uses this gradient in standard GAN optimization.
Indeed, this gradient takes into account the dependency of the optimal discriminator $f^\star_{\alpha_\theta}$ in the generator parameters, since the optimal discriminator depends on the generated distribution.
Yet, in practice and with few exceptions such as Unrolled GANs \citep{Metz2017} and as done in \cref{eq:gen_descent}, this dependency is ignored when computing the gradient of the generator, because of the alternating optimization setting --~where the discriminator is trained in-between generator's updates.
Therefore, despite being constant on the training data, this loss can yield non-zero gradients to the generator.
However, this requires the gradient of $f^\star_{\alpha_\theta}$ to be defined, which is the issue addressed in \cref{sec:issues_large_discr_scpace}.

\subsection{Differentiability of the Bias-Free ReLU Kernel}
\label{app:diff_relu_kernel}

\cref{rk:bias_free_relu} contradicts the results of \citet{Bietti2019} on the regularity of the NTK of a bias-free ReLU MLP with one hidden layer, which can be expressed as follows (up to a constant scaling the matrix multiplication in linear layers):
\begin{equation}
    \label{eq:relu_ntk}
    \app{k}{x, y} = \euclideannorm*{x} \euclideannorm*{y} \app{\kappa}{\frac{\sprod{x}{y}}{\euclideannorm*{x} \euclideannorm*{y}}},
\end{equation}
where:
\begin{equation}
    \begin{aligned}
        \kappa\colon \brackets*{0, 1} & \to \R \\
        u & \mapsto \frac{2}{\pi} u \parentheses*{\pi - \arccos u} + \frac{1}{\pi} \sqrt{1 - u^2}
    \end{aligned}
    \ .
\end{equation}

More particularly, \citet[Proposition\ 3]{Bietti2019} claim that $\app{k}{\cdot, y}$ is not Lipschitz around $y$ for all $y$ in the unit sphere.
By following their proof, it amounts to prove that $\app{k}{\cdot, y}$ is not Lipschitz around $y$ for all $y$ in any centered sphere.
We highlight that this also contradicts empirical evidence, as we did observe the Lipschitzness of such NTK in practice using the Neural Tangents library \citep{Novak2020}.

We believe that the mistake in the proof of \citet{Bietti2019} lies in the confusion between functions $\kappa$ and $k_0\colon x, y \mapsto \app{\kappa}{\frac{\sprod{x}{y}}{\euclideannorm*{x} \euclideannorm*{y}}}$, which have different geometries.
Their proof relies on the fact that $\kappa$ is indeed non-Lipschitz in the neighborhood of $u = 1$.
However, this does not imply that $k_0$ is not Lipschitz, or not derivable.
We can prove that it is actually at least locally Lipschitz.

Indeed, let us compute the following derivative for $x \neq y \in \R^n \setminus \braces*{0}$:
\begin{equation}
    \pd{\app{k_0}{x, y}}{x} = \frac{y \euclideannorm*{x} - \frac{x}{\euclideannorm*{x}}\sprod{x}{y}}{\euclideannorm*{x}^2 \euclideannorm*{y}} \app{\kappa'}{u} = \frac{1}{\euclideannorm*{x} \euclideannorm*{y}} \parentheses*{y - \sprod{x}{y} \frac{x}{\euclideannorm*{x}^2}} \app{\kappa'}{u},
\end{equation}
where $u = \frac{\sprod{x}{y}}{\euclideannorm*{x} \euclideannorm*{y}}$ and:
\begin{equation}
    \pi \cdot \app{\kappa'}{u} = \frac{u}{\sqrt{1 - u^2}} + 2 \parentheses*{\pi - \arccos u}.
\end{equation}
Note that $\app{\kappa'}{u} \sim_{u \to 1^-} \frac{\pi u}{\sqrt{1 - u^2}} \sim_{u \to 1^-} \frac{\pi}{\sqrt{2} \sqrt{1 - u}}$.
Therefore:
\begin{equation}
    \begin{aligned}
        \frac{\pi}{\sqrt{2}} \cdot \dpd{\app{k_0}{x, y}}{x} & \sim_{x \to y} \frac{1}{\euclideannorm*{y}^2} \parentheses*{y - \sprod{x}{y} \frac{x}{\euclideannorm*{x}^2}} \frac{\sqrt{\euclideannorm*{x} \euclideannorm*{y}}}{\sqrt{\euclideannorm*{x} \euclideannorm*{y} - \sprod{x}{y}}} \\
        & \sim_{x \to y} \frac{\euclideannorm*{x}^2 y - \sprod{x}{y} x}{\euclideannorm*{y}^3 \sqrt{\euclideannorm*{x} \euclideannorm*{y} - \sprod{x}{y}}} \\
        & \sim_{x \to y} \frac{\euclideannorm*{y}^2 - \sprod{x}{y}}{\euclideannorm*{y}^3 \sqrt{\euclideannorm*{y}^2 - \sprod{x}{y}}} y \xrightarrow[x \to y]{} 0,
    \end{aligned}
\end{equation}
which proves that $k_0$ is actually Lipschitz around points $\parentheses*{y, y}$, as well as differentiable, and confirms our remark.

\subsection{Integral Operator and Instance Noise}
\label{app:instance_noise}

Instance noise \citep{Sonderby2017} consists in adding random Gaussian noise to the input and target samples.
This amounts to convolving the data distributions with a Gaussian density, which will have the effect of smoothing the discriminator.
In the following, for the case of IPM losses, we link instance noise with our framework, showing that smoothing of the data distributions already occurs via the NTK kernel, stemming from the fact that the discriminator is a neural network trained with gradient descent.

More specifically, it can be shown that if $k$ is an RBF kernel, the optimal discriminators in both case are the same.
This is based on the fact that the density of a convolution of an empirical measure $\hmu = \frac{1}{N} \sum_i \delta_{x_i}$, where $\delta_z$ is the Dirac distribution centered on $z$, and a Gaussian density $\Tilde{k}$ with associated RBF kernel $k$ can be written as $\Tilde{k} *\hmu = \frac{1}{N} \sum_i \app{k}{x_i, \cdot}$.

Let us consider the following regularized discriminator optimization problem in $\app{L^2}{\R}$ smoothed from $\app{L^2}{\Omega}$  with instance noise, i.e.\ convolving $\halpha$ and $\hbeta$ with $\Tilde{k}$.
\begin{equation}
    \sup_{f \in \app{L^2}{\R}}\braces*{\app{\Lo_{\halpha}^{\Tilde{k}}}{f} \triangleq \E_{x \sim \Tilde{k} * \halpha} \brackets*{\app{f}{x}} - \E_{y \sim \Tilde{k} *\hbeta} \brackets*{\app{f}{y}} - \lambda \norm{f}_{L^2}^2}
\end{equation}
The optimum $f^{\mathrm{IN}}$ can be found by taking the gradient:
\begin{align}
    \grad_f \parentheses*{\app{\Lo_{\halpha}^{\Tilde{k}}}{f^{\mathrm{IN}}}  - \lambda \norm{f^{\mathrm{IN}}}_{L^2}^2} = 0 && \Leftrightarrow && f^{\mathrm{IN}} = \frac{1}{2\lambda} \parentheses*{\Tilde{k} *\halpha - \Tilde{k} *\hbeta}.
\end{align}

If we now study the resolution of the optimization problem in $\gH^\hgamma_k$ as in \cref{sec:ipm_loss} with $f_0 = 0$, we find the following discriminator:
\begin{equation}
    f_t = t \parentheses*{\E_{x \sim \halpha} \brackets*{\app{k}{x, \cdot}} - \E_{y \sim \hbeta} \brackets*{\app{k}{y, \cdot}}} = t \parentheses*{\Tilde{k} *\halpha - \Tilde{k} *\hbeta}.
\end{equation}
Therefore, we have that $f^{\text{IN}} \propto f_t$, i.e.\ instance noise and regularization by neural networks obtain the same smoothed solution.

This analysis was done using the example of an RBF kernel, but it also holds for stationary kernels, i.e.\ $\app{k}{x, y} = \app{\Tilde{k}}{x - y}$, which can be used to convolve measures.
We remind that this is relevant, given that NTKs are stationary over spheres \citep{Jacot2018, Yang2019}, around where data can be concentrated in high dimensions.

\subsection{Positive Definite NTKs}
\label{app:characteristic_ntk}

Optimality results in the theory of NTKs usually rely on the assumption that the considered NTK $k$ is positive definite over the training dataset $\hgamma$ \citep{Jacot2018, Zhang2020}.
This property offers several theoretical advantages.

Indeed, this gives sufficient representational power to its RKHS to include the optimal solution over $\hgamma$.
Moreover, this positive definiteness property equates for finite datasets to the invertibility of the mapping
\begin{equation}
    \begin{aligned}
        \rightvert*{\Tk{k}{\hgamma}}_{\supp \hgamma}\colon \app{L^2}{\hgamma} & \to \app{L^2}{\hgamma} \\
        h & \mapsto \rightvert*{\app{\Tk{k}{\hgamma}}{h}}_{\supp \hgamma}
    \end{aligned}
    ,
\end{equation}
that can be seen as a multiplication by the invertible Gram matrix of $k$ over $\hgamma$.
From this, one can retrieve the expression of $f \in \gH_k^{\hgamma}$ from its restriction $\rightvert*{f}_{\supp \hgamma}$ to $\supp \hgamma$ in the following way:
\begin{equation}
    f = \app{\Tk{k}{\hgamma} \circ \rightvert*{\Tk{k}{\hgamma}}_{\supp \hgamma}^{-1}}{\rightvert*{f}_{\supp \hgamma}},
\end{equation}
as shown in \cref{lem:inv_tk}.
Finally, as shown by \citet{Jacot2018} and in \cref{app:optimality_concave}, this makes the discriminator loss function strictly increase during training.

One may wonder whether this assumption is reasonable for NTKs.
\citet{Jacot2018} proved that it indeed holds for NTKs of non-shallow MLPs with non-polynomial activations if data is supported on the unit sphere, supported by the fact that the NTK is stationary over the unit sphere.
Others, such as \citet{Fan2020}, have observed positive definiteness of the NTK subject to specific assumptions on the networks and data.
We are not aware of more general results of this kind.
However, one may conjecture that, at least for specific kinds of networks, NTKs are positive definite for any training data.

Indeed, besides global convergence results \citep{AllenZhu2019}, prior work indicates that MLPs are universal approximators \citep{Hornik1989, Leshno1993}.
This property can be linked in our context to universal kernels \citep{Steinwart2001}, which are guaranteed to be positive definite over any training data \citep{Sriperumbudur2011}.
Universality is linked to the density of the kernel RKHS in the space of continuous functions.
In the case of NTKs, previously cited approximation properties can be interpreted as signs of expressive RKHSs, and thus support the hypothesis of universal NTKs.
Furthermore, beyond positive definiteness, universal kernels are also characteristic \citep{Sriperumbudur2011}, which is interesting when they are used to compute MMDs, as we do in \cref{sec:ipm_loss}.
Note that for the standard case of ReLU MLPs, \citet{Ji2020} showed universal approximation results in the infinite-width regime, and works such as the one of \citet{Chen2021} observed that their RKHS is close to the one of the Laplace kernel, which is positive definite.

\paragraph{Bias-free ReLU NTKs are not characteristic.}
As already noted by \citet{Leshno1993}, the presence of bias is important when it comes to representational power of MLPs.
We can retrieve this observation in our framework.
In the case of a ReLU shallow network with one hidden layer and without bias, \citet{Bietti2019} determine its associated NTK as follows (up to a constant scaling the matrix multiplication in linear layers):
\begin{equation}
    \app{k}{x, y} = \euclideannorm*{x} \euclideannorm*{y} \app{\kappa}{\frac{\sprod{x}{y}}{\euclideannorm*{x} \euclideannorm*{y}}},
\end{equation}
with in particular $\app{k}{x, 0} = 0$ for all $x \in \Omega$; suppose that $0 \in \Omega$.
This expression of the kernel implies that $k$ is not positive definite for all datasets: take for example $x = 0$ and $y \in \Omega \setminus \braces*{0}$; then the Gram matrix of $k$ has a null row, hence $k$ is not strictly positive definite over $\braces*{x, y}$.
Another consequence is that $k$ is not characteristic.
Indeed, take probability distributions $\mu = \delta_{\frac{y}{2}}$ and $\nu = \frac{1}{2} \parentheses*{\delta_{x} + \delta_{y}}$ with $\delta_z$ being the Dirac distribution centered on $z \in \Omega$, and where $x = 0$ and $y \in \Omega \setminus \braces*{0}$.
Then:
\begin{equation}
    \E_{z \sim \mu} \app{k}{z, \cdot} = \app{k}{\frac{1}{2} y, \cdot} = \frac{1}{2} \app{k}{y, \cdot} = \frac{1}{2} \parentheses*{\app{k}{y, \cdot} + \app{k}{x, \cdot}} = \E_{z \sim \nu} \app{k}{z, \cdot},
\end{equation}
i.e., kernel embeddings of $\mu$ and $\nu \neq \mu$ are identical, making $k$ not characteristic by definition.

\subsection{Societal Impact}

As our work is mainly theoretical and does not deal with real-world data, it does not have direct broader negative impact on the society.
However, the practical perspectives that it opens constitute an object of interrogation.
Indeed, the developments of performant generative models can be the source of harmful manipulation \citep{Tolosana2020} and reproduction of existing biases in databases \citep{Jain2020}, especially as GANs are still misunderstood.
While such negative effects should be considered, attempts such as ours at explaining generative models might also lead to ways to mitigate potential harms by paving the way for more principled GAN models.

\section{\texorpdfstring{\GANTK}{GAN(TK)2}\ and Further Empirical Analyses}
\label{app:more_experiments}

We present in this section additional experimental results that complement and explain some of the results already exposed in \cref{sec:experiments}.
All these experiments were conducted using the proposed general toolkit \GANTK.

We focus in this article on particular experiments for the sake of clarity and as an illustration of the potential of analysis of our framework, but \GANTK\ is a general-purpose toolkit centered around the infinite-width of the discriminator and could be leveraged for an even more extensive empirical analysis.
We specifically focus on the IPM and LSGAN losses for the discriminator since they are the two losses for which we know the analytic behavior of the discriminator in the infinite-width limit, but other losses can be studied as well in \GANTK.
We leave a large-scale empirical study of our framework, which is out of the scope of this paper, for future work.

\begin{figure}
    \centering
    \includegraphics[width=\textwidth]{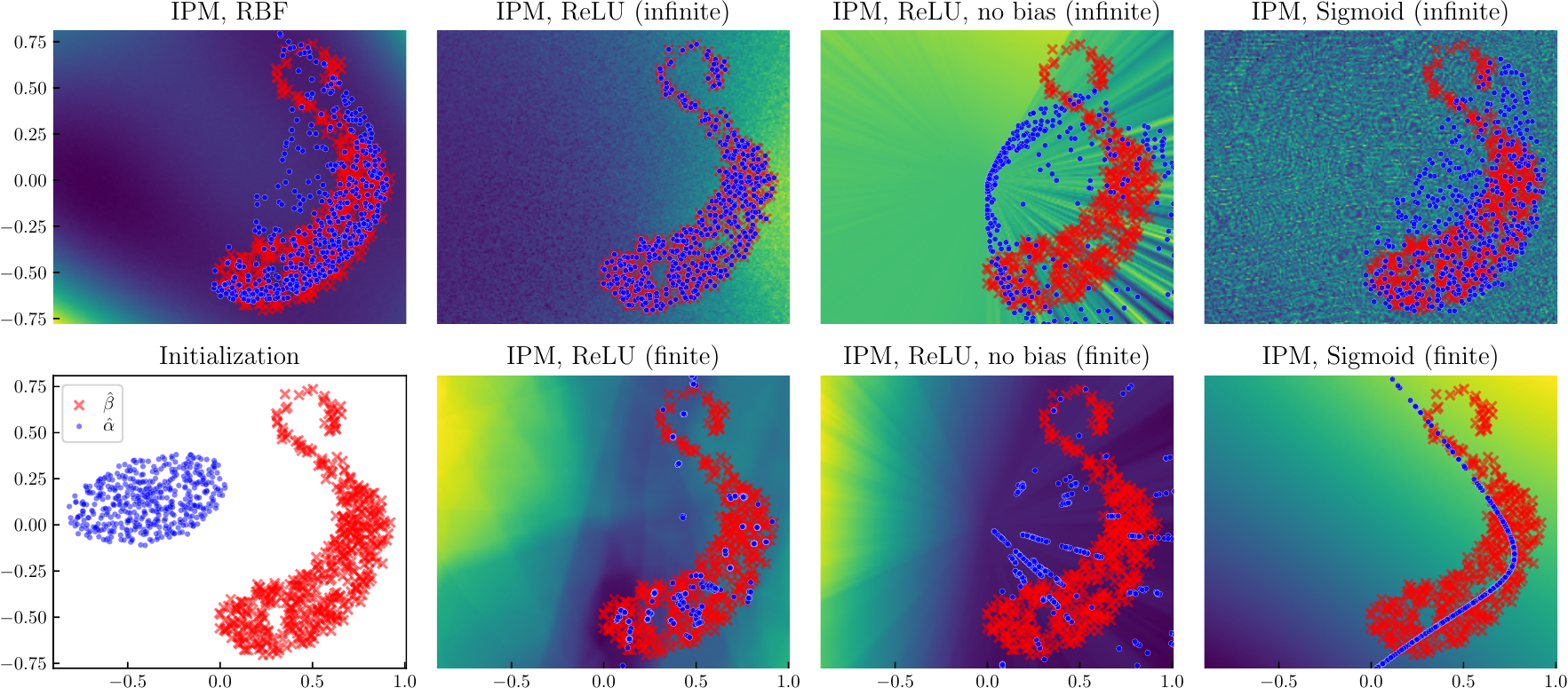}

    \vspace{\baselineskip}
    \caption{
        Generator (\textcolor{blue}{\ding{108}}) and target (\textcolor{red}{$\times$}) samples for different methods applied to the Density problem.
        In the background, $c_{f^\star}$.
    }
    \label{fig:density}
\end{figure}

\begin{figure}
    \centering
    \includegraphics[width=0.4\textwidth]{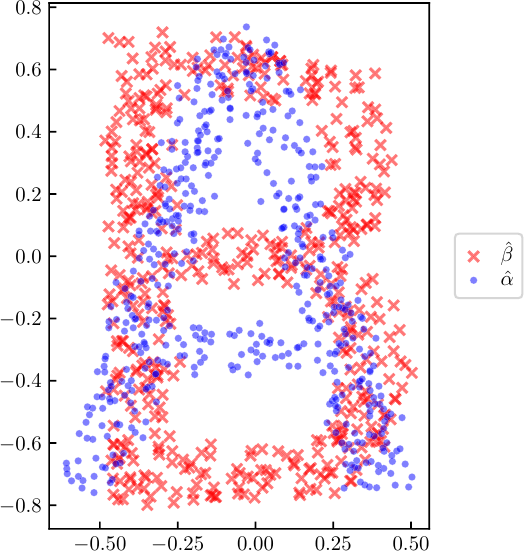}

    \vspace{\baselineskip}
    \caption{
        Initial generator (\textcolor{blue}{\ding{108}}) and target (\textcolor{red}{$\times$}) samples for the AB problem.
    }
    \label{fig:ab}
\end{figure}

\subsection{Two-Dimensional Datasets}

We provide in \cref{tab:res_8_gaussians_sigmoid} numerical results corresponding to the experiments described in \cref{sec:experiments} on the $8$ Gaussians dataset.

We present additional experimental results on two other two-dimensional problems, Density and AB; see, respectively, \cref{fig:density,fig:ab}.
Numerical results are detailed in \cref{tab:res_density,tab:res_ab}.
We globally retrieve the same conclusions that we developed in \cref{sec:experiments} on these datasets with more complex shapes.

\begin{table}
    \caption{
        \label{tab:res_8_gaussians_sigmoid}
        Sinkhorn divergence \citep[lower is better, similar to $\gW_2$]{Feydy2018} averaged over three runs between the final generated distribution and the target dataset for the $8$ Gaussians problem.
    }
    \sisetup{detect-weight, table-align-uncertainty=true, table-number-alignment=center, mode=text}
    \renewrobustcmd{\bfseries}{\fontseries{b}\selectfont}
    \renewrobustcmd{\boldmath}{}
    \centering
    \vspace{\baselineskip}
    \begin{tabular}{lllll}
        \toprule
        Loss & \multicolumn{1}{c}{RBF kernel} & \multicolumn{1}{c}{ReLU} & \multicolumn{1}{c}{ReLU (no bias)} & \multicolumn{1}{c}{Sigmoid} \\
        \midrule
        IPM (inf.) & \num{2.60 \pm 0.06 e-2} & \num{9.40 \pm 2.71 e-7} & \num{9.70 \pm 1.88 e-2} & \num{8.40 \pm 0.02 e-2} \\
        IPM & \multicolumn{1}{c}{\textemdash} & \num{1.21 \pm 0.14 e-1} & \num{1.20 \pm 0.60 e+0} & \num{7.40 \pm 1.30 e-1} \\
        \midrule
        LSGAN (inf.) & \num{4.21 \pm 0.1 e-1} & \num{7.56 \pm 0.45 e-2} & \num{1.27 \pm 0.01 e+1} & \num{7.35 \pm 0.11 e+0} \\
        LSGAN & \multicolumn{1}{c}{\textemdash} & \num{3.07 \pm 0.68 e+0} & \num{7.52 \pm 0.01 e+0} & \num{7.41 \pm 0.54 e+0} \\
        \bottomrule
    \end{tabular}
\end{table}

\begin{table}
    \caption{
        \label{tab:res_density}
        Sinkhorn divergence averaged over three runs between the final generated distribution and the target dataset for the Density problem.
    }
    \sisetup{detect-weight, table-align-uncertainty=true, table-number-alignment=center, mode=text}
    \renewrobustcmd{\bfseries}{\fontseries{b}\selectfont}
    \renewrobustcmd{\boldmath}{}
    \centering
    \vspace{\baselineskip}
    \begin{tabular}{lllll}
        \toprule
        Loss & \multicolumn{1}{c}{RBF kernel} & \multicolumn{1}{c}{ReLU} & \multicolumn{1}{c}{ReLU (no bias)} & \multicolumn{1}{c}{Sigmoid} \\
        \midrule
        IPM (inf.) & \num{2.37 \pm 0.32 e-3} & \num{3.34 \pm 0.49 e-9} & \num{7.34 \pm 0.34 e-2} & \num{6.25 \pm 0.31 e-3} \\
        IPM & \multicolumn{1}{c}{\textemdash} & \num{5.02 \pm 1.19 e-3} & \num{9.25 \pm 0.30 e-2} & \num{3.06 \pm 0.57 e-2} \\
        \midrule
        LSGAN (inf.) & \num{7.53 \pm 0.59 e-3} & \num{1.49 \pm 0.11 e-3} & \num{2.80 \pm 0.03 e-1} & \num{2.21 \pm 0.01 e-1} \\
        LSGAN & \multicolumn{1}{c}{\textemdash} & \num{1.53 \pm 1.08 e-2} & \num{1.64 \pm 0.19 e-1} & \num{5.88 \pm 0.80 e-2} \\
        \bottomrule
    \end{tabular}
\end{table}

\begin{table}
    \caption{
        \label{tab:res_ab}
        Sinkhorn divergence averaged over three runs between the final generated distribution and the target dataset for the AB problem.
    }
    \sisetup{detect-weight, table-align-uncertainty=true, table-number-alignment=center, mode=text}
    \renewrobustcmd{\bfseries}{\fontseries{b}\selectfont}
    \renewrobustcmd{\boldmath}{}
    \centering
    \vspace{\baselineskip}
    \begin{tabular}{lllll}
        \toprule
        Loss & \multicolumn{1}{c}{RBF kernel} & \multicolumn{1}{c}{ReLU} & \multicolumn{1}{c}{ReLU (no bias)} & \multicolumn{1}{c}{Sigmoid} \\
        \midrule
        IPM (inf.) & \num{4.65 \pm 0.82 e-3} & \num{2.64 \pm 2.13 e-9} & \num{6.11 \pm 0.19 e-3} & \num{5.69 \pm 0.38 e-3} \\
        IPM & \multicolumn{1}{c}{\textemdash} & \num{2.75 \pm 0.20 e-3} & \num{3.65 \pm 1.44 e-2} & \num{1.25 \pm 0.32 e-2} \\
        \midrule
        LSGAN (inf.) & \num{1.13 \pm 0.05 e-2} & \num{8.63 \pm 2.24 e-3} & \num{1.02 \pm 0.40 e-1} & \num{1.40 \pm 0.06 e-2} \\
        LSGAN & \multicolumn{1}{c}{\textemdash} & \num{1.32 \pm 1.30 e-1} & \num{2.57 \pm 0.73 e-2} & \num{8.78 \pm 2.23 e-2} \\
        \bottomrule
    \end{tabular}
\end{table}

\begin{figure}
    \centering
    \subfigure[RBF kernel: blurry digits on MNIST, prohibitively noisy images on CelebA.]{\begin{tabular}{@{}c@{}}\includegraphics[width=0.9\textwidth]{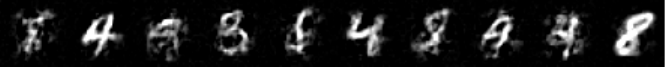}\\\includegraphics[width=0.9\textwidth]{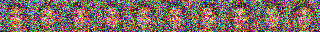}\end{tabular}}
    \subfigure[ReLU: sharp digits on MNIST, high-quality images on CelebA.]{\begin{tabular}{@{}c@{}}\includegraphics[width=0.9\textwidth]{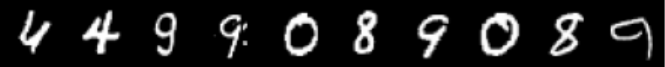}\\\includegraphics[width=0.9\textwidth]{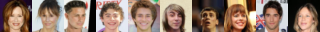}\end{tabular}}
    \subfigure[ReLU (no bias): mostly sharp digits with some artifacts and blurry images on MNIST, blurry and noisy images on CelebA.]{\begin{tabular}{@{}c@{}}\includegraphics[width=0.9\textwidth]{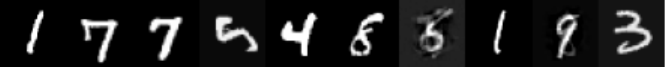}\\\includegraphics[width=0.9\textwidth]{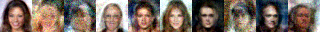}\end{tabular}}
    \caption[
        Convergence experiment on Moving MNIST and CelebA.
    ]{
        Uncurated samples from the results of the descent of a set of \num{1024} particles over a subset of \num{1024} elements of MNIST and CelebA, starting from a standard Gaussian.
        Training is done using the IPM loss in the infinite-width kernel setting.
    }
    \label{fig:mnist}
\end{figure}

\subsection{ReLU vs.\ Sigmoid Activations}

We additionally introduce a new baseline for the $8$ Gaussians, Density and AB problems, where we replace the ReLU activation in the discriminator by a sigmoid-like activation $\Tilde{\sigma}$, that we abbreviate to sigmoid in this experimental study for readability purposes.
We choose $\Tilde{\sigma}$ instead of the actual sigmoid $\sigma$ for computational reasons, since $\Tilde{\sigma}$, contrary to $\sigma$, allows for analytic computations of NTKs in the Neural Tangents library \citep{Novak2020}.
$\Tilde{\sigma}$ is defined in the latter using the error function $\erf$ scaled in order to minimize a squared loss with respect to $\sigma$ over $\brackets{-5, 5}$, with the following expression:
\begin{equation}
    \Tilde{\sigma}\colon x \mapsto \frac{1}{2} \parentheses*{\app{\erf}{\frac{x}{\num{2.4020563531719796}}} + 1}.
\end{equation}

Results are given in \cref{tab:res_8_gaussians_sigmoid,tab:res_density,tab:res_ab} and an illustration is available in \cref{fig:density}.
We observe that the sigmoid baseline is consistently outperformed by the RBF kernel and ReLU activation (with bias) for all regimes and losses.
This is in accordance with common experimental practice, where internal sigmoid activations are found less effective than ReLU because of the potential activation saturation that they can induce.

We provide a qualitative explanation to this underperformance of sigmoid via our framework in \cref{sec:gradient_field}.

\subsection{Qualitative MNIST and CelebA Experiment}

An experimental analysis of our framework on complex image datasets is out the scope of our study --~we leave it for future work.
Nonetheless, we present an experiment on MNIST \citep{LeCun1998} and CelebA \citep{Liu2015} images in a similar setting as the experiments on two-dimensional point clouds of the previous sections.
For each dataset, we make a point cloud $\halpha$, initialized to a standard Gaussian, move towards a subset of the MNIST dataset following the gradients of the IPM loss in the infinite-width regime.
Qualitative results are presented in \cref{fig:mnist}.

We notice, similarly to the two-dimensional experiments, that the ReLU network with bias outperforms its bias-free counterpart and a standard RBF kernel in terms of sample quality.
The difference between the RBF kernel and ReLU NTK is even more flagrant in this complex high-dimensional setting, as the RBF kernel is unable to produce accurate samples.

\subsection{Visualizing the Gradient Field Induced by the Discriminator}
\label{sec:gradient_field}

\begin{figure}
    \centering
    \includegraphics[width=0.87\textwidth]{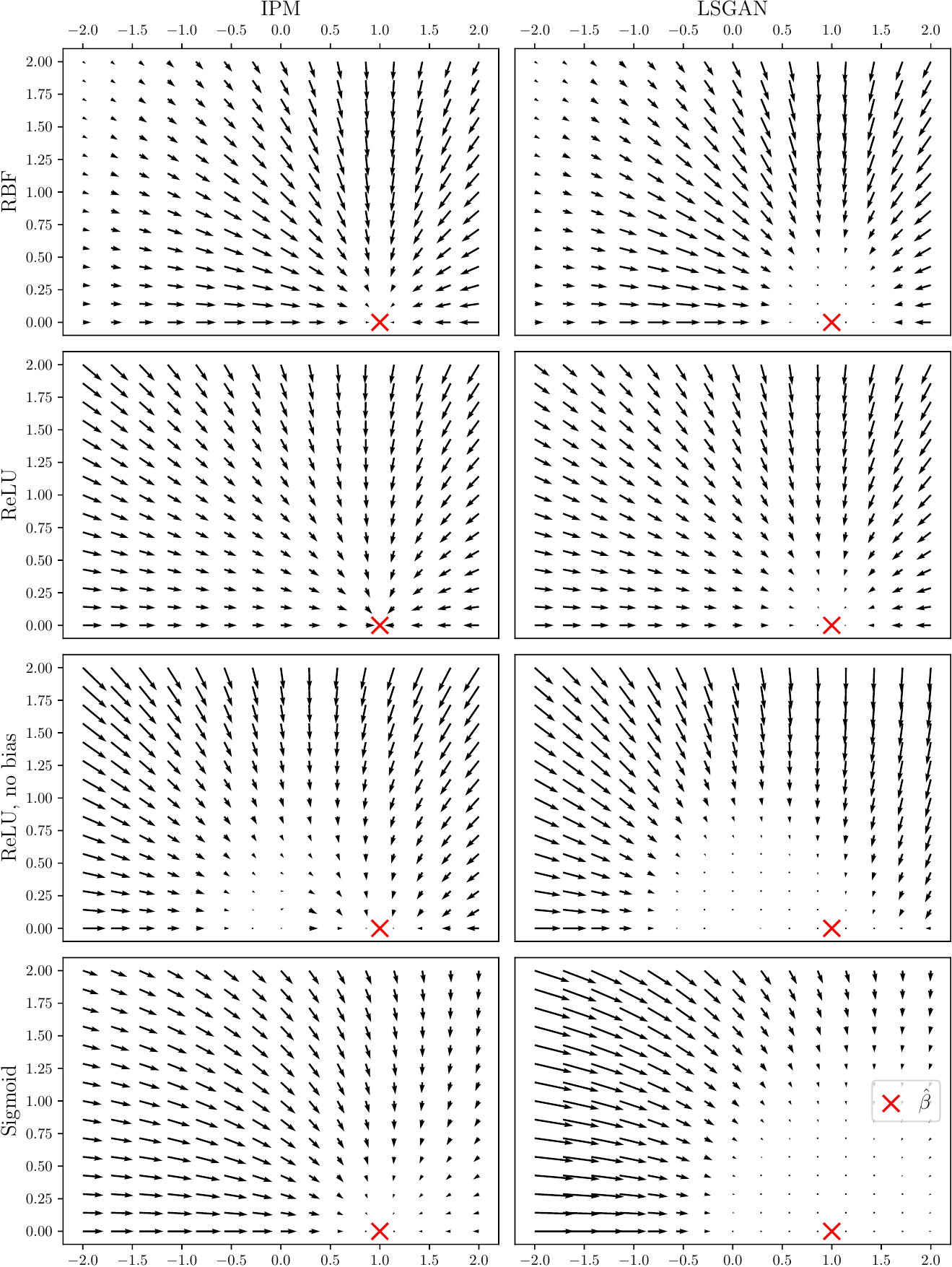}
    \caption{
        Gradient field $\grad \app{c_{f^\star_{\halpha_x}}}{x}$ received by a generated sample $x \in \R^2$ (i.e.\ $\halpha = \halpha_x = \delta_x$) initialized to $x_0$ with respect to its coordinates in $\spn \braces*{x_0, y}$ where $y$, marked by a \textcolor{red}{$\bs{\times}$}, is the target distribution (i.e.\ $\hbeta = \delta_y$), with $\euclideannorm*{y} = 1.$
        Arrows correspond to the movement of $x$ in $\spn \braces*{x_0, y}$ following $\grad \app{c_{f^\star_{\halpha_x}}}{x}$, for different losses and networks; scales are specific for each pair of loss and network.
        The ideal case is the convergence of $x$ along this gradient field towards the target $y$.
        Note that in the chosen orthonormal coordinate system, without loss of generality, $y$ has coordinate $\parentheses*{1, 0}$; moreover, the gradient field is symmetrical with respect to the horizontal axis.
    }
    \label{fig:grad_field}
\end{figure}

\begin{figure}
    \centering
    \includegraphics[width=0.87\textwidth]{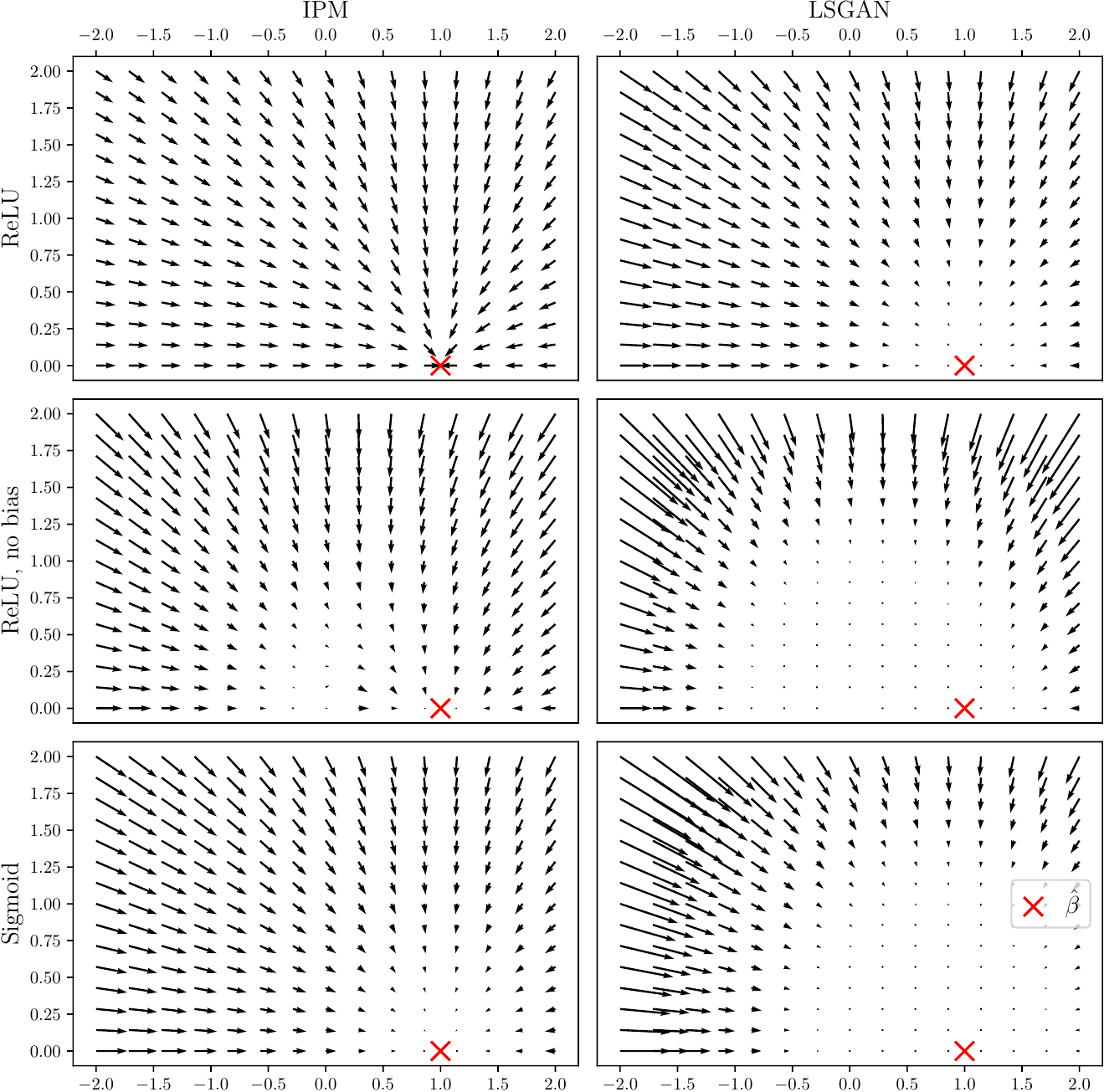}
    \caption{
        Same plot as \cref{fig:grad_field} but with underlying points $x, y \in \R^{512}$.
    }
    \label{fig:grad_field_512}
\end{figure}

We raise in \cref{sec:dynamics_generated_distr,sec:case_analysis} the open problem of studying the convergence of the generated distribution towards the target distribution with respect to the gradients of the discriminator.
We aim in this section at qualitatively studying these gradients in a simplified case that could shed some light on the more general setting and explain some  of our experimental results. These gradient fields can be plotted using the provided \GANTK\ toolkit.

\subsubsection{Setting}
Since we study gradients of the discriminator expressed in \cref{eq:discr_dynamics}, we assume that $f_0 = 0$ --~for instance, using the anti-symmetrical initialization \citet{Zhang2020}~-- in order to ignore residual gradients from the initialization.

By \cref{thm:discr_characterization}, for any loss and any training time, the discriminator can be expressed as ${f^\star_\halpha = \app{\T_{k, \hgamma}}{h_0}}$, for some $h_0 \in \app{L^2}{\hgamma}$.
Thus, there exists $h_1 \in \app{L^2}{\hgamma}$ such that:
\begin{equation}
    f^\star_\halpha = \sum_{x \in \supp \hgamma} \app{h_1}{x} \app{k}{x, \cdot}.
\end{equation}
Consequently,
\begin{align}
    \grad f^\star_\halpha = \sum_{x \in \supp \hgamma} \app{h_1}{x} \grad \app{k}{x, \cdot},&&
    \grad c_{f^\star_\halpha} = \sum_{x \in \supp \hgamma} \app{h_1}{x} \grad \app{k}{x, \cdot} \app{c'}{\app{f^\star_\halpha}{\cdot}}.
\end{align}

\paragraph{Dirac-GAN setting.}
The latter linear combination of gradients indicates that, by examining gradients of $c_{f^\star_\halpha}$ for pairs of $\parentheses*{x, y} \in \supp \halpha \times \supp \hbeta$, one could already develop potentially valid intuitions that can hold even when multiple points are considered.
This is especially the case for the IPM loss, as $h_0, h_1$ have a simple form: $\app{h_1}{x} = 1$ if $x \in \supp \halpha$ and $\app{h_1}{y} = -1$ if $y \in \supp \halpha$ (assuming points from $\halpha$ and $\hbeta$ are uniformly weighted); moreover, note that $\app{c'}{\app{f^\star_\halpha}{\cdot}} = 1$.
Thus, we study here $\grad c_{f^\star_\halpha}$ when $\halpha$ and $\hbeta$ are only comprised of one point, i.e.\ the setting of Dirac GAN \citep{Mescheder2018}, with $\halpha = \delta_{x} \triangleq \halpha_{x}$ and $\hbeta = \delta_{y}$.

\paragraph{Visualizing high-dimensional inputs.}
Unfortunately, the gradient field is difficult to visualize when the samples live in a high-dimensional space.
Interestingly, the NTK $\app{k}{x, y}$ for any architecture starting with a fully connected layer only depends on $\euclideannorm*{x}$, $\euclideannorm*{y}$ and $\sprod{x}{y}$ \citep{Yang2019}, and therefore all the information of $\grad c_{f^\star_\halpha}$ is contained in $\spn\braces*{x, y}$.
From this, we show in \cref{fig:grad_field,fig:grad_field_512} the gradient field $\grad c_{f^\star_\halpha}$ in the two-dimensional space $\spn\braces*{x, y}$ for different architectures and losses in the infinite-width regime described in \cref{sec:experiments} and in this section.
\cref{fig:grad_field} corresponds to two-dimensional $x, y \in \R^2$, and \cref{fig:grad_field_512} to high-dimensional $x, y \in \R^{512}$. Note that in the plots, the gradient field is symmetric w.r.t. the horizontal axis and for this reason we have restricted the illustration to the case where the second coordinate is positive.

\paragraph{Convergence of the gradient flow.}
In the last paragraph, we have seen that the gradient field in the Dirac-GAN setting lives in the two-dimensional $\spn\braces*{x, y}$, independently of the dimensionality of $x, y$. This means that when training the generated distribution, as in \cref{sec:experiments}, the position of the particle $x$ always remains in this two-dimensional space, and hence (non-)convergence in this setting can be easily checked by studying this gradient field. This is what we do in the following, for different architectures and losses.

\subsubsection{Qualitative Analysis of the Gradient Field}

\paragraph{$x$ is far from $y$.} When generated outputs are far away from the target, it is essential that their gradient has a large enough magnitude in order to pull these points towards the target. The behavior of the gradients for distant points can be observed in the plots. For ReLU networks, for both losses, the gradients for distant points seem to be well behaved and large enough. Note that in the IPM case, the magnitude of the gradients is even larger when $x$ is further away from $y$. This is not the case for the RBF kernel when the variance parameter is too small, as the magnitude of the gradient becomes prohibitively small. We highlight that we selected a large variance parameter in order to avoid such a behavior, but diminishing magnitudes can still be observed. Note that choosing an overly large variance may also have a negative impact on the points that are closer to the target.

\paragraph{$x$ is close to $y$.} A particularity of the NTK of ReLU discriminators with bias that arises from this study is that the gradients vanish more slowly when the generated $x$ tends to the target $y$, compared to NTKs of ReLU without bias and sigmoid networks, and to the RBF kernel.
We hypothesize that this is also another distinguishing feature that helps the generated distribution to converge more easily to the target distribution, especially when they are not far apart.
On the contrary, this gradient vanishes more rapidly for NTKs of ReLU without bias and sigmoid networks, compared to the RBF kernel.
This can explain the worse performance of such NTKs compared to the RBF kernel in our experiments (see \cref{tab:res_density,tab:res_ab,tab:res_8_gaussians_sigmoid}).
Note that this phenomenon is even more pronounced in high-dimensional spaces such as in \cref{fig:grad_field_512}.

\paragraph{$x$ is close to $0$.} Finally, we highlight gradient vanishing and instabilities around the origin for ReLU networks without bias.
This is related to its differentiability issues at the origin exposed in \cref{sec:differentiability}, and to its lack of representational power discussed in \cref{app:characteristic_ntk}.
This can also be retrieved on larger scale experiments of \cref{fig:8_gaussians,fig:density} where the origin is the source of instabilities in the descent.

\paragraph{Sigmoid network.}
It is also possible to evaluate the properties of the discriminator's gradient for architectures that are not used in practice, such as networks with the sigmoid activation. \cref{fig:8_gaussians,fig:density} provide a clear explanation: as stated above, the magnitudes of the gradients become too small when $x \to y$, and heavily depend on the direction from which $x$ approaches $y$. Ideally, the induced gradient flow should be insensitive to the direction in order for the convergence to be reliable and robust, which seems to be the case for ReLU networks.

\section{Experimental Details}
\label{app:experiments_details}

We detail in this section the experimental parameters needed to reproduce our experiments.

\subsection{\texorpdfstring{\GANTK}{GAN(TK)2}\ Specifications and Computing Resources}

\GANTK\ is implemented in Python (tested on versions 3.8.1 and 3.9.2) and based on JAX \citep{Bradbury2018} for tensor computations and Neural Tangents \citep{Novak2020} for NTKs.
We refer to the code released at \url{https://github.com/emited/gantk2} for detailed specifications and instructions.

All experiments presented in this paper were run on Nvidia GPUs (Nvidia Titan RTX --~24GB of VRAM~-- with CUDA 11.2 as well as Nvidia Titan V --~12GB~-- and Nvidia GeForce RTX 2080 Ti --~11 GB~-- with CUDA 10.2).
All two-dimensional experiments require only a few minutes of computations on a single GPU.
Experiments on MNIST and CelebA were run using simultaneously four GPUs for parallel computations, for at most a couple of hours.

\subsection{Datasets}

\paragraph{$8$ Gaussians.}

The target distribution is composed of $8$ Gaussians with their means being evenly distributed on the centered sphere of radius $5$, and each with a standard deviation of $0.5$.
The input fake distribution is drawn at initialization from a standard normal distribution $\app{\gN}{0, 1}$.
We sample in our experiments $500$ points from each distribution at each run to build $\halpha$ and $\hbeta$.

\paragraph{AB and Density.}

These two datasets are taken from the Geomloss library examples \citep{Feydy2018}\footnote{They can be downloaded at \url{https://github.com/jeanfeydy/geomloss/tree/main/geomloss/examples/optimal_transport/data}: AB corresponds to files \texttt{A.png} (source) and \texttt{B.png} (target), and Density corresponds to files \texttt{density_a.png} (source) and \texttt{density_a.png} (target).} and are distributed under the MIT license.
To sample a point from a distribution based on these greyscale images files, we sample a pixel (considered to lie in $\brackets{-1, 1}^2$) in the image from a distribution where each pixel probability is proportional to the darkness of this pixel, and then apply a Gaussian noise centered at the chosen pixel coordinates with a standard deviation equal to the inverse of the image size.
We sample in our experiments $500$ points from each distribution at each run to build $\halpha$ and $\hbeta$.

\paragraph{MNIST and CelebA.}

We preprocess each MNIST image \citep{LeCun1998} by extending it from $28 \times 28$ frames to $32 \times 32$ frames (by padding it with black pixels).
CelebA images \citep{Liu2015} are downsampled from a size of $178 \times 218$ to $32 \times 39$ and then center-cropped to $32 \times 32$.

For both datasets, we normalize pixels in the $\brackets*{-1, 1}$ range.
For our experiments, we consider a subset of \num{1024} elements of each dataset, which are randomly sampled for each run.

\subsection{Parameters}

\paragraph{Sinkhorn divergence.}

The Sinkhorn divergence is computed using the Geomloss library \citep{Feydy2018}, with a blur parameter of \num{0.001} and a scaling of $0.95$, making it close to the Wasserstein $\gW_2$ distance.

\paragraph{RBF kernel.}

The RBF kernel used in our experiments is the following:
\begin{equation}
    \app{k}{x, y} = \erm^{\frac{\euclideannorm*{x - y}^2}{2 n}},
\end{equation}
where $n$ is the dimension of $x$ and $y$, i.e.\ the dimension of the data.

\paragraph{Architecture.}

We used for the neural networks of our experiments the standard NTK paramaterization \citep{Jacot2018}, with a scaling factor of $1$ for matrix multiplications and, when bias in enabled, a multiplicative constant of $1$ for biases (except for sigmoid where this bias factor is lowered to $0.2$ to avoid saturating the sigmoid, and for CelebA where it is equal to $4$).
All considered networks are composed of $3$ hidden layers and end with a linear layer.
In the finite-width case, the width of these hidden layers is $128$.
We additionally use antisymmetrical initialization \citep{Zhang2020}, except for the finite-width LSGAN loss.

\paragraph{Discriminator optimization.}

Discriminators in the finite-width regime are trained using full-batch gradient descent without momentum, with one step per update to the distributions and the following learning rates $\varepsilon$:
\begin{itemize}
    \item for the IPM loss: $\varepsilon = 0.01$;
    \item for the IPM loss with reset and LSGAN: $\varepsilon = 0.1$.
\end{itemize}

In the infinite-width limit, we use the analytic expression derived in \cref{sec:case_analysis} with training time $\tau = 1$ (except for MNIST and CelebA where $\tau = 1000$) and $f_0 = 0$ (through the initialization of \citet{Zhang2020}) to avoid the computational cost of accumulating discriminators' analytic expressions across the generator's optimization steps.

\paragraph{Point cloud descent.}

The multiplicative constant $\eta$ over the gradient applied to each datapoint for two-dimensional problems is chosen as follows:
\begin{itemize}
    \item for the IPM loss in the infinite-width regime: $\eta = 1000$;
    \item for the IPM loss in the finite-width regime: $\eta = 100$;
    \item for the IPM loss in the finite-width regime and discriminator reset: $\eta = 1000$;
    \item for LSGAN in the infinite-width regime: $\eta = 1000$;
    \item for LSGAN in the finite-width regime: $\eta = 1$.
\end{itemize}
We multiply $\eta$ by \num{1000} when using sigmoid activations, because of the low magnitude of the gradients it provides.
We choose for MNIST $\eta = 100$.

Training is performed for the following number of iterations:
\begin{itemize}
    \item for $8$ Gaussians: \num{20000};
    \item for Density and AB: \num{10000};
    \item for MNIST: \num{50000}.
\end{itemize}